%% file: neurips_2026.tex
\documentclass{article}

    \PassOptionsToPackage{numbers, compress}{natbib}

 \usepackage[main, final]{neurips_2026}

\usepackage[utf8]{inputenc} 
\usepackage[T1]{fontenc}    
\usepackage{hyperref}       
\usepackage{url}            
\usepackage{booktabs}       
\usepackage{amsfonts}       
\usepackage{nicefrac}       
\usepackage{microtype}      
\usepackage{xcolor}         
\usepackage{nccmath}
\usepackage{algorithm}
\usepackage{algorithmic}
\usepackage{subcaption}
\usepackage{xurl} 
\usepackage{hyperref}
\usepackage{multirow}
\usepackage{placeins}
\usepackage{changepage}
\usepackage{enumitem}
\usepackage{wrapfig}
\usepackage{amsmath}
\usepackage[capitalize,noabbrev]{cleveref}
\crefname{equation}{eq.}{eqs.}
\Crefname{equation}{Eq.}{Eqs.}
\usepackage{amssymb}
\usepackage{mathtools}
\usepackage{macros}
\usepackage{amsthm}
\theoremstyle{plain}
\newtheorem{theorem}{Theorem}[section]
\newtheorem{proposition}[theorem]{Proposition}
\newtheorem{lemma}[theorem]{Lemma}

\theoremstyle{definition}
\newtheorem{definition}[theorem]{Definition}
\newtheorem{assumption}[theorem]{Assumption}
\theoremstyle{remark}
\newtheorem{remark}[theorem]{Remark}
\title{Federate the Router: Learning Language Model Routers with Sparse and Decentralized Evaluations}

\newcommand{\camera}[1]{{#1}}
\author{%
\makebox[\textwidth][c]{%
\parbox[t]{0.31\textwidth}{\centering
\textbf{Baris Askin}\thanks{Equal contribution. Correspondence to: Baris Askin \texttt{<baskin@andrew.cmu.edu>}}\\
{\normalfont Carnegie Mellon University}}%
\hfill
\parbox[t]{0.31\textwidth}{\centering
\textbf{Shivam Patel}$^{*}$\\
{\normalfont Carnegie Mellon University}}%
\hfill
\parbox[t]{0.31\textwidth}{\centering
\textbf{Anupam Nayak}$^{*}$\\
{\normalfont Carnegie Mellon University}}%
}%
\AND
\makebox[\textwidth][c]{%
\parbox[t]{0.42\textwidth}{\centering
\textbf{Andrea Vigano}\\
{\normalfont NVIDIA}}%
\hfill
\parbox[t]{0.42\textwidth}{\centering
\textbf{Jiin Woo}\\
{\normalfont Amazon}}%
}%
\AND
\makebox[\textwidth][c]{%
\parbox[t]{0.42\textwidth}{\centering
\textbf{Gauri Joshi}\\
{\normalfont Carnegie Mellon University}}%
\hfill
\parbox[t]{0.42\textwidth}{\centering
\textbf{Carlee Joe-Wong}\\
{\normalfont Carnegie Mellon University}}%
}%
}

\usepackage{fontawesome5}

\begin{document}

\maketitle

\begin{abstract}
Large language models (LLMs) are increasingly accessed as remotely hosted services by edge and enterprise clients that cannot run frontier models locally. Since models vary widely in capability and price, routing queries to models that balance quality and inference cost is essential. Existing router approaches assume access to centralized query--model evaluation data. However, these data are often fragmented across clients, such as end users and organizations, and are privacy-sensitive, which makes centralizing data infeasible. Additionally, per-client router training is ineffective since local evaluation data is limited and covers only a restricted query distribution and a biased subset of model evaluations. We introduce the first federated framework for LLM routing, enabling clients to learn a shared routing policy from local offline query--model evaluation data. Our framework supports both parametric multilayer perceptron router and nonparametric K-means router under heterogeneous client query distributions and non-uniform model coverage. Across three benchmarks, federated collaboration improves the accuracy--cost frontier over client-local routers, both via increased effective model coverage and better query generalization. Our theoretical results also validate that federated training reduces routing suboptimality.
\end{abstract}

\begin{center}
\small
\href{https://github.com/askinb/federated_llm_routers}{%
  \faGithub\enspace Code Repository%
}
\end{center}

\section{Introduction}

Large language models (LLMs) have become general-purpose tools for language-centric tasks, spanning customer support, search, programming assistance, and scientific writing. Their prompt-conditioned flexibility \cite{NEURIPS2020_1457c0d6_lm_are_few_shot_learners} allows a single model to support a broad and continually expanding set of applications \citep{openai2024gpt4technicalreport, bommasani2021opportunities, team2023gemini}. 
In many deployments, however, these capabilities are utilized by \emph{edge} devices and in client-facing environments (e.g., mobile devices, on-premises applications, enterprise endpoints) which are limited in terms of compute, memory, and power and thus cannot host most large language models locally \cite{10.1145/3719664}. As a result, queries are served by \emph{remotely hosted} models, which are accessed through application programming interfaces (APIs). 
Such remote hosting enables clients to select suitable models for individual queries by simply accessing different model APIs.

In practical scenarios, real-world queries are highly heterogeneous, spanning diverse tasks and varying substantially in difficulty. Some queries require sophisticated mathematical reasoning or long-context understanding \citep{ahn-etal-2024-large}, while others may demand specialized domain knowledge 
\citep{dehghani2025large}. 
Inference costs also vary widely across models, making the use of capable (and generally more expensive) models for simple queries inefficient \cite{zhang2025routerr}.
Furthermore, the query inference cost, which depends on the number of tokens in the response, is unknown prior to generating the response.
Consequently, no single model is uniformly best once both quality and cost are factored in, motivating mechanisms that utilize multiple models to deliver high-quality responses efficiently
\cite{woisetschlager2025mess}. 

Language query routing aims to map each incoming query to a suitable model from the pool of models, typically to optimize an accuracy-cost trade-off \cite{ong2025routellmlearningroutellms}. 
Routers are trained on historical query--model evaluations, i.e., measured quality (e.g., correctness) and associated costs
\cite{hu2024routerbench}. Given a query embedding (generated from a sentence encoder), \emph{parametric} routers (e.g.,  Multilayer Perceptron (MLP) predictors) \cite{ding2024hybridllmcostefficientqualityaware,ong2025routellmlearningroutellms} estimate per-model response accuracy 
and cost and then select the best model under their desired 
trade-off, while \emph{nonparametric} routers (e.g., K-Means clustering router) \cite{hu2024routerbench,jitkrittum2025universal} summarize evaluations in the query embedding space and select models using local neighborhood statistics. Despite their algorithmic differences, both families of routers rely on a diverse dataset of queries that are evaluated on many language models, allowing them to learn a reliable mapping of queries to model capabilities.

In practice, constructing a centralized query-evaluation dataset for router training is challenging. Firstly, evaluating queries on all models is prohibitively expensive, and this cost is exacerbated by rapidly evolving model pools, requiring continuous evaluation\footnote{\href{https://huggingface.co}{Hugging Face} had approximately $1\mathrm{M}$ models by the end of 2024 and over $2\mathrm{M}$ models by the end of 2025 \citep{huggingface_1m_models_2024, huggingface_2m_models_2025}.}. Second, even defining a ``broad'' query set is nontrivial as public benchmarks and datasets are not representative of real usage, and practical query distributions are long-tailed and dynamic \citep{arora2023reasoning, miller2025evaluatingllmmetricsrealworld}. Realistic and valuable evaluations come from individual users (\emph{clients}) themselves, i.e., from their actual queries/prompts and observed outcomes. 
These evaluations are naturally fragmented across {clients}, ranging from individual end users to organizations deploying LLM-backed systems (e.g., customer-support agents,  or internal coding copilots) \cite{brynjolfsson2025generative, maharaj2024evaluation}. Third, each client typically possesses queries from its own domain and evaluates them on only a small subset of the model pool, and model usage among that subset can be highly skewed leaving many models severely under-sampled \cite{zhao2024eagleefficienttrainingfreerouter, maharaj2024evaluation}. As a result, individual clients have limited data for learning a router that generalizes to a wide distribution of queries. 
Privacy and regulatory constraints prevent centralizing raw query-evaluation data such as personal/proprietary data, customer conversations and internal code \cite{sun2024improving, yao2025federatedlargelanguagemodels}.

The key insight of this paper is to propose federated learning (FL) \citep{fedavg} as a solution to these challenges. Rather than collecting client queries and their model evaluations in a central dataset,
federated learning allows clients to collaborate and transfer supervision across semantically similar requests by sharing routing policies they locally train on private offline datasets.
This enables training on diverse and privacy-sensitive query distributions, and additionally supports sparse and imbalanced client coverage over models. 

We develop frameworks spanning both canonical router methods in the federated setting: parametric \mlp{} and nonparametric \kmeans{}. 
We then demonstrate that these federated frameworks solve the challenges of learning routing policies from sparse, fragmented client data.
Empirically, federated routers consistently improve over routers trained on any single client’s data: they generalize better on aggregate query distributions and also improve routing performance on individual clients by leveraging cross-client signal without exposing private client queries.

\textbf{Contributions.}
After reviewing related work in FL and LLM routing (\Cref{sec:related_work}), we contribute the following:
\begin{itemize}[
    wide=0pt,
   topsep=-2pt,
    parsep=0pt,
    partopsep=0pt
]
    \item 
    In \Cref{sec:problem_formulation}, \textbf{we formulate LLM routing in a federated setting} that addresses practical problems: (i) prompts and responses are private and remain on clients, (ii) each client has limited query volume and a narrow, biased slice of the global prompt distribution, and (iii) query--model evaluations are sparse and imbalanced.
    To our knowledge, {this is the first work to study LLM routing through a federated learning framework}.

    \item
    In \Cref{sec:methods}, \textbf{we propose federated training procedures for both \mlp{} and \kmeans{}}, enabling collaborative learning of routing-policy,
    which supports evolving model and client pools.

    \item 
    In \Cref{sect:theory}, we establish convergence guarantees for federated optimization of \mlp{}, and \textbf{provide theoretical support} for why FL improves routing under sparse evaluations by increasing effective query and model coverage through cross-client aggregation for both routers.

    \item 
    In \Cref{sec:experiments}, \textbf{we empirically evaluate both router families under heterogeneity and incomplete query--model evaluation data}.
    Our proposed federated training improves client-local routers for both out-of-distribution (better generalization) and in-distribution (increased effective model coverage) data.
    For high heterogeneity regime, we also propose an \textit{adaptive personalization mechanism} that interpolates between the federated router and a client’s local router, improving robustness when global collaboration can be misaligned with a client’s own distribution. \Cref{sect:conc} concludes the paper, and additional discussion, experiments, and proofs are provided in the Appendix.

\end{itemize}
\vspace{3pt}

\section{Related Work}
\label{sec:related_work}

\paragraph{Federated Learning and LLMs}

Federated learning (FL) trains shared models from decentralized data while keeping raw client data local \citep{fedavg}. In practice, FL must handle \emph{statistical heterogeneity} as client data are non-iid, yet client distributions often share structure that collaborative training can exploit to improve generalization \cite{advances_and_open_fl}.
LLMs have advanced rapidly in scale and capability, making them the default backbone for many language-centric applications. However, adapting LLMs into FL is challenging because na\"{\i}vely communicating or training billions of parameters is prohibitive for many clients. Accordingly, federated LLM research largely focuses on resource-aware adaptation, for example, parameter-efficient and quantization-aware fine-tuning, federated NLP/LLM benchmarks, federated instruction tuning under heterogeneous clients, and hybrid split paradigms that offload parts of the model \citep{flora,ravan,lin2022fednlp,fedit,cheng2024towards,chen2024integration,zhao2024fedsllm}. 
In this work, we address a complementary problem: learning \emph{federated routing policies} that enable model selection among multiple LLMs without centralizing client data. Routing requires 
query--model evaluations which are often expensive and distributed across clients. FL enables clients to collaboratively learn routers that generalize well across heterogeneous prompts and improve model coverage while keeping data local.
\vspace{-10pt}
\paragraph{LLM Query Routing}
Query routing focuses on assigning a given query to the most suitable model to optimize a chosen tradeoff between response quality and inference cost. Most practical routers utilize a fixed dimensional query representation, generally obtained from pre-trained sentence encoders \cite{cer2018universalsentenceencoder, paraphrase_albert_reimers-2019-sentence-bert}. Given such representations for queries, one direction of routers estimates modelwise correctness and cost on queries via small neural predictors or other scoring models \cite{hu2024routerbench,ding2024hybridllmcostefficientqualityaware,ding2025bestrouteadaptivellmrouting,zhuang2025embedllm,sakota_2024_flyswat,huang2025routerevalcomprehensivebenchmarkrouting,lu2023routingexpertefficientrewardguided_zooter,treacle_thrifty_reasoning}.  
A complementary direction uses nonparametric designs that rely on the geometry of the embedding space and empirical performance statistics, e.g., clustering or nearest-neighbor selection \cite{60c19788-1128-3b5f-9275-2d63cc155832_knn_orig,a_kmeans_clustering_algorithm, hu2024routerbench, jitkrittum2025universal, stripelis2024tensoroperaroutermultimodelrouter, li2025rethinkingpredictivemodelingllm, zhang2025avengerssimplerecipeuniting, avengers_pro_zhang2025gpt5makingllmscheaper, patel2025proxrouterproximityweightedllmquery}. 
Prior works utilize full evaluation data where all queries are evaluated on all models, and discussion on router design using sparse and decentralized evaluation data is largely missing. We address this practical problem in our work across both parametric and nonparametric router paradigms.

\begin{figure*}
    \centering
    \includegraphics[width=0.96\linewidth]{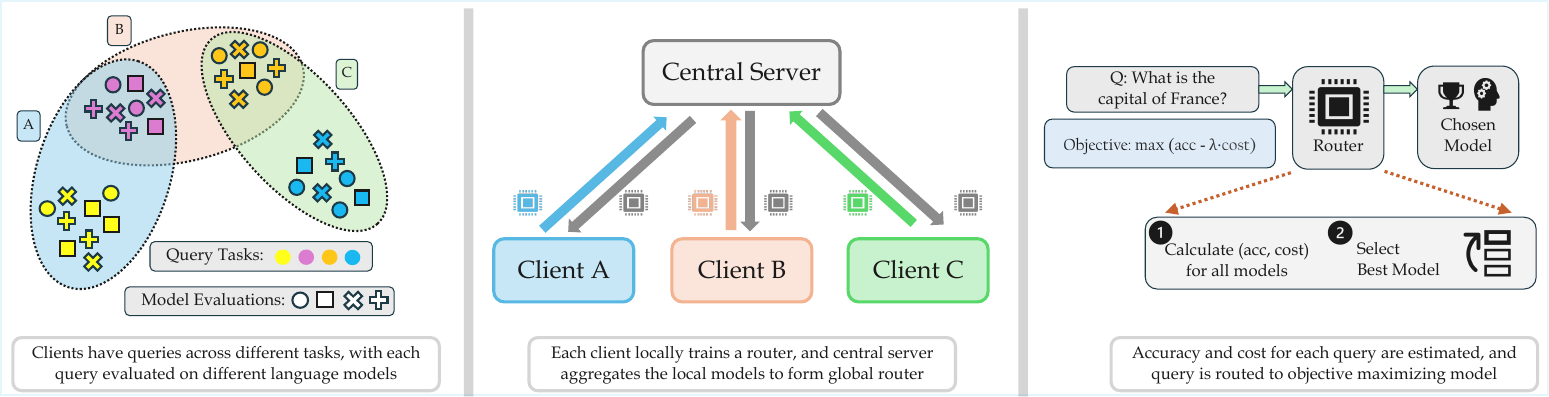}
    \caption{\textbf{Federated LLM Router System Description.} (i) Clients possess queries from different tasks, and each query is evaluated only on some LMs (in our experiments, we consider the extreme case of one LM per query, where the choice of LM is client-specific and can be highly non-uniform). (ii) Federated Learning setup where clients train their local router on local data, and central server aggregates these local models and returns global model to clients. (iii) Functioning of router, where modelwise accuracy and cost is calculated for the given query, which is routed to the most suitable model that maximizes the objective value. 
    }
    \label{fig:fedrouter_diagram}
\end{figure*}

\section{Problem Formulation}
\label{sec:problem_formulation}
We consider a pool of $M$ language models, denoted by $\mcal$.
Each query is a text prompt $s$ that we embed via a fixed, pretrained sentence encoder to obtain a representation\footnote{Throughout the text, we use terms prompt and query interchangeably to denote $\xb$.} $\xb \in \mathbb{R}^{d_{\text{emb}}}$, i.e., $\xb = \mathrm{Enc}(s)$ \cite{cer2018universalsentenceencoder}.
For each model $m \in \mcal$ and query embedding $\xb$, let $\ac(\xb,m) \in [0,1]$ denote the (expected) accuracy/performance score of the model $m$'s response on that query, and let $\cs(\xb,m)\in [0,\cm]$ denote the expected cost of responding to query, including input and output sequence costs \cite{artificialanalysis}.

Following the standard formulation \cite{hu2024routerbench,tsiourvas2025causal_llm_routing_regret_minimization},
we define the utility of prompt $\xb$ for model $m$ as its expected accuracy performance minus a penalty proportional to its expected cost (expectation taken over randomness of LLM responses):
\begin{equation}
    U_{\lmb}(\xb,m) := \Big(
    \ac(\xb,m)
    ~-~
    \lmb \, \cs(\xb,m)
  \Big),
  \label{eq:router_main}
\end{equation}
where the parameter $\lmb \ge 0$ controls the trade-off between cost and accuracy. Smaller $\lmb$ prioritizes accuracy, while larger $\lmb$ encourages cheaper models. A routing policy is a function that takes a prompt embedding as input and selects a corresponding model, written as $\pi : \mathbb{R}^{d_{\text{emb}}} \rightarrow \mathcal{M}$. Our objective is to learn routing policies that maximize the utility defined in \cref{eq:router_main} given a query $\xb$ and trade-off parameter $\lmb$.

In practice, $\ac(\xb,m)$ and $\cs(\xb,m)$ are unknown quantities that need to be estimated.
While one can directly
train a classifier that learns the routing policy (utility $U_\lmb$  maximizing model) for a fixed $\lmb$ \cite{ong2025routellmlearningroutellms,huang2025routerevalcomprehensivebenchmarkrouting}, we focus on modeling accuracy and cost functions via estimators 
$\aes(\xb,m)\!\approx\!\acc(\xb,m)$ and $\ces(\xb,m)\!\approx\!\cs(\xb,m)$
for all $m\!\in\!\mcal$, and then instantiate the router by plugging these estimates into
\Cref{eq:router_main}. We then choose the model that maximizes the estimated utility for a given query $\mathbf{x}$ as the routing policy.
This design is particularly natural in federated settings since different clients may have different budget/quality preferences, and can therefore choose $\lmb$ at inference time without retraining the router.

In the federated setting, we consider $N$ clients whose raw queries and embeddings are not shared, and only a \emph{single} language model is prompted for each query. Each client $i$ holds a local offline dataset $\dcal_i$ with $|\dcal_i| = \D_i$:
\begin{equation}
\dcal_i \;=\; \Bigl\{ \bigl(\xb_i^{(j)},\, m_i^{(j)},\, \hac_i^{(j)},\, \hcs_i^{(j)}\bigr) \Bigr\}_{j=1}^{\D_i},
\label{eq:local_dataset}
\end{equation}
where $\xb_{i}^{(j)}\!\in\!\mathbb{R}^{d_{\text{emb}}}$ is the embedding of the $j$-th local query and $m_{i}^{(j)} \in \mcal$ denotes the (single) model that is evaluated for  query $j$ at client $i$. $\hac_i^{(j)}$ and $\hcs_i^{(j)}$
are observed accuracy and cost samples 
corresponding to the evaluation, where $\E\!\left[\hac_i^{(j)} \mid \xb_i^{(j)}, m_i^{(j)}\right] = \ac\left(\xb_i^{(j)}, m_i^{(j)}\right)$ and $
\E\!\left[\hcs_i^{(j)} \mid \xb_i^{(j)}, m_i^{(j)}\right] = \cs\left(\xb_i^{(j)},m_i^{(j)}\right)$. 
We also denote the union of all local datasets by $\dcal=\bigcup_{i}\dcal_i$ with $|\dcal|=\D=\sum_i \D_i$.

\section{Proposed Federated Routing Methods}
\label{sec:methods}
In this section, we develop federated training and deployment pipelines for two common router families, parametric \mlp{} and nonparametric \kmeans{}. For both designs, we study how federated collaboration (\cref{fig:fedrouter_diagram}) mitigates (i) heterogeneous client query distributions and (ii) sparse and imbalanced query--model evaluations, while supporting personalization and dynamic participation. 

\subsection{Federated \mlp{}}
The goal of the federated setting is to learn global estimators that predict the expected accuracy and cost for all $m\!\in\!\mcal$ given a prompt $\xb$ and generalize across clients. We employ accuracy and cost estimators $\aesw$ and $\cesw$, parametrized by $\mlpw\!\in\!\wcal$. Specifically, we use a single \mlp{} consisting of a shared MLP trunk $h_{\mlpw}\!:\!\mathbb{R}^{d_{\text{emb}}}\!\to\!\mathbb{R}^{d_h}$ that maps a query embedding $\xb$ to a hidden representation $h_{\mlpw}(\xb)$, followed by model-specific linear heads which output
$\aesw(\xb,m)$ and $\cesw(\xb,m)$ for each $m\!\in\!\mcal$  from $h_{\mlpw}(\xb)$.
We denote the collection of all trunk and head parameters by $\mlpw$ and define the global empirical loss function $\mathcal{L}(\mlpw)$ given by,
\begin{align}
\mathcal{L}(\mlpw)=
\frac{1}{|\dcal|}&\sum_{i=1}^N \sum_{j=1}^{|\dcal_i|}
\Big[
\loss\lp
\hac(\xb_i^{(j)},m_i^{(j)}),
\aesw(\xb_i^{(j)},m_i^{(j)})
\rp
+~
\loss\lp
\hcs(\xb_i^{(j)},m_i^{(j)}),
\cesw(\xb_i^{(j)},m_i^{(j)})
\rp
\Big]. \label{eq:fed_objective}
\end{align}

\begin{wrapfigure}{r}{0.4\textwidth}
\vspace{-2em}
\begin{minipage}{0.4\textwidth}
\begin{algorithm}[H]
\caption{Federated \mlp{} training}
\label{alg:mlp_fl}
\small
\begin{algorithmic}[1]
\STATE {\bfseries Input:} clients $\{\dcal_i\}_{i=1}^N$, rounds $T$, local steps $\tau$, LR $\eta$, init global model. $\mlpw^0$
\FOR{$t = 0$ {\bfseries to} $T-1$}
    \STATE Send $\mlpw^t$ to active clients $\mathcal{S}^{(t)}$
    \FOR{client $i \in \mathcal{S}^{(t)}$ {\bfseries in parallel}}
        \STATE Set local model $\mlpw_i \leftarrow \mlpw^t$
        \FOR{$s = 1$ {\bfseries to} $\tau$}
            \STATE $\mlpw_i \leftarrow \mlpw_i - \eta \widetilde{\nabla}\mathcal{L}_i(\mlpw_i)$
        \ENDFOR
        \STATE Send updated model $\mlpw_i$ to server
    \ENDFOR
    \STATE $\displaystyle
    \mlpw^{t+1} \leftarrow
    \frac{\sum_{i\in\mathcal{S}^{(t)}}|\dcal_i|\mlpw_i}
    {\sum_{i\in\mathcal{S}^{(t)}}|\dcal_i|}$
\ENDFOR
\STATE {\bfseries Output:} parameters $\mlpw^T$
\end{algorithmic}
\end{algorithm}
\end{minipage}
\vspace{-3.5em}
\end{wrapfigure}

We train \mlp{} using Federated Averaging (FedAvg) \citep{fedavg} as depicted in \cref{alg:mlp_fl}.
At each communication round, the server broadcasts the current global weights $\mlpw^t$ to the active clients.
Each client $i$ minimizes its empirical objective
${\mathcal{L}}_i(\mlpw)
= \frac{1}{|\dcal_i|}\sum_{(\xb,m)\in\dcal_i}
\big[
\loss(\hac(\xb,m), \aesw(\xb,m))
+
\loss(\hcs(\xb,m), \cesw(\xb,m))
\big]$
using $\tau$ steps of local stochastic gradient descent on its private dataset $\dcal_i$,
and returns the updated parameters to the server. To quantify sample-wise discrepancies, we adopt Mean Squared Error (MSE) as the loss function $\loss$. The server then aggregates the client updates via weighted averaging to produce the next global model. 
Finally, the $\mlp{}$ induced by a parameterization $\mlpw$ is defined as the
routing policy that maximizes the utility approximated by the estimators

\noindent
\begin{minipage}{\dimexpr 0.6\textwidth-\columnsep\relax}
\begin{equation}
    \pi_{\mlpw,\lmb}(\xb)
    =
    \arg\max_{m\in\mcal}
    U_{\lmb}(\xb,m;\mlpw),
    \label{eq:mlpind}
\end{equation}
\end{minipage}

where
$U_{\lmb}(\xb,m;\mlpw)
= \aesw(\xb,m) - \lmb \cesw(\xb,m)$.

\begin{wrapfigure}{r}{0.6\textwidth}
\vspace{-2em}
\hspace{-0.5em}
\begin{minipage}{0.58\textwidth}
\begin{algorithm}[H]
\caption{Federated \kmeans{} training}
\label{alg:kmeans_fl}
\small
\begin{algorithmic}[1]
\STATE {\bfseries Input:} clients $\{\dcal_i\}_{i=1}^N$, nb. of local clusters $\klocal$, nb. of global clusters $\kglobal$, distance metric $d(\cdot,\cdot)$
\FOR{each client $i \in \{1,\ldots,N\}$ {\bfseries in parallel}}
    \STATE Run $k$-means on local embeddings $\{\xb\}_{\xb \in \dcal_i}$ 
    \STATE Send $\{(\mu_{i,j}, n_{i,j})\}_{j=1}^{\klocal}$ obtained in line 3 to server
\ENDFOR
\STATE Server runs weighted $k$-means on $\bigcup_i \{(\mu_{i,j}, n_{i,j})\}_{j=1}^{\klocal}$
\STATE Broadcast $\{\mu_k\}_{k=1}^{\kglobal}$ obtained in line 6 to all clients
\vspace{0.25em}
\FOR{each client $i \in \{1,\ldots,N\}$ {\bfseries in parallel}}
    \FOR{each cluster $k \in [\kglobal]$ and model $m \in \mcal$}
        \STATE Compute local counts $n^{(m)}_{i,k}$ and local averages:
{\scriptsize
\vspace{-1.5em}
\begin{align*}
\bar{a}^{(m)}_{i,k}
= \mfrac{1}{n^{(m)}_{i,k}}
\!\!\sum_{\substack{\xb \in \dcal_i \\ k_{\mu}(\xb)=k}}\!\!
\hac(\xb, m),
\bar{c}^{(m)}_{i,k}
= \mfrac{1}{n^{(m)}_{i,k}}
\!\!\sum_{\substack{\xb \in \dcal_i \\ k_{\mu}(\xb)=k}}\!\!
\hcs(\xb, m)
\end{align*}
\vspace{-1.5em}
}
    \ENDFOR
    \STATE Send $\{(\bar{a}^{(m)}_{i,k}, \bar{c}^{(m)}_{i,k}, n^{(m)}_{i,k})\}_{m,k}$ to server
\ENDFOR
\STATE Server aggregates for all $m,k$:
\vspace{-0.7em}
\begin{align*}
\aest[(m)][k] 
\leftarrow 
\mfrac{\sum_i n^{(m)}_{i,k}\,\bar{a}^{(m)}_{i,k}}
{\sum_i n^{(m)}_{i,k}}, \quad
\cest[(m)][k] 
\leftarrow 
\mfrac{\sum_i n^{(m)}_{i,k}\,\bar{c}^{(m)}_{i,k}}
{\sum_i n^{(m)}_{i,k}}.
\end{align*}
\vspace{-0.7em}
\STATE Output $\kmw
=
\left\{
\{\mu_k\}_{k=1}^{\kglobal}
\text{, }
\left\{
\aest[(m)][k], \cest[(m)][k]
\right\}_{\substack{m\in \mcal\\k \in [\kglobal]}}
\right\}$

\end{algorithmic}
\end{algorithm}
\end{minipage}
\vspace{-1em}
\end{wrapfigure}

\subsection{Federated \kmeans{}.}
\kmeans{} provides a simple, training-free estimator for routing utilities by exploiting locality in the query-embedding space. Concretely, the method partitions embeddings into a finite set of regions (Voronoi cells) and treats the per-model accuracy and cost as approximately constant within each region. This yields a piecewise-constant nonparametric estimator. For a new query, we map it to its nearest cluster center and then use the cluster empirical averages of estimated accuracy and cost for each model. We use a text encoder function $\mathrm{Enc}(\cdot)$ to group semantically similar prompts \cite{cer2018universalsentenceencoder,paraphrase_albert_reimers-2019-sentence-bert}.

In this section, we adopt slightly different notation, since the accuracy ($\aest[(m)][k]$) and cost estimators ($\cest[(m)][k]$) operate on clusters rather than directly on prompts. Our federated training of $\kmeans$ is elaborated in \Cref{alg:kmeans_fl}. First, each client runs local K-means \citep{lloyd} with $\klocal$ centers on its own embeddings (line 3) and transmits the resulting centroids along with cluster sizes (line 4). The server performs weighted K-means  (i.e., each centroid is weighted in the server’s objective by its local cluster size, so the global centers minimize a size-weighted squared distance) over those centroids to obtain $\kglobal$ global centers (line 6), which are broadcast back to clients (line 7). Using the received global centers, each client assigns its local samples to clusters (line 9), computes the average observed accuracy and cost for every (cluster, model) pair, and sends them to the server along with the number of evaluated samples (lines 10-13). No information is sent for (cluster, model) pairs with no samples. The server then takes weighted average of these statistics to obtain global estimators (line 14). At inference time, accuracy and cost statistics of closest cluster center to a query embedding $\xb$ are used in \Cref{eq:router_main} for routing. Both local and global K-means use Euclidean distance. $\klocal$ and $\kglobal$ are chosen with validation experiments.

Finally, given the 
$\kmw\!=\!\left\{\mu_{1:\kglobal},\,
\left\{\aest[(m)][k],\,\cest[(m)][k]\right\}_{
\raisebox{0.4ex}{$\scriptstyle m\in\mcal,\; k\in[\kglobal]$}
}\right\}$\footnote{We denote the output as a parameterization using $\kmw$.}
cluster centers and cost accuracy estimates,  
define the assigned clusters $k_\mu(\xb) = \arg\min_{k \in [\kglobal]} \lVert \xb - \mu_k\rVert_2$, and then the induced router is given as
\begin{equation*}
    \pi_{\kmw,\lmb}(\xb) := \arg\max_{m\in\mcal} \left(\aest[(m)][k_\mu(\xb)] - \lambda \cest[(m)][k_\mu(\xb)] \right).
\end{equation*}

\textbf{Comparison of both methods.}
Both parametric and nonparametric routers have been examined in literature, showing complementary strengths and weaknesses. 
Parametric routers (like learned MLPs)
can capture complex model behavior over a broad range of queries, improving generalization to unseen queries. In federated settings, we expect them to perform better due to the increased “collective coverage" by diverse clients.
Alternatively, nonparametric routers (like clustering based routers) \cite{jitkrittum2025universal,hu2024routerbench,patel2025proxrouterproximityweightedllmquery} are naturally incremental by design, eliminating the need for adapting or retraining router to accommodate new language models or queries. New language models can be onboarded by evaluating on the existing pool of queries and calculating cluster level averages, new queries can be incorporated through re-clustering and updating query embeddings. In contrast, a parametric router needs new head training or tuning parameters to adapt. 

\section{Theoretical Results}
\label{sect:theory}
In this section, we provide informal statements of theoretical results to support our claims concerning the effectiveness of federated routing.
\subsection{Federated \mlp{}}
\begin{theorem}[Convergence: \Cref{alg:mlp_fl} - Informal] Under standard bounded heterogeneity, unbiased gradient, and smoothness assumptions \cite{wang2020tackling,koloskova2020unified}, with step size
$\eta = \Theta\!\bigl(\sqrt{N/(\tau T)}\bigr)$, the empirical risk objective in \Cref{eq:fed_objective} converges at the rate
\[
\min_{t\in\{1,\dots,T\}} \E\!\left[\bigl\|\nabla {\mathcal{L}}(\mlpw^{(t)})\bigr\|^2\right]
\le
\tilde{\mathcal{O}}\!\left(\frac{1+A\sigma^2}{\sqrt{N\tau T}}\right).
\]
Here, $\tilde{\mathcal{O}}$ hides the faster decaying $1/T$ terms and $\sigma^2$ denotes the upper bound on heterogeneity in mini-batch gradients at a client, and the constant $A := N\sum_{i=1}^N(\D_i/\D)^2$. 
\end{theorem}

The analysis is similar to \citet{wang2020tackling}.
When all clients have the same number of data points, i.e. $\D_i/\D = 1/N\;\forall i\in [N]$, we obtain a linear speedup of $N$ (the number of clients). Moreover, the convergence rate also scales with the number of local steps $\tau$.

\begin{definition}[Suboptimality]
    We use $\pi^{\star}:\X\rightarrow \mcal$ to denote the optimal router policy defined as 
    \begin{align}
        \pi^\star(\xb) = \arg\max_{m\in\mcal} U_{\lmb}(\xb,m),
    \end{align}
     define the suboptimality of a routing policy $\widehat \pi$ on a  query distribution $\dt$ 
    and tradeoff parameter $\lmb$,
    \[
\mathrm{Subopt}(\widehat\pi)
:=
\E_{\xb\sim\dt}\!\left[U_{\lmb}\!\bigl(\xb,\pi^\star(\xb)\bigr)-U_{\lmb}\!\bigl(\xb,\widehat \pi(\xb)\bigr)\right].
\]
\end{definition}

\begin{theorem}[Suboptimality - Informal]
\label{thm:sampmaintext}
Let $\widehat{\mlpw}_i$ be the parameters of router learnt on the local dataset $\dcal_i$ and $\widehat{\mlpw}_{\mathrm{fed}}$ be the parameters of the federated router on $\dcal=\bigcup_{j=1}^N \dcal_j$. Let $\pi_{\mlpw,\lmb}$ denote router induced by a parameterization $\mlpw$ for tradeoff parameter $\lmb$ (\Cref{eq:mlpind})  
Then, under standard realizability assumptions \citep{zhang2023mathematical, foster2023foundations}
we have with probability  $1-\delta$,
{\small
\begin{align*}
\hspace{-1em}\mathrm{Subopt}\left(\pi_{\widehat{\mlpw}_i,\lmb}\right)
&\le
C'\,\Gamma(\dcal_i)\,
\sqrt{\log\!\left(\frac{2\,\mathcal N_{\mathcal F}(1/\D_i)}{\delta}\right)},\;\;\;\\
\mathrm{Subopt}\left(\pi_{\widehat{\mlpw}_{\mathrm{fed}},\lmb_i}\right)
&\le
C'\,\Gamma(\dcal)\,
\sqrt{\log\!\left(\frac{2\,\mathcal N_{\mathcal F}(1/\D)}{\delta}\right)}.
\end{align*}
}
Where $C' = c\max\{1,\lmb\}\cm\;$ for some universal constant $c$, 
$\mathcal{N}_{\mathcal{F}}$ is the covering number \citep{wainwright2019high} of the estimator class and $\Gamma(\dcal_{\text{in}})$ denotes the coverage coefficient of the dataset $\dcal_{\text{in}}$ for the test distribution at client $i$.
\end{theorem}

A formal statement of the theorem, along with the relevant definitions and an interpretation of the coverage coefficient, is provided in \Cref{sec:theoryapp}, \Cref{thm:mlp_subopt_bounds}. The coverage coefficients of a dataset of size $\D'$ decays roughly as $1/\sqrt{\D'}$ and the covering number grows as $\log(\D')$. Since the federated router optimizes on the dataset $\dcal = \bigcup_{i=1}^N\dcal_i$ we get a lower suboptimality against local routing $\tilde{\mathcal{O}}(1/\sqrt{\D_i})$ vs $\tilde{\mathcal{O}}(1/\sqrt{\D})$.

\begin{remark}
    \Cref{thm:sampmaintext} extends to each client’s test distributions $\dt_i$ when routers are induced using learned parameters and local tradeoff coefficients~$\lmb_i$. A stronger version of it is in \Cref{sec:theoryapp}.
\end{remark}

\subsection{Federated \kmeans{}}
    \begin{theorem}[Suboptimality-Informal]
        Let $\widehat \kmw$ be the output of \Cref{alg:kmeans_fl}. Under mild smoothness of utility over the embedding space and uniform model logging  assumptions, the suboptimality of the router $ \pi_{\widehat\kmw,\lmb}$ induced by the output from \Cref{alg:kmeans_fl} on the global test distribution $\dt$ is upper bounded as
        \begin{align*}
\Subopt\left( \pi_{\widehat\kmw,\lmb}\right)
&~\le~
2L_\lmb\,\E_{\xb\sim\distr^{\mathrm{test}}}\!\big[\lVert \xb - \mu_{\widehat\kmw}(\xb)\rVert_2\big]
 ~+~
2L_\lmb\,\Delta_{\max}
~+~
2B_\lmb\sqrt{\frac{\log\!\big(\tfrac{2\kglobal|\mcal|}{\delta}\big)}{2\,n_{\min}}},
\end{align*}
where $B_\lmb = 1 + \lmb\,\cm$, and
the term $\mu_{\widehat\kmw}(\xb)$ denotes the embedding of the global cluster center associated with $\xb$ using ${\widehat\kmw}$.
The constant $L_\lmb$ is the smoothness parameter of the utility in the embedding space and $n_{\min} = \min_{k,m}\; n_{k,m}$, where $n_{k,m}$ is the number of points in global cluster $k$ evaluated on model $m$ in the training dataset. Finally, $\Delta_{\max}$ quantifies the maximum discrepancy across clusters between the training and test distributions.
    \end{theorem}
    The formal version of the theorem and its proof are presented in \Cref{sec:theoryapp}, \Cref{thm:kmeans_regret}.
    Federated training will generally result in a higher $n_{\min}$ as a result of more datapoints $n_{i,k,m} = \sum_i n_{k,m}$. Moreover, better ``coverage" of the input embedding space in the federated dataset leads to smaller first term $\E_{\xb\sim\distr^{\mathrm{test}}}\!\big[\lVert \xb - \mu_{\widehat\kmw}(\xb)\rVert_2$ in federated routing. Hence, the federated methods incur a lower suboptimality compared to a local $\kmeans{}$. 
\section{Experiments}
\label{sec:experiments}
We evaluate our federated routers on $\routerbench$ \citep{hu2024routerbench} (results are in the main text), $\proxr$ \citep{patel2025proxrouterproximityweightedllmquery} and SPROUT dataset \citep{somerstep2025carrot} (results are in the appendix).
\routerbench{} contains evaluations of 11 public and proprietary LLMs over 8 datasets, \proxr{} contains 14 public LLMs over 12 datasets, and SPROUT dataset contains evaluations of 13 public and proprietary LLMs over 6 datasets. 
In our experiments, we use \text{all-mpnet-base-v2} by \cite{song2020mpnetmaskedpermutedpretraining} as the query encoder $\mathrm{Enc}(\cdot)$. 
We also experiment with two other encoder models in \Cref{apndx:different_enc_results}, and report no significant performance difference in centralized experiments. We detail the experimental settings in \Cref{app:exp_setting_details}.

\textbf{Federated simulation protocol.}
We simulate a FL setting with $N=10$ clients and a partial participation rate of $0.6$ fraction of clients per round.
To induce statistical heterogeneity, we partition \routerbench{} across clients using Dirichlet distribution with concentration parameter $\alpha=0.6$ over task labels, following \cite{dirichlet}.
To reflect realistic client behavior, we assume that each query is evaluated by {a single} model, i.e., each training sample contains feedback for one $(\text{query},\text{model})$ pair. 
The models observed in client datasets are non-uniformly distributed, so some models receive fewer supervision signals than others within same $\dcal_i$. Details on the FL setup, a t-SNE visualization for query heterogeneity and a bubble plot for model heterogeneity are in \Cref{apndx:dataset_details}. 

In the results, we report \emph{accuracy--cost trade-off} curves from sweeping the routing trade-off parameter $\lambda$, which controls the relative preference between selecting cheaper models versus higher-quality models.
Larger $\lambda$ prioritizes cost reduction while smaller $\lambda$ prioritizes accuracy by routing more queries to typically stronger and  more expensive models. The cost is reported as the average price across all test queries based on the chosen routed models. 
To summarize each curve with a scalar, we additionally report normalized area under the curve (AUC), by integrating accuracy as a function of cost over the sweep and normalizing by the cost range. A higher AUC indicates a better frontier.

\begin{wrapfigure}{r}{0.5\textwidth}
\vspace{-0em}
  \centering
  \includegraphics[width=0.48\textwidth]{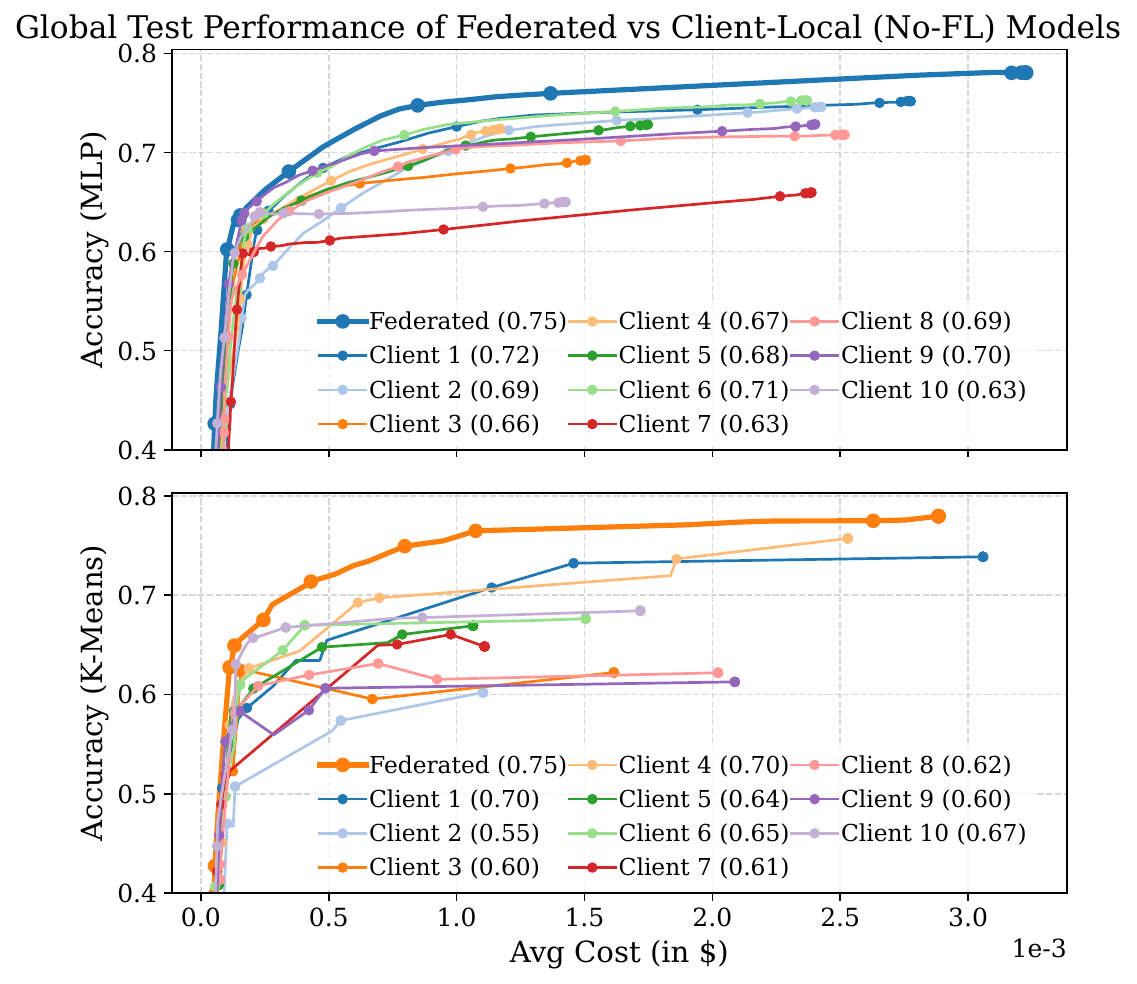}
  \caption{\textbf{Federated vs.\ client-local routers on the global test distribution.}
  Accuracy--cost curves for \mlp{} and \kmeans{} by sweeping $\lmb$.
  AUC scores are shown in parentheses.}
  \label{fig:global_vs_local_only_global_test}
\vspace{-2em}
\end{wrapfigure}

\subsection{Federated Learning Improves Generalization with Better Query Distribution Coverage}
In practice, each client typically observes a small number of training queries, and its local query distribution may cover a narrow region of the global query space.
As a result, routers trained purely on local data can fail to generalize to the broader population distribution.
To quantify this effect, we compare a federated router trained collaboratively across clients against \emph{client-local} (no-FL) routers trained independently on each client, where both are evaluated on a global test set.
\Cref{fig:global_vs_local_only_global_test} shows that federated training consistently yields a better accuracy--cost trade-off for both \mlp{} and \kmeans{}.
The gains are particularly pronounced for \kmeans{}, since clustering with limited and heterogeneous local samples produces unstable centers and poor global generalization, whereas federated collaboration effectively increases query coverage without requiring central access to raw queries.
\camera{
Across random seeds, the global-test AUCs (mean $\pm$ standard
error) are $0.750\pm0.003$ versus $0.683\pm0.003$ for federated
versus client-local \mlp{}, and $0.747\pm0.003$ versus
$0.655\pm0.011$ for \kmeans{}, supporting the consistency of
these gains (\Cref{apndx:additional_seed_results}).
}
\camera{
We also train the same router families on the centralized union of client training datasets for comparison.
\Cref{fig:global_vs_centralized} shows closely overlapping frontiers,
indicating federated router matches the performance of centralized routers.
}

\begin{wrapfigure}{r}{0.52\textwidth}
\vspace{-1.5em}
  \centering
  \includegraphics[width=0.52\textwidth]{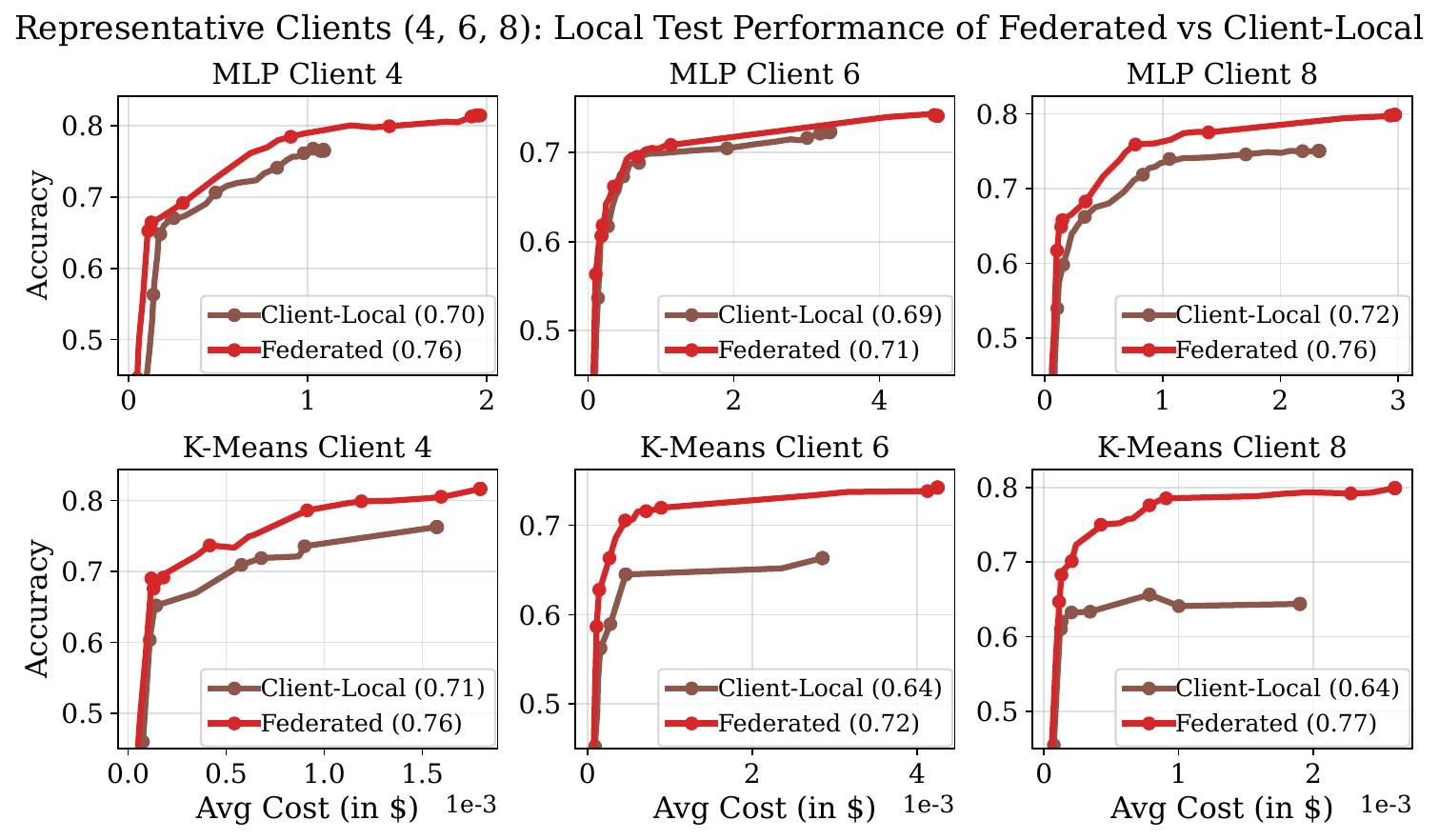}
  \caption{\textbf{Federated vs.\ client-local routers evaluated on local test sets.}
  Representative clients for \mlp{} and \kmeans{} are shown. Federated training improves the in-distribution accuracy--cost frontier via better effective model coverage.}
  \label{fig:local_vs_global_local_test_selected_clients}
\vspace{-1em}
\end{wrapfigure}
\vspace{-3pt}
\subsection{FL Even Improves In-distribution Local Performance via Better Model Coverage}
To reflect practice, we assume that each training query is evaluated by one model, rather than exhaustively by all models, and the model distribution in a client training set is non-uniform.
This sparsity and imbalance implies that a single client may have limited {model coverage} even for queries drawn from its own distribution, making local routers noisy and biased.
Federated learning mitigates this by learning more accurate estimators for models with collaboration across clients.
As shown in \Cref{fig:local_vs_global_local_test_selected_clients}, the federated router improves the accuracy--cost trade-off \emph{even when evaluated on clients' local test set}, indicating that collaboration helps clients learn better routing in-distribution via more effective model coverage.
The results for all clients are in \Cref{apndx:all_clients_local_tests}.

\begin{figure}[t]
  \centering
  \begin{minipage}[t]{0.46\linewidth}
    \vspace{0pt}
    \centering
    \includegraphics[width=\linewidth]{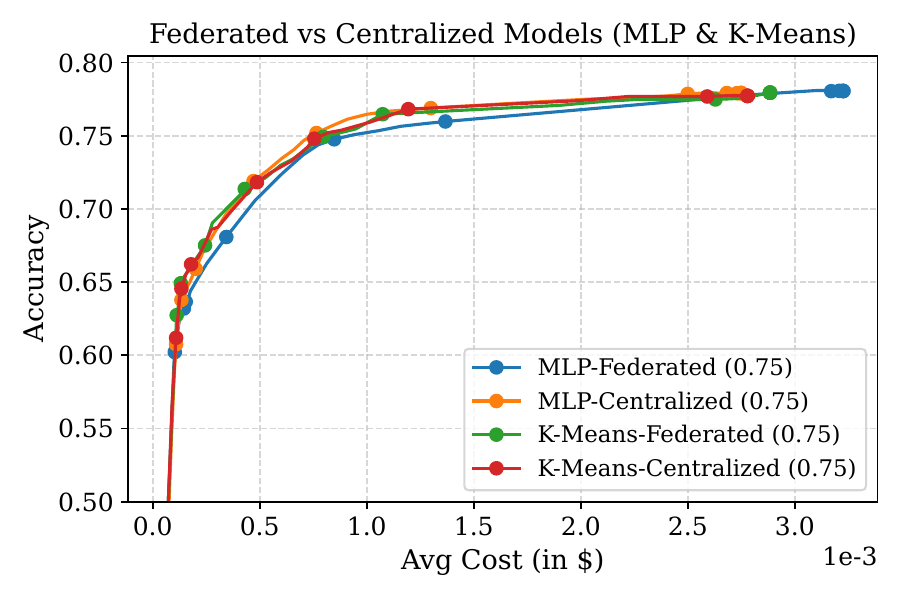}
    \caption{\camera{\textbf{Federated vs.\ centralized routers.}
    Centralized routers use pooled client training data.
    Global-test frontiers closely overlap; parentheses denote
    normalized AUC.}}
    \label{fig:global_vs_centralized}
  \end{minipage}
  \hfill
  \begin{minipage}[t]{0.5\linewidth}
    \vspace{0pt}
    \centering
  \includegraphics[width=\linewidth]{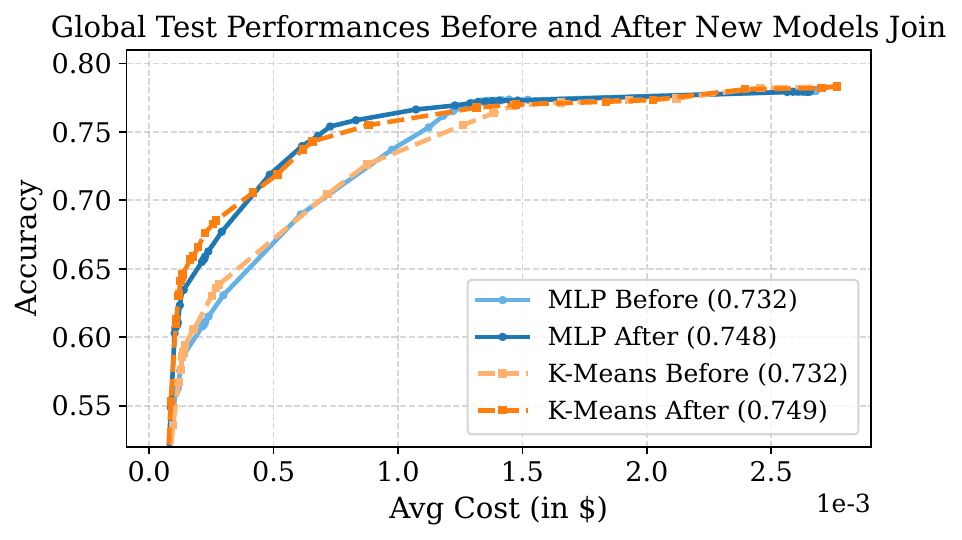}
  \vspace{-0.5em}
  \caption{\textbf{Adapting to new models.}
  Global test frontiers for routers trained with three models withheld and then incorporated via a lightweight calibration step.}
  \label{fig:new_model_in_the_system}
  \end{minipage}
\end{figure}

\subsection{When New Models are Added to the Pool}
\label{sect:new_models_join}
In deployed systems, the set of available LLMs is not static since new models are introduced every day, requiring routers to support \emph{cheap} extension without full retraining \cite{feng2025graphroutergraphbasedrouterllm,jitkrittum2025universal}.
Both \mlp{} and \kmeans{} naturally support such expansions.
For \mlp{}, to add a new model, we append a new head and train only that head by keeping the trunk and existing heads frozen.
For \kmeans{}, since the query embedding space is unchanged, the introduction of a new model reduces to estimating its accuracy \& cost statistics over existing regions of the embedding space.
We simulate a scenario where three models are withheld during initial router training and introduced later.
Upon introduction, each client evaluates the new models on a small calibration subset (10\% of its prompts) to estimate the required per-model statistics (for \kmeans{}) or to train the newly added heads (for \mlp{}).
\Cref{fig:new_model_in_the_system} shows the accuracy--cost trade-off before and after expansion on the global test distribution.
Both routers benefit from the enlarged model pool after a lightweight adaptation without re-training the full router.

\begin{wrapfigure}{r}{0.54\textwidth}
\vspace{-1em}
  \centering
  \vspace{-0.5em}
  \includegraphics[width=\linewidth]{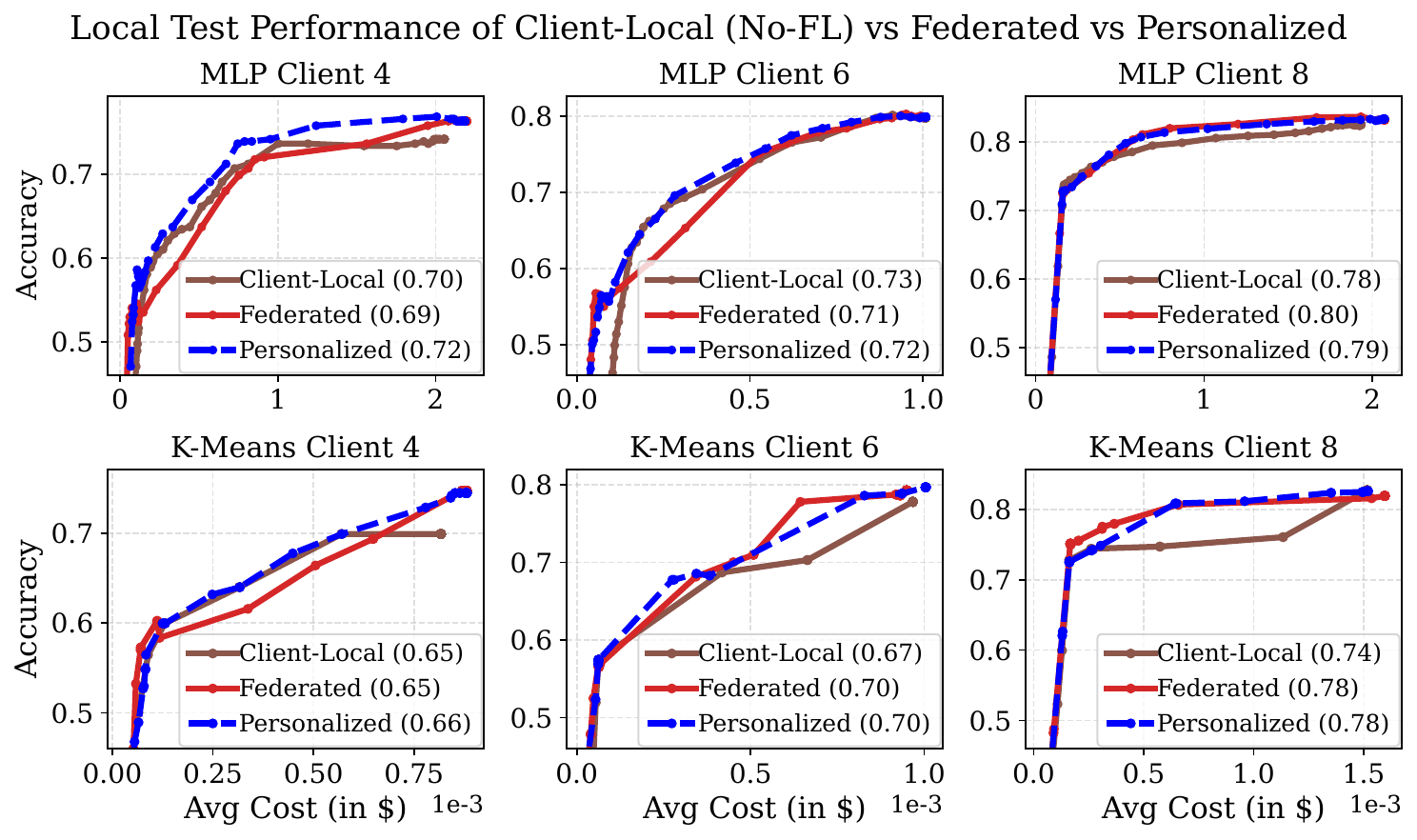}
  \caption{\textbf{Adaptive personalization under extreme heterogeneity.} Evaluated on local test sets of clients with federated training, isolated local training, and the proposed adaptive mixture.}
  \label{fig:adaptive_personalization_main}
  \vspace{-1em}
\end{wrapfigure}
\vspace{-5pt}
\subsection{Adaptive Personalization for High Heterogeneity}
\label{sec:adaptive_personalization}

Under extreme client data heterogeneity, a federated router can underperform on some clients' {local} distributions, even on test sets.
For instance, if a client mostly has biology prompts while other clients rarely observe it, the federated router may be poorly trained for that client, making a local router preferable despite the reduced model-coverage it suffers from.
To address this, we propose a lightweight \emph{adaptive personalization} that interpolates between the federated router and the client's locally trained router based on their empirical calibration errors on the client's existing data points.
Assume each client $i$ maintains the global federated ($\aest[(m)]$, $\cest[(m)]$) and locally trained accuracy and cost estimators ($\aest[(i,m)]$, $\cest[(i,m)]$).

After training, client $i$ calculates the mean absolute error (denoted as operator $e(\cdot)$) of its training samples (no additional model calls) for the accuracy and cost of each model for both routers. 
Then, at inference time, for every model $m\in\mcal$, the client calculates the weights of the federated and local estimators inversely proportional to their errors separately for accuracy and cost,
\[w^{(i,m)}_a=\mfrac{e(\aest[(m)])}{e(\aest[(m)])+e(\aest[(i,m)])}, \quad w^{(i,m)}_c=\mfrac{e(\cest[(m)])}{e(\cest[(m)])+e(\cest[(i,m)])}. \]
Then, they weight estimates of local and federated estimators for accuracy by $ w^{(i,m)}_a \aest[(i,m)]+(1-w^{(i,m)}_a)\aest[(m)]$, and cost by $ w^{(i,m)}_c \cest[(i,m)]+(1-w^{(i,m)}_c)\cest[(m)]$ to plug in the routing rule in \Cref{eq:router_main} in the inference time. 
\camera{
Relative local errors serve as a heuristic aggregation signal for this personalized extension targeting settings with extreme heterogeneity. We do not assume that training error predicts test AUC.
}

We observe that reusing the same training data points for this calibration does not result in underperformance in our experiments.

We repeat the main experiment under extreme heterogeneity, using Dirichlet parameter $\alpha=0.03$.
\Cref{fig:adaptive_personalization_main} reports local-test accuracy–cost frontiers for representative clients, comparing federated routers, local routers, and our adaptive personalization.
Under this extreme heterogeneity, federated \mlp{} can underperform local routers for some clients.
In contrast, adaptive personalization generally matches or improves upon both by combining broader model coverage with client-specific correction.
For \kmeans{}, federated routing remains competitive with local routers.
We attribute this to the robustness of K-means under heterogeneity \cite{dennis2021heterogeneity}, while local \kmeans{} suffers under low coverage since it estimates per-centroid, per-model statistics from limited client data.
Additional experimental results are in \Cref{sect:app_routerb_exp}. \Cref{apndx:new_clients_join} discusses the adaptation of new clients joining the system after initial training. 
We also compare with a federated adaptation of EmbedLLM~\citep{zhuang2025embedllm}, a recent centralized LLM routing method, in \Cref{apndx:embedllm_comparison}, showing that FL benefits another router architecture beyond our proposed \mlp{} and \kmeans{} designs. 
\camera{
Ablations in \Cref{apndx:client_count,apndx:query_heterogeneity} show that federated router's gains persist across settings with different number of clients and query heterogeneity.
Additional experiments on SPROUT dataset \citep{somerstep2025carrot}, where the accuracy signal is graded quality of the model response, are presented in \Cref{apndx:sprout_results}.  
}

\section{Conclusion}
\label{sect:conc}

To the best of our knowledge, this is the first work to tackle \emph{sparse} and \emph{decentralized} query--model evaluation data for training language query routers. We examine the practical scenario where prompts and evaluations are local to clients, and query--model coverage is incomplete and imbalanced across the model pool. We propose federated training as a solution across both canonical routing paradigms: a parametric \mlp{} and a nonparametric \kmeans{}, enabling collaborative learning from clients' offline data. Empirically, federated routing improves the accuracy--cost frontier over client-local baselines on both global and client-level test distributions, and adaptive personalization further improves robustness under extreme heterogeneity. We validate our empirical observations through theoretical results characterizing how federation reduces routing suboptimality compared to client-local routers. Overall, federated learning offers a practical foundation for training LLM routers from privacy-sensitive, fragmented data. Online routing, where the router is updated after deployment with new query evaluations, remains a promising direction for future work.

\camera{
\section*{Acknowledgements}
This work was partially supported by the US Department of Energy under grant DESC0025652 and the US National Science Foundation under grants CNS-2409138 and CNS-2533813 to CJW. The work was also partially supported by the AI2C Seed grant and the US National Science Foundation under grants CCF 2045694, CNS-2112471, CPS-2111751, and ONR N00014-23-1-2149 to GJ.
This work used PSC Bridges-2 GPU at Pittsburgh Supercomputing Center through allocation CIS250087 from the Advanced Cyberinfrastructure Coordination Ecosystem: Services \& Support (ACCESS) program, which is supported by US National Science Foundation grants \#2138259, \#2138286, \#2138307, \#2137603, and \#2138296.
}

\bibliographystyle{plainnat}
\bibliography{ref}

\newpage
\appendix
\crefalias{section}{appendix}
\crefalias{subsection}{appendix}
\crefalias{subsubsection}{appendix}

\onecolumn

\section*{Impact Statement}
This paper presents work whose goal is to advance the field of Machine Learning. In particular, we study federated training of {LLM routers} (routing policies) to reduce the cost of LLM use while preserving client privacy by keeping prompts and model-evaluation signals decentralized. There are many potential societal consequences of our work, none of which we feel must be specifically highlighted here.

\section{Extended Related Work}

\subsection{Federated Learning and LLMs}
\label{apndx:rw_fedllm}

Federated learning (FL) studies how to train a shared model from decentralized data while keeping raw client data on-device \citep{fedavg}.
In practice, FL must cope with \emph{statistical heterogeneity} (client data are non-iid)
and tight communication and compute budgets \citep{wang2020tackling}.
These challenges have motivated a rich set of optimization methods, including proximal regularization \citep{li2020federated} and variance-reduction techniques \citep{jhunjhunwala2022fedvarp}.
At scale, practical deployments also rely on systems mechanisms such as partial participation, asynchronous aggregation, secure aggregation, and failure tolerance \citep{xie2019asynchronous,bonawitz2019towards,bonawitz2016practical}. 
Overall, most FL work focuses on collaboratively training \emph{a single} global model (or personalized variants) under decentralized data and resource constraints.

LLMs have advanced rapidly in scale and capability, making them a default backbone for many language-centric applications.
However, bringing LLMs into FL is challenging because na\"{\i}vely communicating or fine-tuning billions of parameters is prohibitive for many clients.
As a result, a major thread in ``federated LLM'' research centers on parameter-efficient adaptation (e.g., adapter- or LoRA-style methods) and quantization-aware fine-tuning to reduce memory and communication, alongside benchmarks and frameworks for federated NLP/LLM evaluation \citep{flora,ravan,lin2022fednlp}.
Recent methods study federated instruction tuning under heterogeneous client resources and tasks \citep{fedit,cheng2024towards,chen2024integration}, as well as hybrid paradigms such as federated split learning that offload parts of the model when clients cannot host the full network \citep{zhao2024fedsllm}. 
Notably, these efforts largely assume a single deployed global model, whereas we target the complementary problem of learning \emph{federated routing policies} that select among a pool of LLMs without centralizing client prompts.
This distinction is crucial because high-quality routing typically requires access to private and broad query--model evaluations, which are naturally fragmented across clients and often cannot be uploaded to a central server. 
By leveraging FL, clients can collaboratively learn a router that generalizes across diverse prompt distributions and improves model-coverage across the pool, while keeping prompt data local.
\textit{To our knowledge, this work is the first to study LLM routing in a federated setting}.

\subsection{LLM Query Routing}

Language query routing selects a suitable model for generating response to a query based on relative importance on performance vs cost. This requires capturing query tasks through a convenient representation, usually obtained from sentence encoders \cite{cer2018universalsentenceencoder, paraphrase_albert_reimers-2019-sentence-bert}. Parametric routers estimate suitability scores or correctness probabilities of language models for given queries using trainable neural networks or scoring functions \cite{hu2024routerbench,ding2024hybridllmcostefficientqualityaware,ding2025bestrouteadaptivellmrouting,zhuang2025embedllm,sakota_2024_flyswat,huang2025routerevalcomprehensivebenchmarkrouting,lu2023routingexpertefficientrewardguided_zooter,treacle_thrifty_reasoning}. Beyond direct performance and cost prediction, routing is also formulated for estimating relative model quality or human judgments \cite{ong2025routellmlearningroutellms, frick2025prompttoleaderboard, chiang2024chatbotarenaopenplatform}, 
and via auxiliary signals such as task tags or learned difficulty estimators \cite{chen2025tagrouterlearningroutellms, treacle_thrifty_reasoning}.

An alternate direction of router design avoids learning a dedicated routing network and instead makes predictions using the geometry of the query embedding space and empirical performance statistics therein. 
Common versions of such nonparametric routers include partitioning the query space (e.g., clustering) and assigning models based on cluster-level summaries, or selecting models based on local neighborhoods under similarity search \cite{hu2024routerbench, jitkrittum2025universal, stripelis2024tensoroperaroutermultimodelrouter, li2025rethinkingpredictivemodelingllm, zhang2025avengerssimplerecipeuniting, avengers_pro_zhang2025gpt5makingllmscheaper, patel2025proxrouterproximityweightedllmquery}. These designs are often attractive in settings where the model pool evolves, since additional queries or new models can be incorporated by updating clusters or query embedding pool rather than retraining from scratch.

Routing is also closely related to broader multi-model inference strategies that allocate computation adaptively rather than committing to a single model for generating a response. Cascading and escalation frameworks route queries through sequences of models, using intermediate performance signals to decide when to stop or when to route to a stronger model \cite{dohan2022languagemodelcascades, chen2023frugalgptuselargelanguage, dekoninck2025unifiedapproachroutingcascading}. Other strategies exploit multiple parallel candidate generations, such as best-of-$n$ sampling \cite{ding2025bestrouteadaptivellmrouting}, or reduce generation cost via speculative decoding, where a small model proposes tokens and a larger model verifies them \cite{leviathan2023fastinferencetransformersspeculativedecoding}. These techniques are largely orthogonal to routing policies and can be composed with them.

Finally, several works examine routing as an online decision problem under explicit resource constraints. Bandit and budgeted-selection formulations make model choices based on sequential evaluations of model responses, seeking improved utility while respecting cost constraints \cite{li2025llmbanditcostefficientllm, treacle_thrifty_reasoning, poon2025onlinemultillmselectioncontextual_meta_llm}, and related solutions also include POMDP-based mixtures of models \cite{aggarwal2025automixautomaticallymixinglanguage}. Separately, interpretability and structure in routing signals has been demonstrated through model-comparison scores such as Elo \cite{zhao2024eagleefficienttrainingfreerouter}, psychometric inspired formulations such as Item Response Theory \cite{song2025irtroutereffectiveinterpretablemultillm}, and weakly supervised or metric learning based approaches \cite{guha2024smoothie, feng2025graphroutergraphbasedrouterllm, chen2024routerdcquerybasedrouterdual}.

\section{Details on Datasets and Federated Setup}
\label{apndx:dataset_details}

Details on the task, prompting techniques and evaluation scheme for \routerbench{} and \proxr{} are elaborated in \citet{hu2024routerbench} and \citet{patel2025proxrouterproximityweightedllmquery}, respectively. 
\camera{
The quality labels in \routerbench{} combine exact-match and LLM-based evaluations, whereas \proxr{} evaluates numerical and multiple-choice answers against benchmark targets. More generally,
$\ac(\xb,m)$ can be considered as a representation of expected scalar feedback such as user
approval or task success.
}

\subsection{Query Heterogeneity.} To simulate heterogeneity, we split the samples of \routerbench{} and \proxr{} across clients following Dirichlet distribution over subtasks \cite{dirichlet} with $\alpha$ parameters of 0.6 and 0.4, respectively.

We provide visualization of the induced client heterogeneity used in our experiments.
To visualize the query distribution, we compute a single 2D t-SNE embedding \cite{tsne} on the full dataset (using the same query representations for all plots), and then visualize each client's local subset by highlighting the corresponding points in the shared embedding.
\Cref{fig:tsne_clients_grid} reports the global distribution alongside client-level views, illustrating the non-IID partitions induced by the Dirichlet split ($\alpha=0.6$ in \Cref{sec:experiments}). Similarly, analogous t-SNE visualization for \proxr{} (Dirichlet with $\alpha=0.4$) is seen in \Cref{fig:proxr_tsne_clients_grid}.

\begin{figure}[htb]
  \centering
  \includegraphics[width=0.8\linewidth]{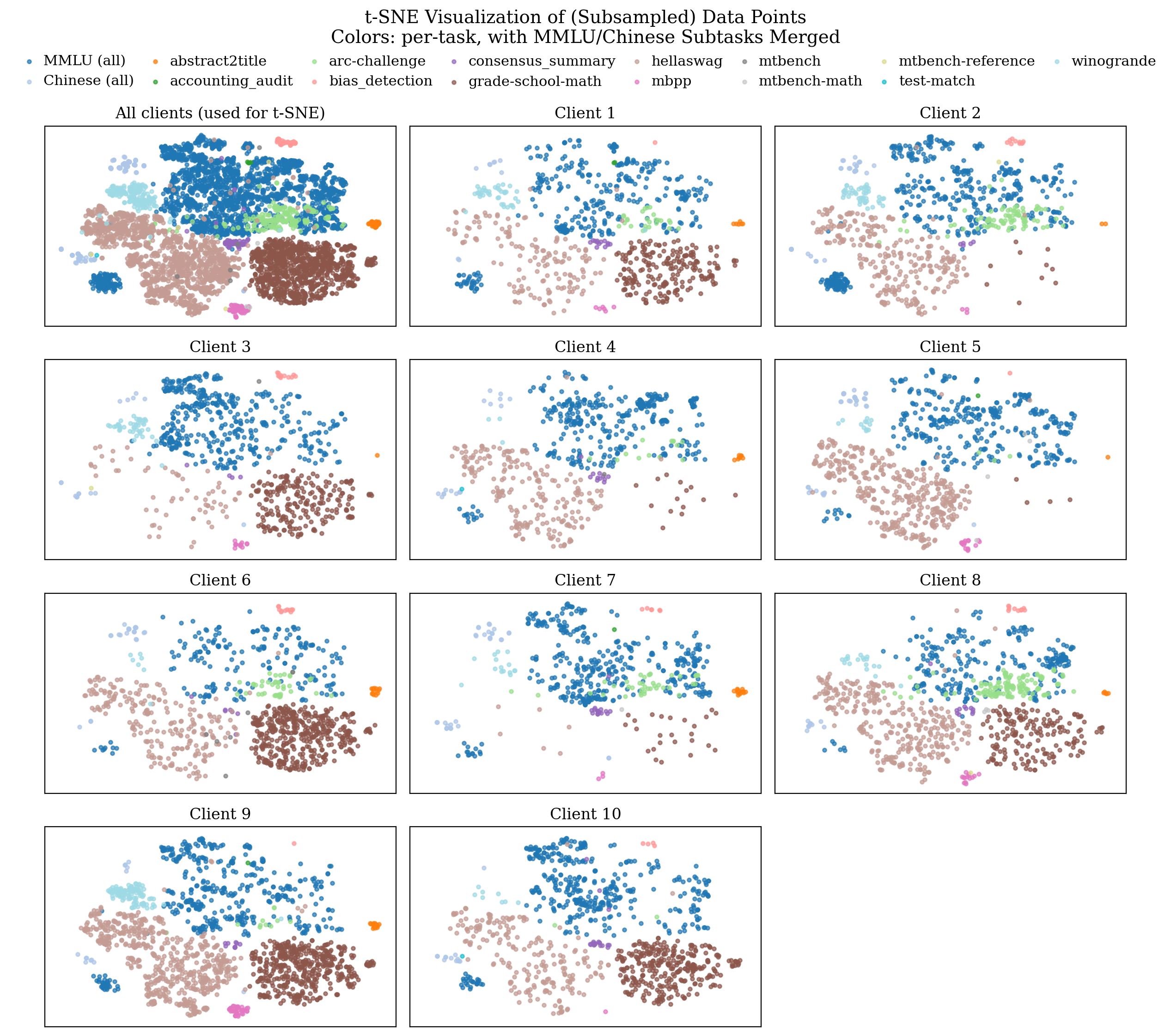}
  \caption{\textbf{t-SNE visualization of query distributions.}
  We run t-SNE once on the full dataset and reuse the same 2D embedding for all panels.
  The grid highlights each client's local subset, illustrating (subsampled) the resulting heterogeneity across clients.}
  \label{fig:tsne_clients_grid}
\end{figure}

\begin{figure}[htb]
  \centering
  \includegraphics[width=0.8\linewidth]{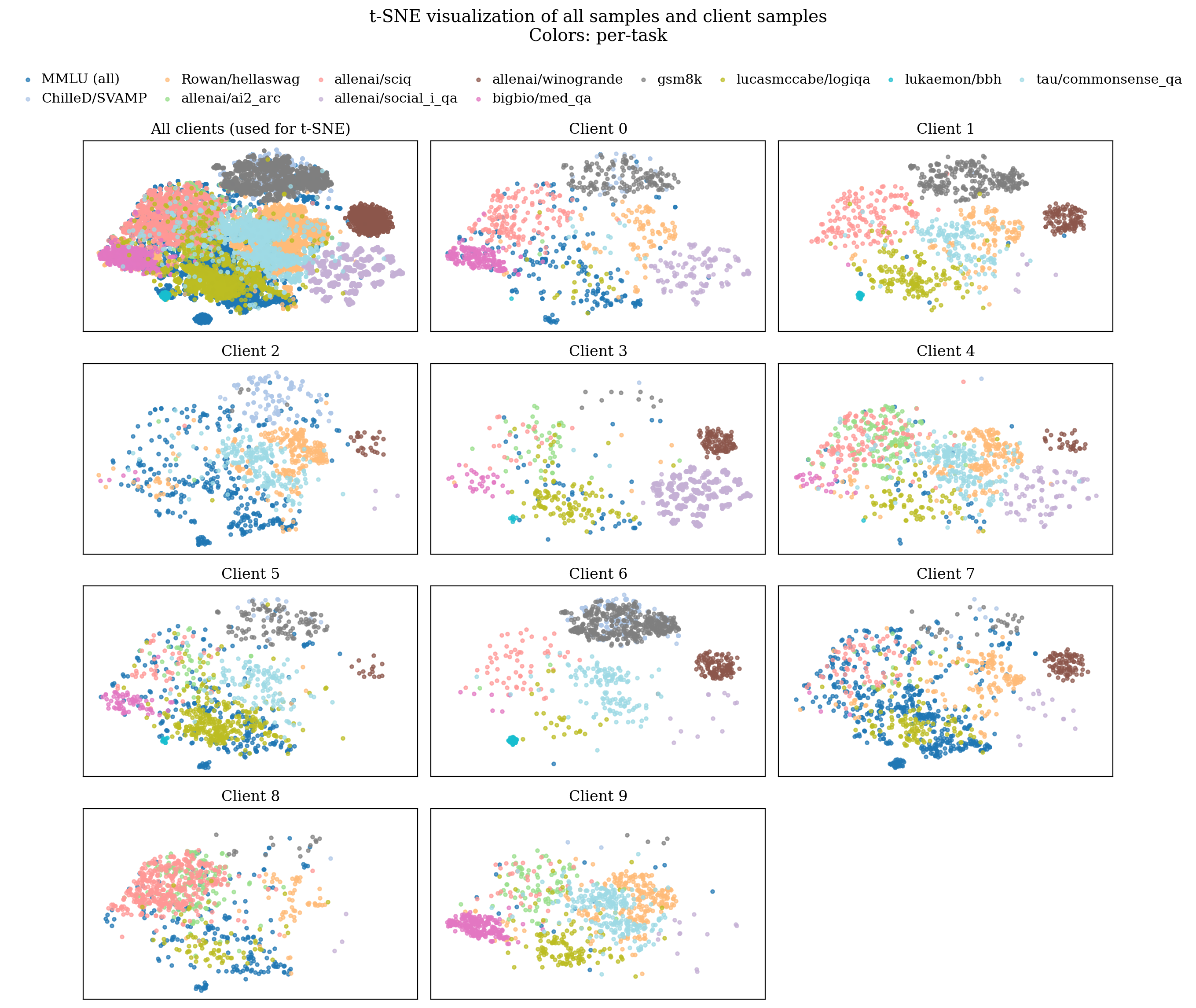}
  \caption{\textbf{t-SNE visualization of query distributions for $\proxr$.}}
  \label{fig:proxr_tsne_clients_grid}
\end{figure}

\subsection{Model Heterogeneity}
\label{sect:appendix_model_heterogeneity}

In our experiments, we assume only a single model evaluation per query is logged in the training sets. To simulate heterogeneous model assignments across clients, for each client $i$, we first sample a client-specific distribution over $M$ models with the Dirichlet distribution with $\alpha = 0.45$ for \routerbench{}.
Then, for each local sample $j \in [|\dcal_i|]$, we assign a model index by sampling with probabilities of the client-specific distribution.
This procedure induces client-dependent, potentially highly non-uniform model proportions within $\dcal_i$. For \routerbench{}, \Cref{fig:model_heterogeneity} shows the bubble plot of the proportions of models in the client datasets. For \proxr{}, we use uniform model logging for variety in experiments.

\begin{figure}[htb]
    \centering
    \includegraphics[width=0.5\linewidth]{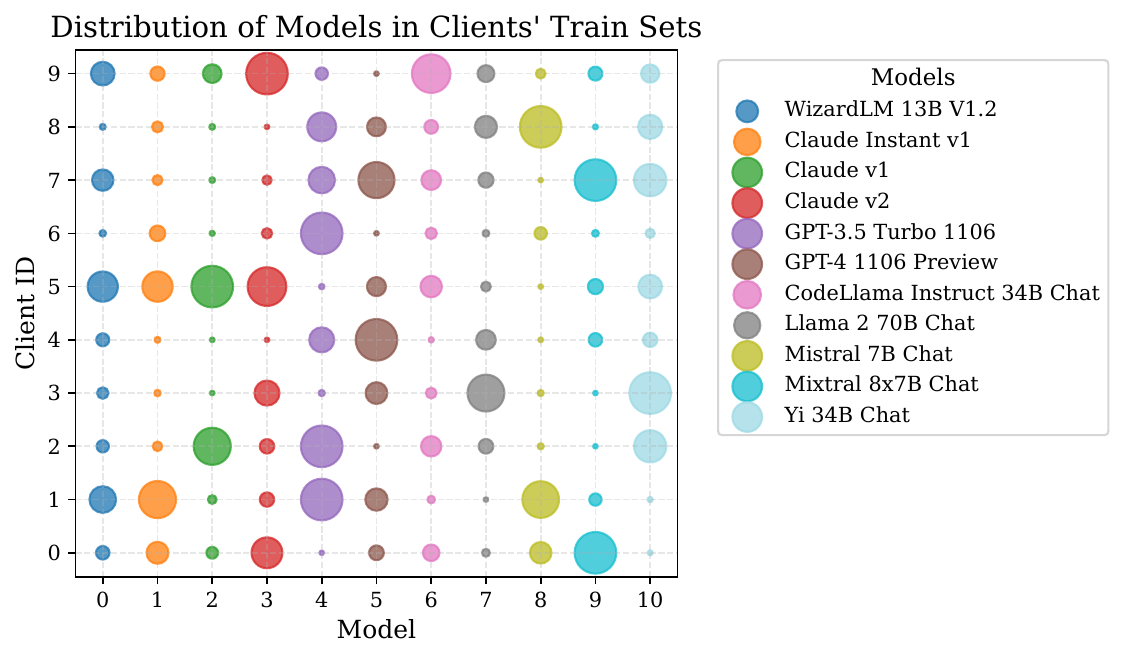}
    \caption{\textbf{Client-specific model heterogeneity for \routerbench{}.} Bubble plot of the per-model proportions in each client local dataset $\dcal_i$ under the Dirichlet assignment scheme ($\alpha=0.45$).}
    \label{fig:model_heterogeneity}
\end{figure}

\section{Details on Experimental Setting}
\label{app:exp_setting_details}

We provide additional details for main text experiments here.

\paragraph{Client simulation and train/test split.}
We simulate $N=10$ clients.
Within each client, we split its allocated data into local train/test with fractions $0.75/0.25$, and the \emph{global} train/test sets are defined as the unions of client train/test splits.

\paragraph{Evaluation protocol.}
For acc--cost tradeoff curves, we sweep the $\lmb$ parameter on a log grid $\lmb \in [10^{-2}, 10^{7}]$ with 100 points and report the resulting curves/AUC.

\subsection{Federated \mlp{}}
\paragraph{Model.}
The router is an MLP with a shared trunk of two hidden layers of widths $(512,512)$, each followed by LayerNorm, GELU, and dropout ($p=0.1$), and with per-model heads that predict (i) an accuracy logit (sigmoid at inference) and (ii) a normalized cost scalar.

\paragraph{Optimization and FL hyperparameters.}
We train with FedAvg. The local optimizer is AdamW with learning rate $\eta=10^{-3}$ and weight decay $3\cdot 10^{-4}$.
Each communication round performs 1 local epoch (over full local train data) per participating client (mini-batch size $128$; gradient clipping with max-norm $1.0$). \camera{The router contains approximately $670$K parameters, corresponding to a one-way model communication of $2.68$~MB per active client per round.}

\subsection{Federated K-Means Router}
\paragraph{Router and clustering hyperparameters.}
This router is training-free (no learning rate). Each client runs Lloyd’s K-means \citep{lloyd} in embedding space using Euclidean distance.
We use $K_{\text{local}}=15$ clusters per client and $K_{\text{global}}=20$ clusters at the server.
Both local and global K-means use $n_{\text{init}}=3$ random restarts and at most $30$ iterations.
We treat the overall procedure as a single federated clustering stage ($T=1$), consisting of (i) clients $\rightarrow$ server upload of local centroids (with sizes), (ii) server-side clustering into $K_{\text{global}}$ centers, and (iii) clients $\rightarrow$ server aggregation of per-cluster accuracy/cost statistics for routing. \camera{The communication cost in our experiment is approximately $0.5$~MB.}

\section{\routerbench{} Experiments not Presented in the Main Text due to the Space Limit}
\label{sect:app_routerb_exp}

In this section, we present the remaining main experimental results using \routerbench{}.

\camera{
\subsection{Sensitivity to the Number of Clients}
\label{apndx:client_count}

We repeat the main \routerbench{} experiment with
$N\in\{5,10,20,40\}$ clients, keeping the total dataset size and
participation fraction fixed.
\Cref{tab:client_count} reports global-test normalized AUCs across
random seeds.
Federated performance remains stable as $N$ increases, whereas
client-local performance generally decreases as the available
training data are divided among more clients.
FL thus preserves collective query and model coverage
despite reduced local sample sizes.

\begin{table}[h]
\centering
\small
\caption{\camera{Sensitivity to the number of clients on \routerbench{}.
Global-test normalized AUCs are reported as mean $\pm$ standard
error across seeds; client-local results are averaged over clients.}}
\label{tab:client_count}
\begin{tabular}{c cc cc}
\toprule
& \multicolumn{2}{c}{\mlp{}}
& \multicolumn{2}{c}{\kmeans{}} \\
\cmidrule(lr){2-3}\cmidrule(lr){4-5}
$N$ & Federated & Client-local & Federated & Client-local \\
\midrule
5  & $0.747 \pm 0.004$ & $0.679 \pm 0.010$
   & $0.733 \pm 0.007$ & $0.656 \pm 0.023$ \\
10 & $0.750 \pm 0.005$ & $0.679 \pm 0.008$
   & $0.742 \pm 0.007$ & $0.646 \pm 0.018$ \\
20 & $0.750 \pm 0.003$ & $0.661 \pm 0.007$
   & $0.745 \pm 0.003$ & $0.601 \pm 0.011$ \\
40 & $0.749 \pm 0.009$ & $0.652 \pm 0.003$
   & $0.749 \pm 0.005$ & $0.584 \pm 0.008$ \\
\bottomrule
\end{tabular}
\end{table}

\subsection{Sensitivity to Query Heterogeneity}
\label{apndx:query_heterogeneity}

We vary the Dirichlet concentration parameter
$\alpha\in\{0.03,0.10,0.30,0.60,5.00\}$ in the main
\routerbench{} setting, where smaller $\alpha$ induces greater
heterogeneity in client query distributions.
As shown in \Cref{tab:query_heterogeneity}, both federated routers
remain robust across the sweep and consistently outperform
client-local training.
The gains are particularly significant at $\alpha=0.03$, where
individual clients have highly restricted query coverage.

\begin{table}[h]
\centering
\small
\caption{\camera{Sensitivity to query heterogeneity on \routerbench{}.
Global-test normalized AUCs are reported as mean $\pm$ standard
error across seeds; client-local results are averaged over clients.}}
\label{tab:query_heterogeneity}
\begin{tabular}{c cc cc}
\toprule
& \multicolumn{2}{c}{\mlp{}}
& \multicolumn{2}{c}{\kmeans{}} \\
\cmidrule(lr){2-3}\cmidrule(lr){4-5}
$\alpha$ & Federated & Client-local & Federated & Client-local \\
\midrule
0.03 & $0.742 \pm 0.007$ & $0.640 \pm 0.015$
     & $0.731 \pm 0.007$ & $0.601 \pm 0.033$ \\
0.10 & $0.743 \pm 0.012$ & $0.666 \pm 0.010$
     & $0.742 \pm 0.005$ & $0.614 \pm 0.014$ \\
0.30 & $0.744 \pm 0.012$ & $0.669 \pm 0.010$
     & $0.734 \pm 0.015$ & $0.619 \pm 0.014$ \\
0.60 & $0.746 \pm 0.015$ & $0.672 \pm 0.011$
     & $0.741 \pm 0.008$ & $0.650 \pm 0.018$ \\
5.00 & $0.748 \pm 0.012$ & $0.669 \pm 0.013$
     & $0.749 \pm 0.004$ & $0.625 \pm 0.015$ \\
\bottomrule
\end{tabular}
\end{table}
}

\subsection{The Local Test Results for Federated Models vs.\ Locally Trained Models on All Clients}
\label{apndx:all_clients_local_tests}

We report the complete set of local test results across all clients.
For each client, we compare the federated router trained collaboratively across clients against a client-local (no-FL) router trained using that client's data, where evaluation is performed on the same client's local test distribution.
\Cref{fig:all_clients_local_test_mlp,fig:all_clients_local_test_kmeans} show the full grids for \mlp{} and \kmeans{}, respectively, confirming the consistent improvements discussed in the main text.

\begin{figure}[h]
  \centering
  \includegraphics[width=0.74\linewidth]{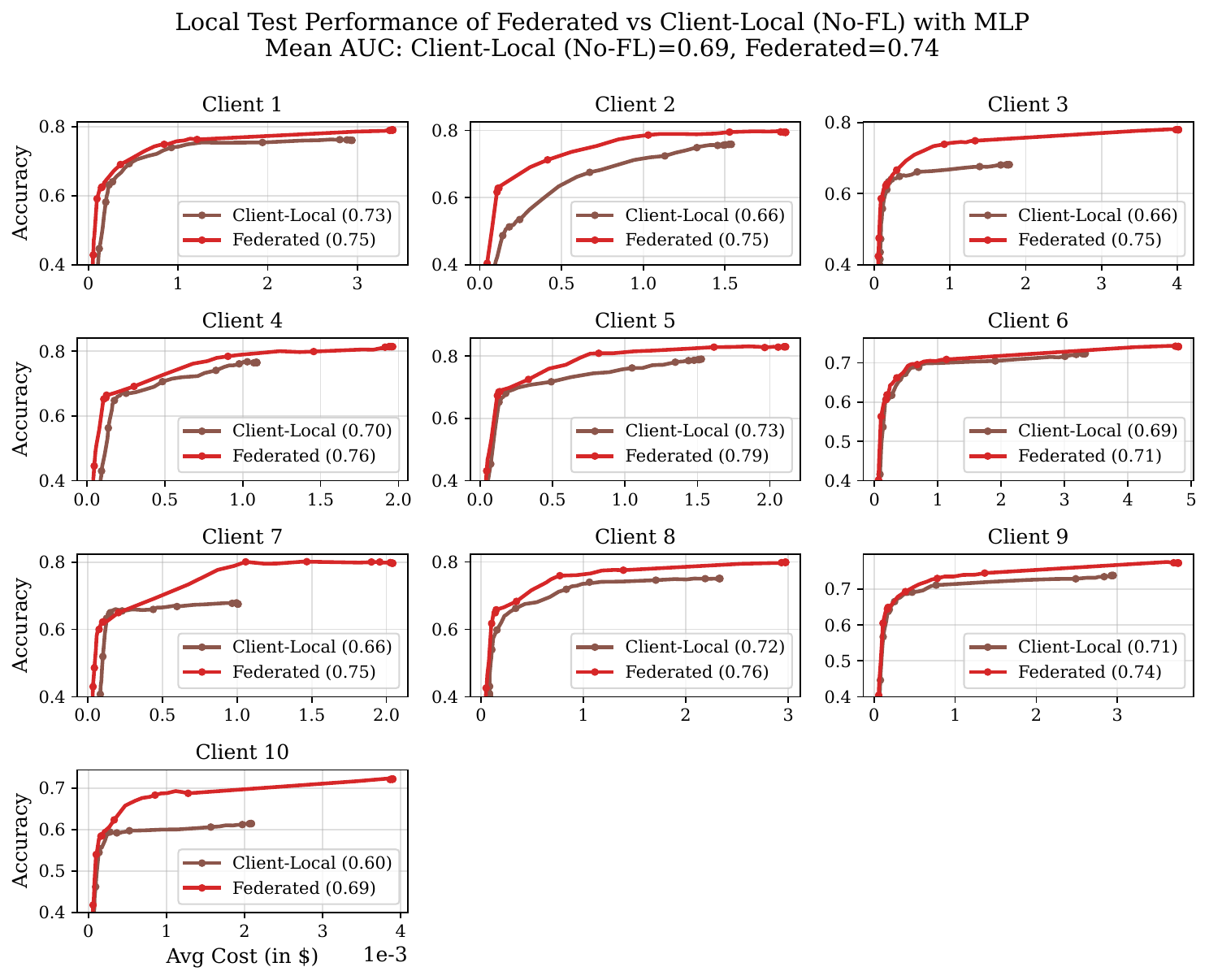}
  \caption{\textbf{All-clients local test results (\mlp{}).}
  Each panel corresponds to one client and reports accuracy--cost trade-offs for the federated router versus the client-local (no-FL) router, evaluated on that client's local test set.}
  \label{fig:all_clients_local_test_mlp}
\end{figure}

\begin{figure}[h]
  \centering
  \includegraphics[width=0.74\linewidth]{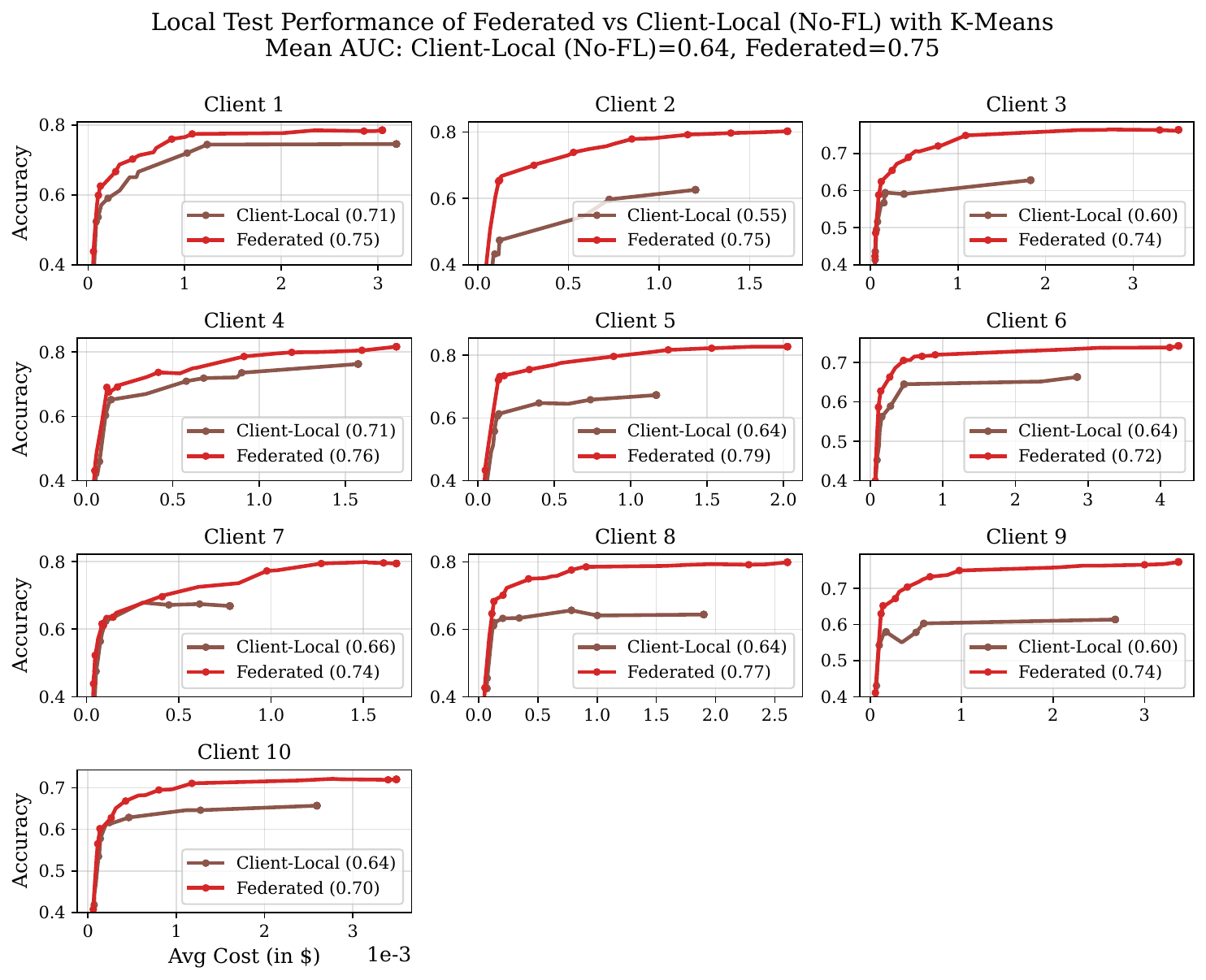}
  \caption{\textbf{All-clients local test results (\kmeans{}).}
  Each panel corresponds to one client and reports accuracy--cost trade-offs for the federated router versus the client-local (no-FL) router, evaluated on that client's local test set.}
  \label{fig:all_clients_local_test_kmeans}
\end{figure}

\subsection{When New Clients Join the System}
\label{apndx:new_clients_join}
In practice, new clients may join after the router has already been trained.
An effective system should (i) improve using the new clients' data, (ii) preserve performance on previously-seen clients (avoid \emph{catastrophic forgetting}), and (iii) require no additional training/communication from the existing clients.
We study this setting by initially training the router with 7 clients and then introducing 3 new clients whose data account for 30\% of the unique task labels in the system.
For \mlp{}, we continue training of existing router using only newly joined clients, while adding a distillation-style regularizer that penalizes deviation from the original router's outputs to preserve the initial routing policy. 
Let $\mlpw^{(0)}$ denote the router parameters before the new clients join (kept frozen during adaptation). We regularize by distilling the base router's model-wise predictions on the new clients' prompts with regularization loss, 
\(
\E_{\xb \sim \dcal_{\mathrm{new}}}\!\left[\frac{1}{|\mcal|}\sum_{m\in\mcal}
\big(\aes_{\mlpw}(\xb,m)-\aes_{\mlpw^{(0)}}(\xb,m)\big)^2
+\big(\ces_{\mlpw}(\xb,m)-\ces_{\mlpw^{(0)}}(\xb,m)\big)^2\right].
\)
 \kmeans{} adaptation is training-free by weighted updating the accuracy and cost statistics of new clients in the server.
\Cref{fig:new_client_in_the_system} reports the global accuracy--cost trade-off before and after new clients join.
Both methods improve after incorporating the new clients, and \kmeans{} achieves this with minimal additional machinery.

\begin{figure}[H]
  \centering
  \includegraphics[width=0.65\linewidth]{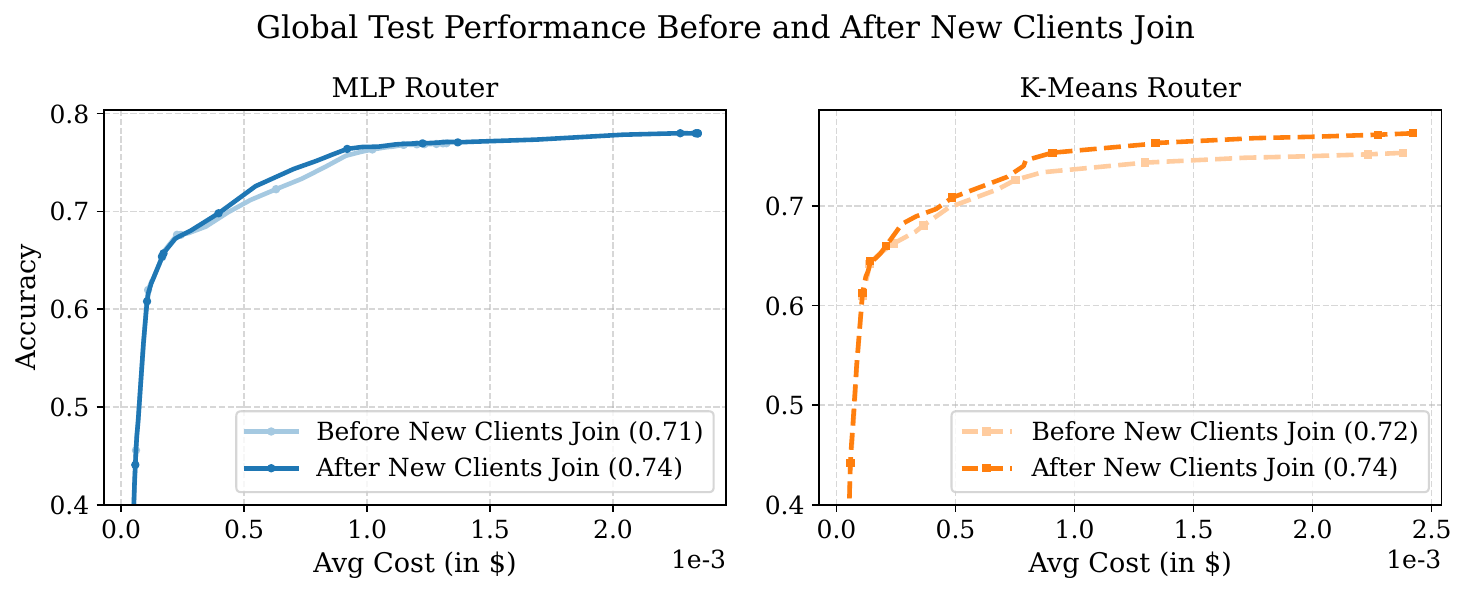}
  \caption{\textbf{Adapting to new clients.}
  Routers are trained initially by 7 clients, then adapted by 3 new clients, without additional participation from the original clients.
  \mlp{} uses continued training with a distillation regularizer to retain the base policy; \kmeans{} updates cluster-level statistics via weighted averaging.
  We report accuracy--cost frontiers on the global test set.}
\label{fig:new_client_in_the_system}
\end{figure}

\subsection{Adaptive Personalization Under High Heterogeneity}
\label{apndx:adaptive_personalization}

\Cref{fig:adaptive_personalization_mlp_grid,fig:adaptive_personalization_kmeans_grid} provide the same accuracy--cost frontier comparisons as in \Cref{fig:adaptive_personalization_main}, but for all clients.
Overall, the per-client plots confirm the main-text trend: under extreme heterogeneity, federated \mlp{} can be suboptimal for a subset of clients, while adaptive personalization reliably recovers and sometimes improves upon the best achievable tradeoff between federated and local routing.
For \kmeans{}, federated routing remains competitive across clients, whereas isolated local routing can degrade sharply due to insufficient per-centroid model-coverage.

\camera{
\paragraph{Global-test performance.}
In the same high-heterogeneity setting, evaluating the personalized
routers on the global test set yields normalized AUCs of $0.735$
for \mlp{} and $0.691$ for \kmeans{}, compared with average
client-local AUCs of $0.673$ and $0.636$, respectively.
Personalization reduces global performance relative to the
unpersonalized federated router as the estimates adapt toward
individual clients' distributions.
Nevertheless, it retains better global generalization than
client-local training while improving robustness on local test
distributions.
}

\begin{figure}[h]
  \centering
  \includegraphics[width=0.8\linewidth]{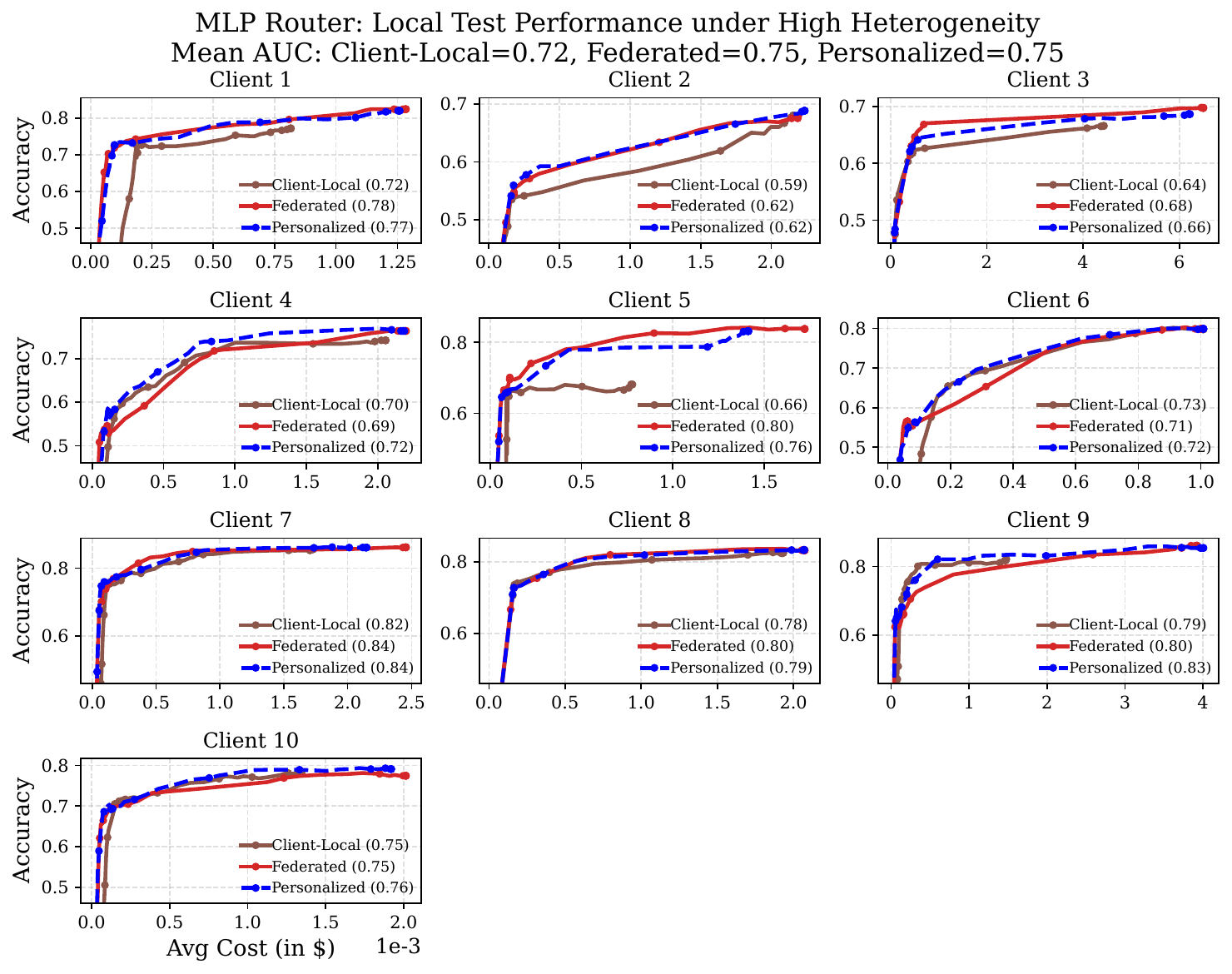}
  \vspace{-1em}
  \caption{\textbf{Per-client accuracy--cost frontiers under extreme heterogeneity ($\alpha=0.03$) for \mlp{}.}}
  \label{fig:adaptive_personalization_mlp_grid}
\end{figure}

\begin{figure}[htb]
  \centering
  \includegraphics[width=0.8\linewidth]{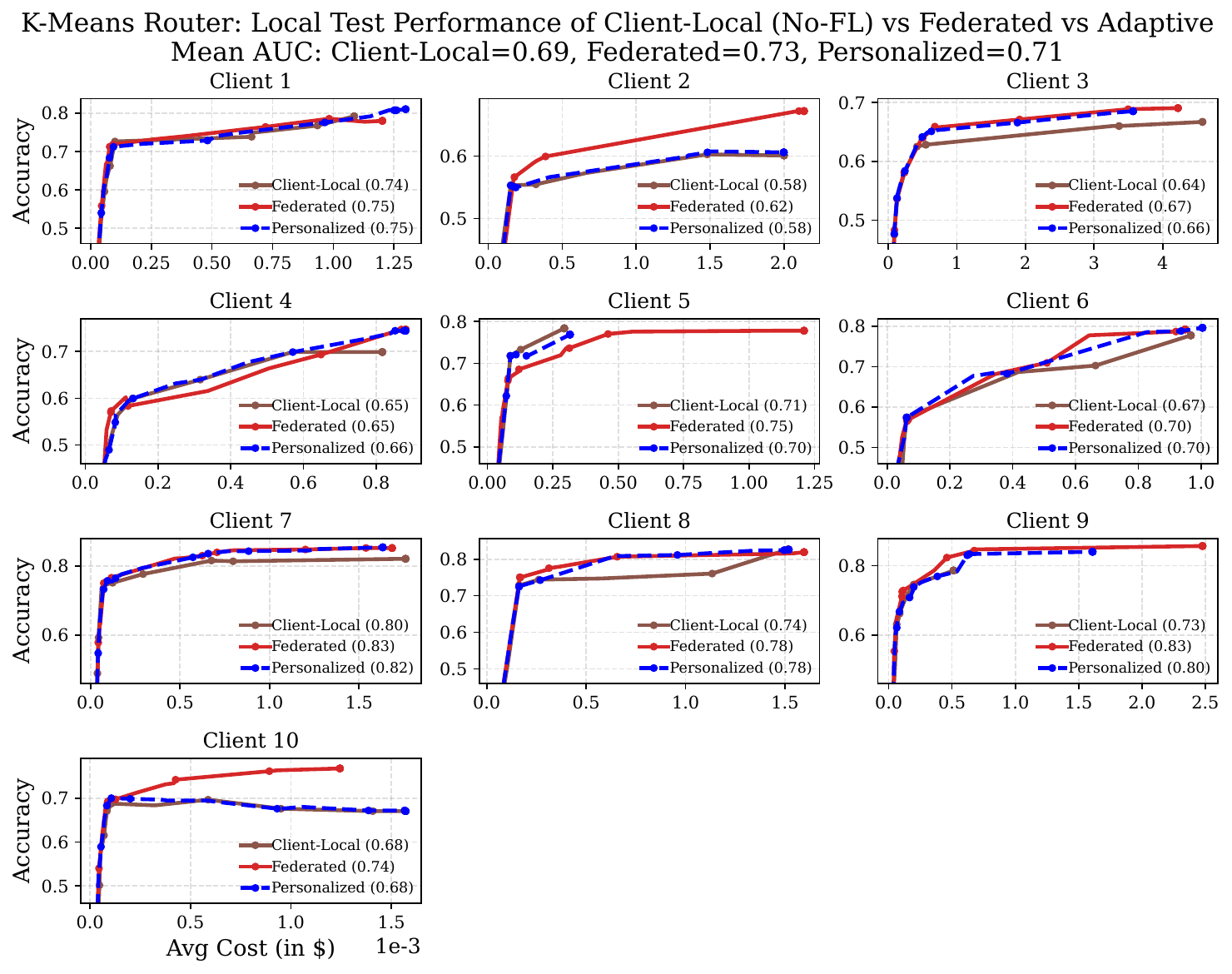}
  \vspace{-1em}
  \caption{\textbf{Per-client accuracy--cost frontiers under extreme heterogeneity ($\alpha=0.03$) for \kmeans{}.}}
  \label{fig:adaptive_personalization_kmeans_grid}
\end{figure}

\FloatBarrier

\section{Experimental Results with Different Sentence Encoder Models}
\label{apndx:different_enc_results}
We present ablations on centralized router performance on both \routerbench{} and \proxr{} by varying the choice of sentence encoders used for representing queries. In the main text, we utilize \text{all\_mpnet\_base\_v2} \cite{song2020mpnetmaskedpermutedpretraining} sentence encoder ($768$ dimensional query representation). We additionally esperiment with \text{all\_minilm\_l6\_v2} \cite{wang2020minilmdeepselfattentiondistillation} ($384$ dimensional query representation) and \text{paraphrase\_albert\_small\_v2} \cite{paraphrase_albert_reimers-2019-sentence-bert} ($768$ dimensional query representation). We observe that routing performance is relatively constant across sentence encoders, highlighting that presented approach is generalizable to a variety of sentence encoders.

\begin{table}[h]
\centering
\caption{Performance (through normalized AUC) of \kmeans{} and \mlp{} in centralized training regime across different choices of sentence encoders used for representing queries, for both \routerbench{} and \proxr{}.}
\label{tab:encoder_performance}
\begin{tabular}{l cc cc}
\toprule
\multirow{2}{*}{\shortstack[l]{\textbf{Sentence Encoder} \\ \textbf{(Emb. Dim.)}}}
& \multicolumn{2}{c}{\textbf{\routerbench{}}}
& \multicolumn{2}{c}{\textbf{\proxr{}}} \\
\cmidrule(lr){2-3} \cmidrule(lr){4-5}
& \textbf{\kmeans{}} & \textbf{\mlp{}}
& \textbf{\kmeans{}} & \textbf{\mlp{}} \\
\midrule
\text{all\_mpnet\_base\_v2} (768) & 0.747 & 0.759 & 0.759 & 0.760 \\
\text{all\_minilm\_l6\_v2} (384) & 0.737 & 0.761 & 0.728 & 0.751 \\
\text{paraphrase\_albert\_} & 0.746 & 0.758 & 0.751 & 0.755 \\
\text{small\_v2} (768) & & & & \\
\bottomrule
\end{tabular}
\end{table}

\section{Experimental Results with \proxr{}}
\label{apndx:results_with_proxr}

\camera{We repeat the global- and local-test comparisons, model and client
expansion experiments, and high-heterogeneity personalization
analysis using \proxr{}}. In this section, we present the results and conclusions for \proxr{} experiments.

\paragraph{Federated vs.\ client-local routers on the global test distribution.}
In \Cref{fig:proxr_global_vs_local_only_global_test}, we repeat the comparison in \Cref{fig:global_vs_local_only_global_test} using $\proxr$.
As in $\routerbench$, federated training improves global test distribution generalization for both \mlp{} and \kmeans{} in \proxr{}.

\begin{figure}[h]
  \centering
  \includegraphics[width=0.46\linewidth]{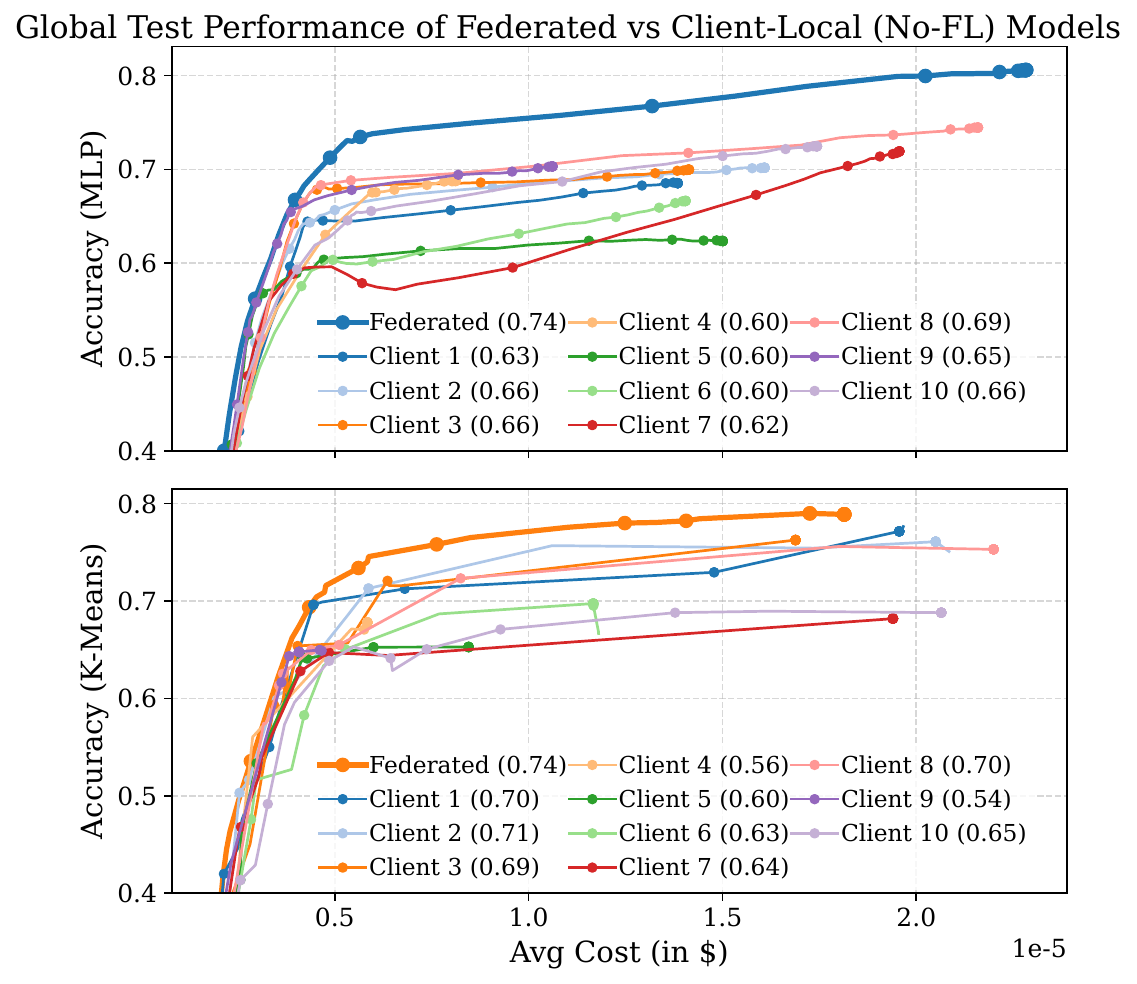}
  \caption{\textbf{($\proxr$) Federated vs.\ client-local (no-FL) routers on the global test distribution.}
  On \proxr{}, we evaluate federated model and locally trained models on the global test set as in \Cref{fig:global_vs_local_only_global_test}.}
  \label{fig:proxr_global_vs_local_only_global_test}
\end{figure}

\paragraph{Federated learning improves in-distribution local performance via better model coverage.}
We show the local-test set evaluation results of the federated model and locally trained models in \Cref{fig:proxr_local_vs_global_local_test_selected_clients} on $\proxr$, and again observe consistent improvements from federated collaboration. Also, in \Cref{fig:proxr_all_clients_local_test_mlp} and \Cref{fig:proxr_all_clients_local_test_kmeans}, we provide the results for all clients. 

\begin{figure}[h]
  \centering
  \includegraphics[width=0.60\linewidth]{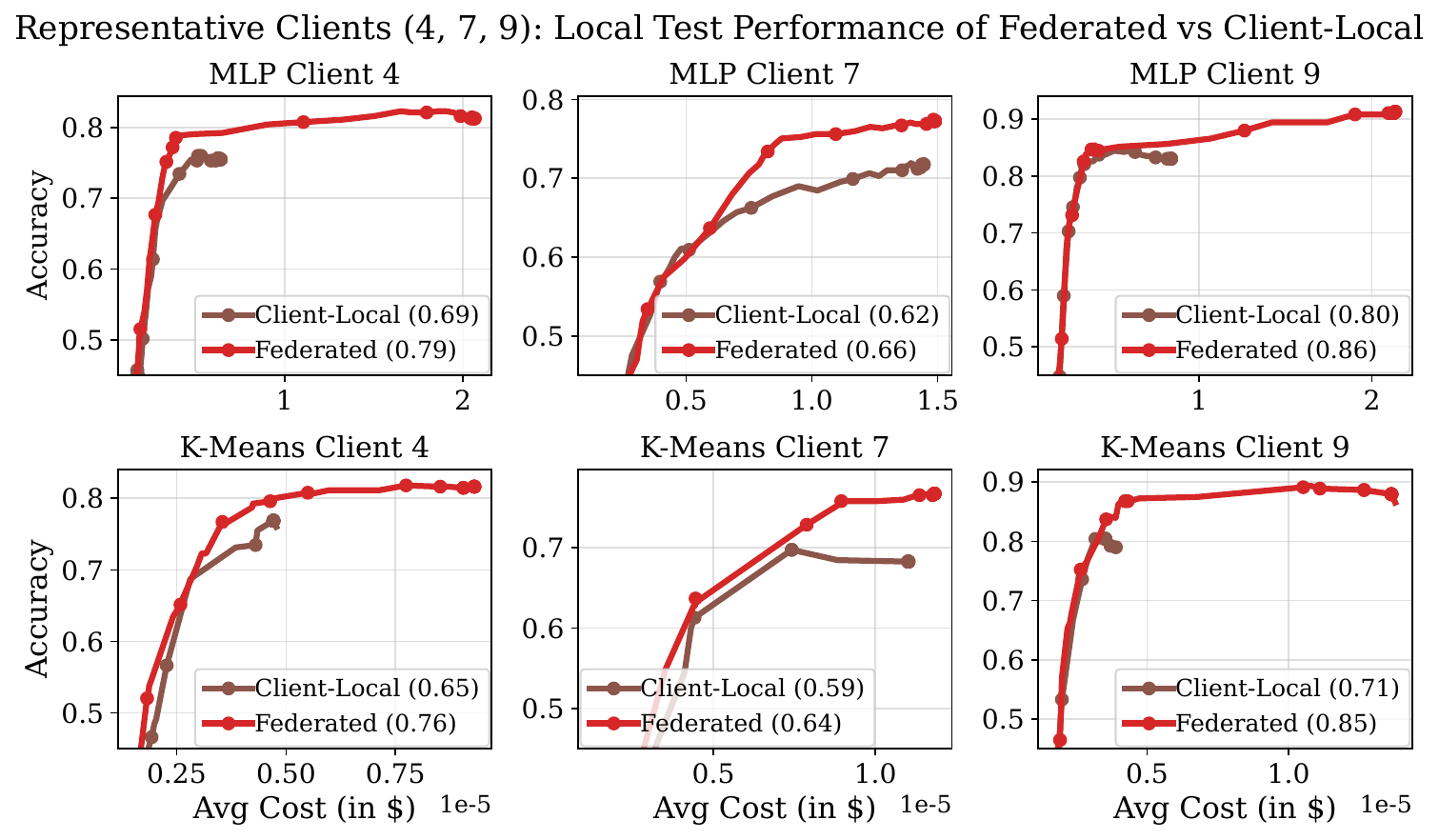}
  \vspace{-0.5em}
  \caption{\textbf{($\proxr$ - Subsampled Clients) Local test set results of subsampled clients.}
  We evaluate the federated and local models on \proxr{}, as in \Cref{fig:local_vs_global_local_test_selected_clients}.}
  \label{fig:proxr_local_vs_global_local_test_selected_clients}
\end{figure}

\begin{figure}
  \centering
  \includegraphics[width=0.65\linewidth]{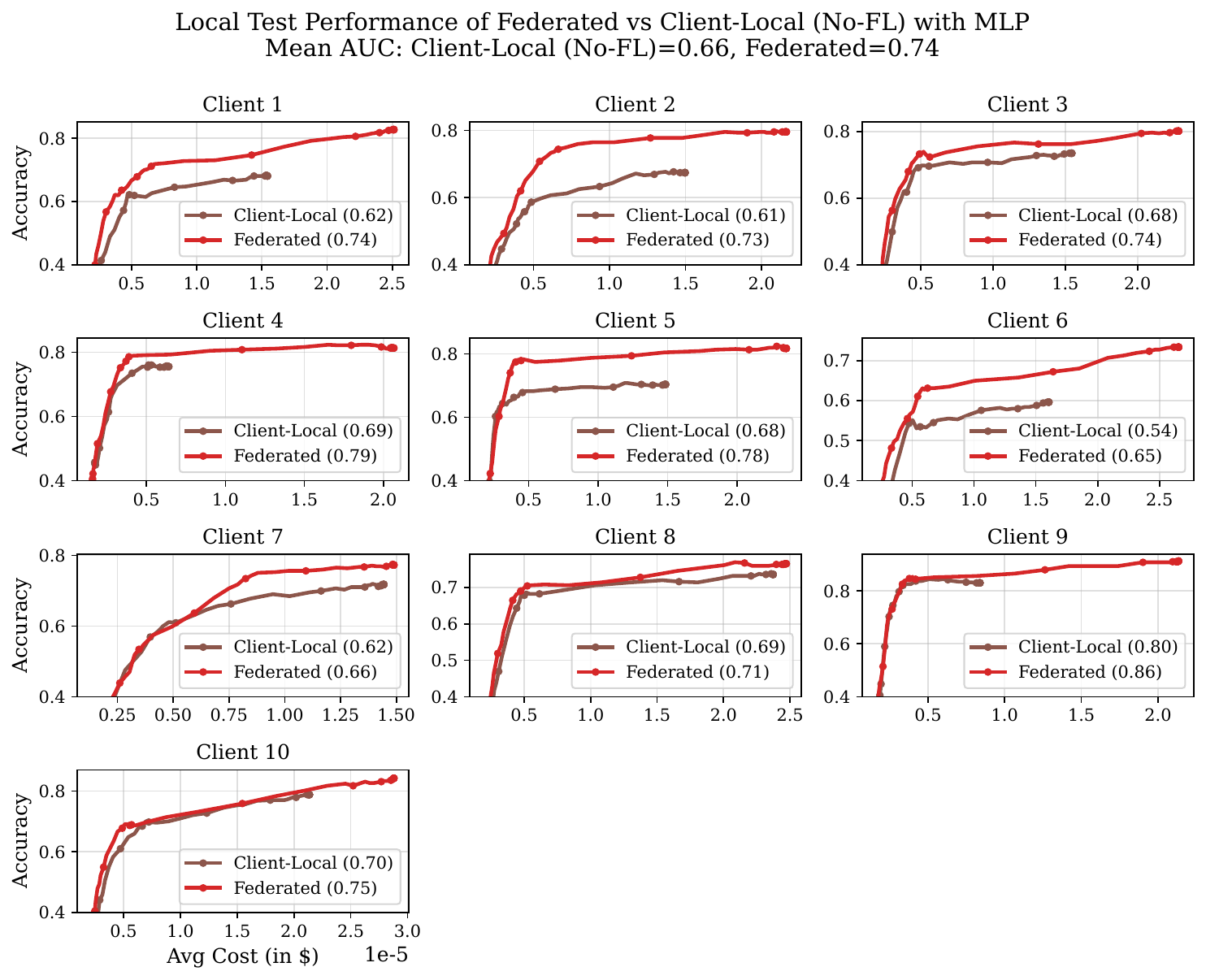}
  \vspace{-0.1em}
  \caption{\textbf{($\proxr$) Local test set results of all clients with \mlp{}.}
  We evaluate the federated and local models on \proxr{}, as \Cref{fig:all_clients_local_test_mlp}.}
  \label{fig:proxr_all_clients_local_test_mlp}
\end{figure}

\begin{figure}[H]
  \centering
  \includegraphics[width=0.65\linewidth]{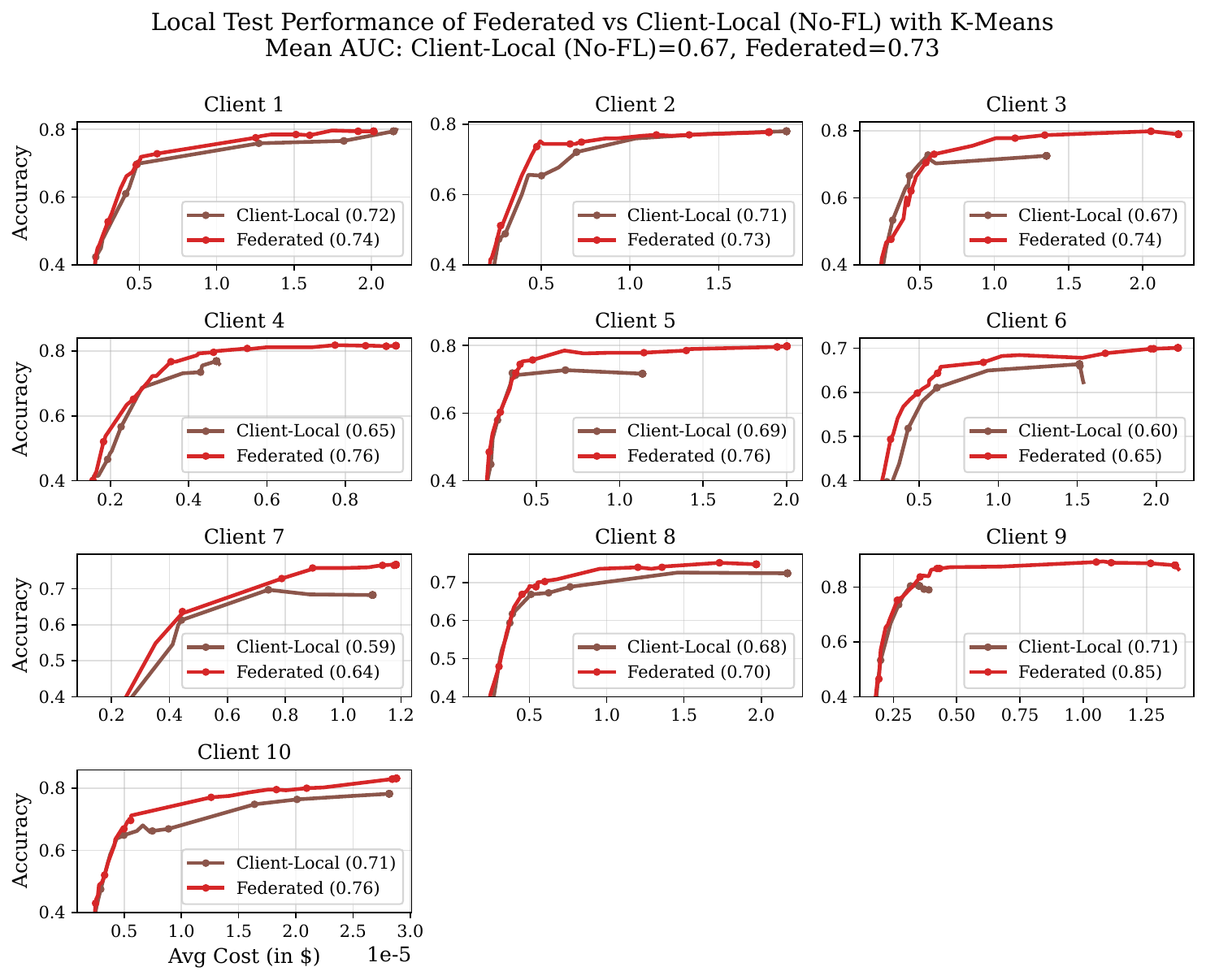}
  \vspace{-0.1em}
  \caption{\textbf{($\proxr$) Local test set results of all clients with \kmeans{}.}
  We evaluate the federated and local models on \proxr{}, as \Cref{fig:all_clients_local_test_kmeans}.}
  \label{fig:proxr_all_clients_local_test_kmeans}
\end{figure}

\paragraph{Federated routers vs.\ centralized training.}
On $\proxr$, we repeat the comparison of federated routers and centralized baseline and show the results in \Cref{fig:proxr_global_vs_centralized}.

\begin{figure}[H]
  \centering
  \includegraphics[width=0.4\linewidth]{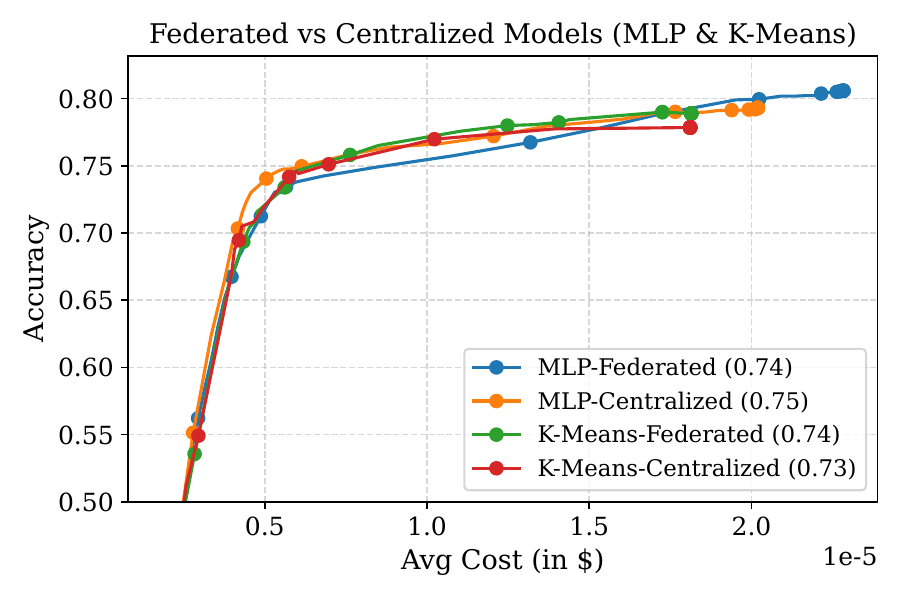}
  \vspace{-0.5em}
  \caption{\textbf{($\proxr$) Federated vs.\ centralized training.}
}
  \label{fig:proxr_global_vs_centralized}
\end{figure}

\paragraph{Model expansion (onboarding new models).}
We repeat the model-expansion experiment from \Cref{sect:new_models_join} with \proxr{} with two models are withheld during initial training and introduced later. We plot the results in \Cref{fig:proxr_model_expansion}. 

\begin{figure}[h]
  \centering
  \includegraphics[width=0.45\linewidth]{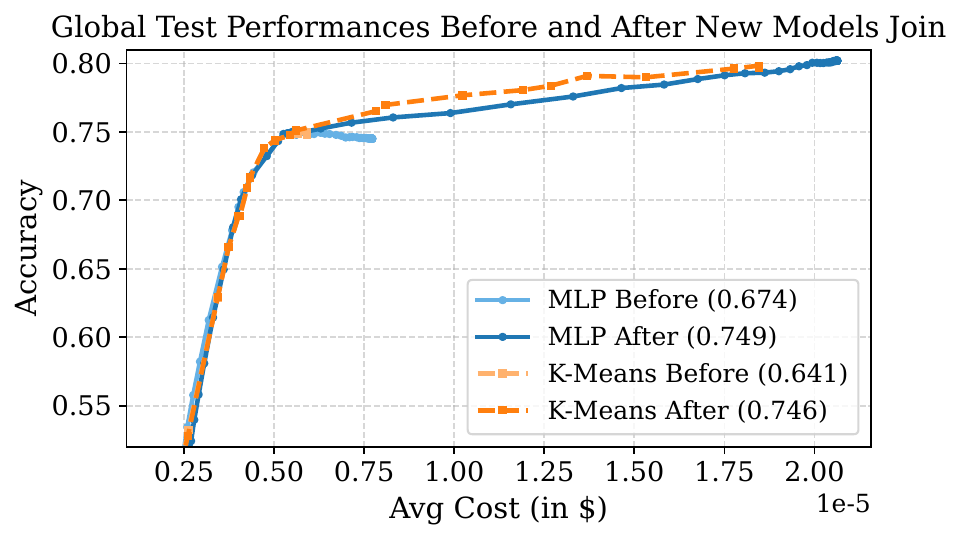}
  \caption{\textbf{($\proxr$) Adapting to newly joined models.}
  We observe that our framework allows successful integration of the new models introduced in the system.}
  \label{fig:proxr_model_expansion}
\end{figure}

\paragraph{Adapting to newly joined clients.}
We repeat the client-expansion experiment from \Cref{apndx:new_clients_join} with \proxr{}. We study this setting by initially training the router with 7 clients and then introducing 3 new clients whose data accounts for 65\% of the unique task labels in the system. We plot the results in \Cref{fig:proxr_new_client_in_the_system}.

\begin{figure}[h]
  \centering
  \includegraphics[width=0.7\linewidth]{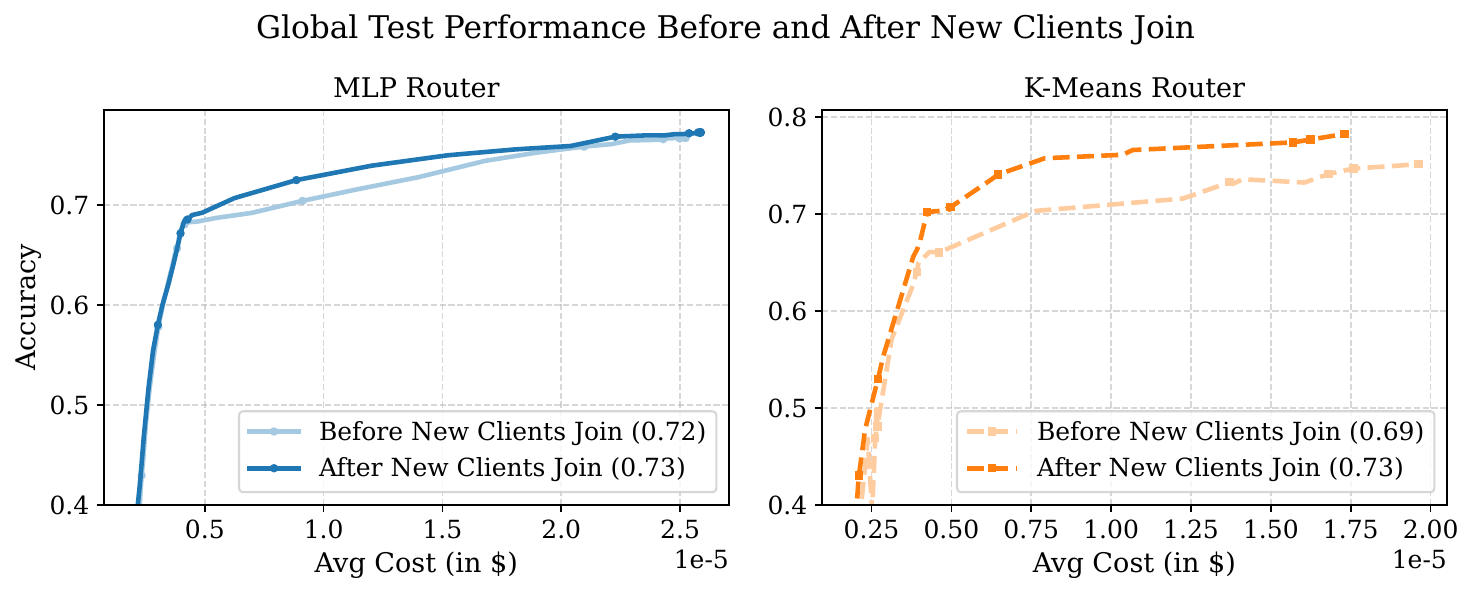}
  \caption{\textbf{($\proxr$) Adapting to newly joined clients.}
  We observe that our framework allows successful integration of the new clients joining  in the system.}
  \label{fig:proxr_new_client_in_the_system}
\end{figure}

\paragraph{Adaptive personalization under high heterogeneity.}
Using the \proxr{}, we repeat the high-heterogeneity experiment from the main text
(\Cref{sec:adaptive_personalization}), where client query distributions are highly non-iid (Dirichlet distribution with $\alpha=0.04$).
In \Cref{fig:adaptive_personalization_proxr_selected_clients}, we report local-test performance for a set of
representative clients, comparing (i) client-local routers, (ii) federated routers, and (iii) our adaptive
personalization. We further provide the full per-client grids in
\Cref{fig:adaptive_personalization_proxr_all_clients_mlp,fig:adaptive_personalization_proxr_all_clients_kmeans}
for \mlp{} and \kmeans{}, respectively.
On \proxr{}, the federated router's local-test performance varies more sharply across clients than the variation in \routerbench{}. We observe that adaptive personalization generally matches or improves upon the better of
the federated and client-local routers on each client.

\begin{figure}[h]
  \centering
  \includegraphics[width=0.64\linewidth]{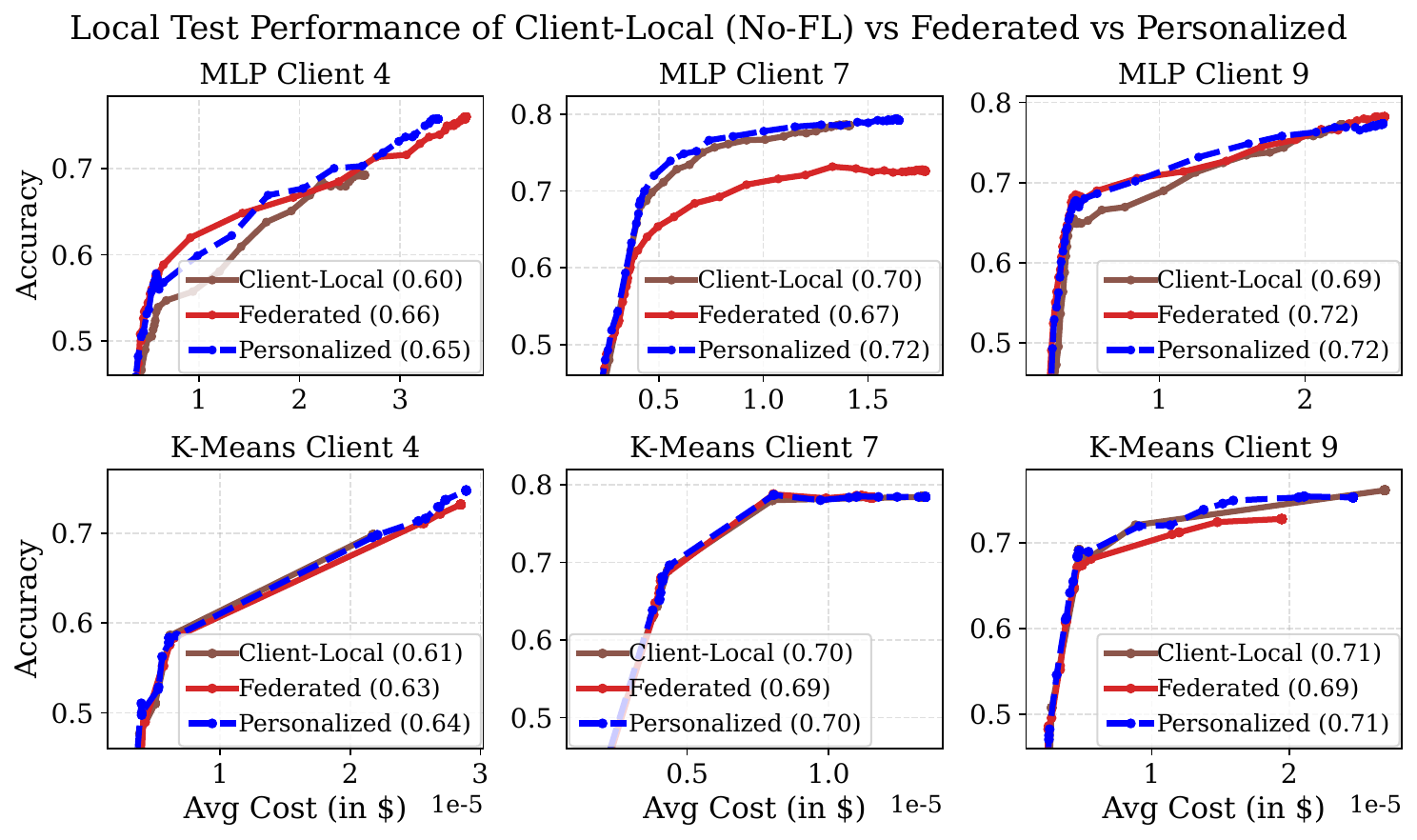}
  \caption{\textbf{(\proxr{}) Adaptive personalization under high heterogeneity for representative clients with both \mlp{} and \kmeans{}.}
  Local test set performance of client-local, federated, and adaptive routers.}
  \label{fig:adaptive_personalization_proxr_selected_clients}
\end{figure}

\begin{figure}[h]
  \centering
  \includegraphics[width=0.74\linewidth]{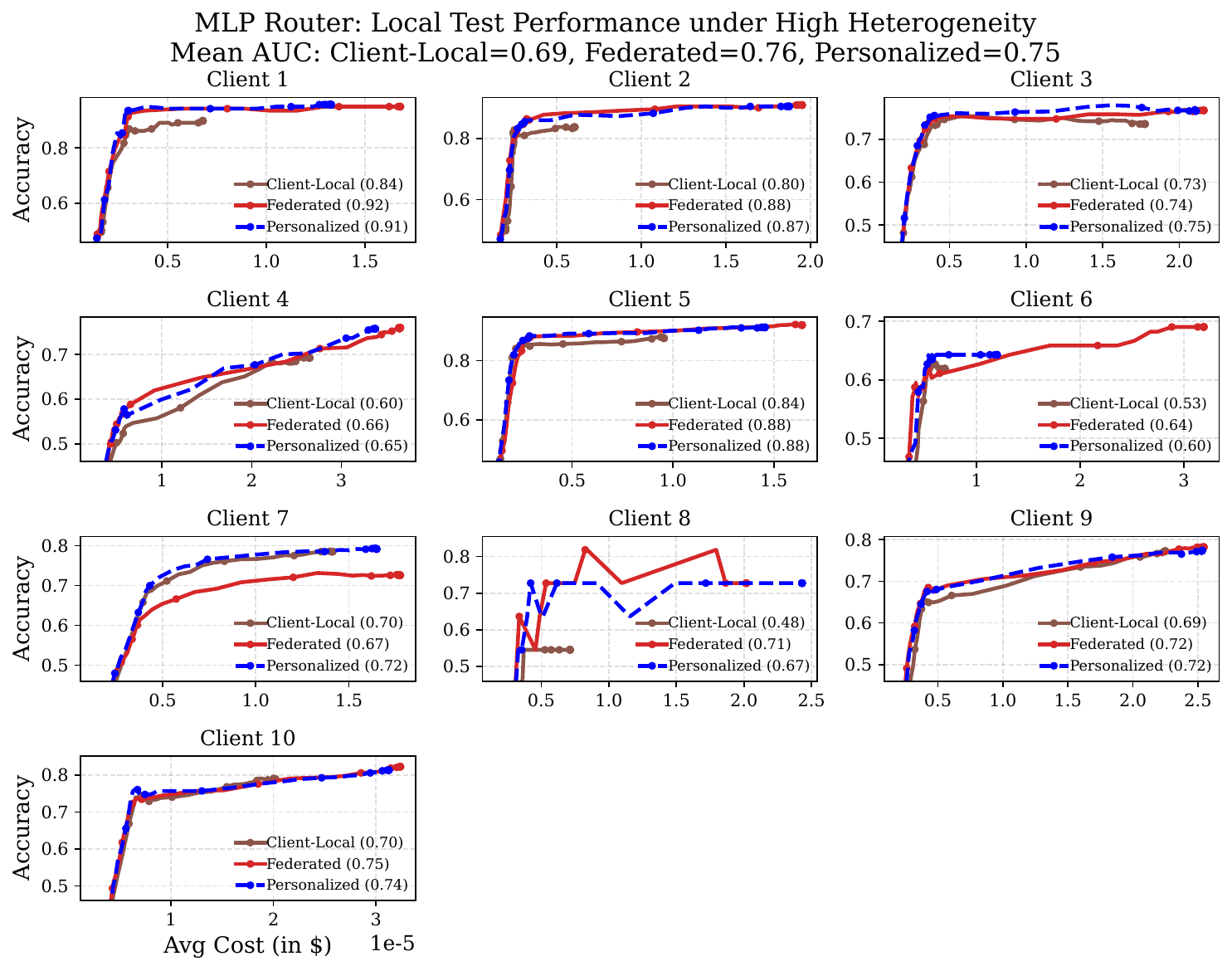}
  \caption{\textbf{(\proxr{}) Adaptive personalization under high heterogeneity for all clients with \mlp{}.}
  Local test set performance of client-local (no-FL), federated, and personalized routers.}
  \label{fig:adaptive_personalization_proxr_all_clients_mlp}
\end{figure}

\begin{figure}[h]
  \centering
  \includegraphics[width=0.74\linewidth]{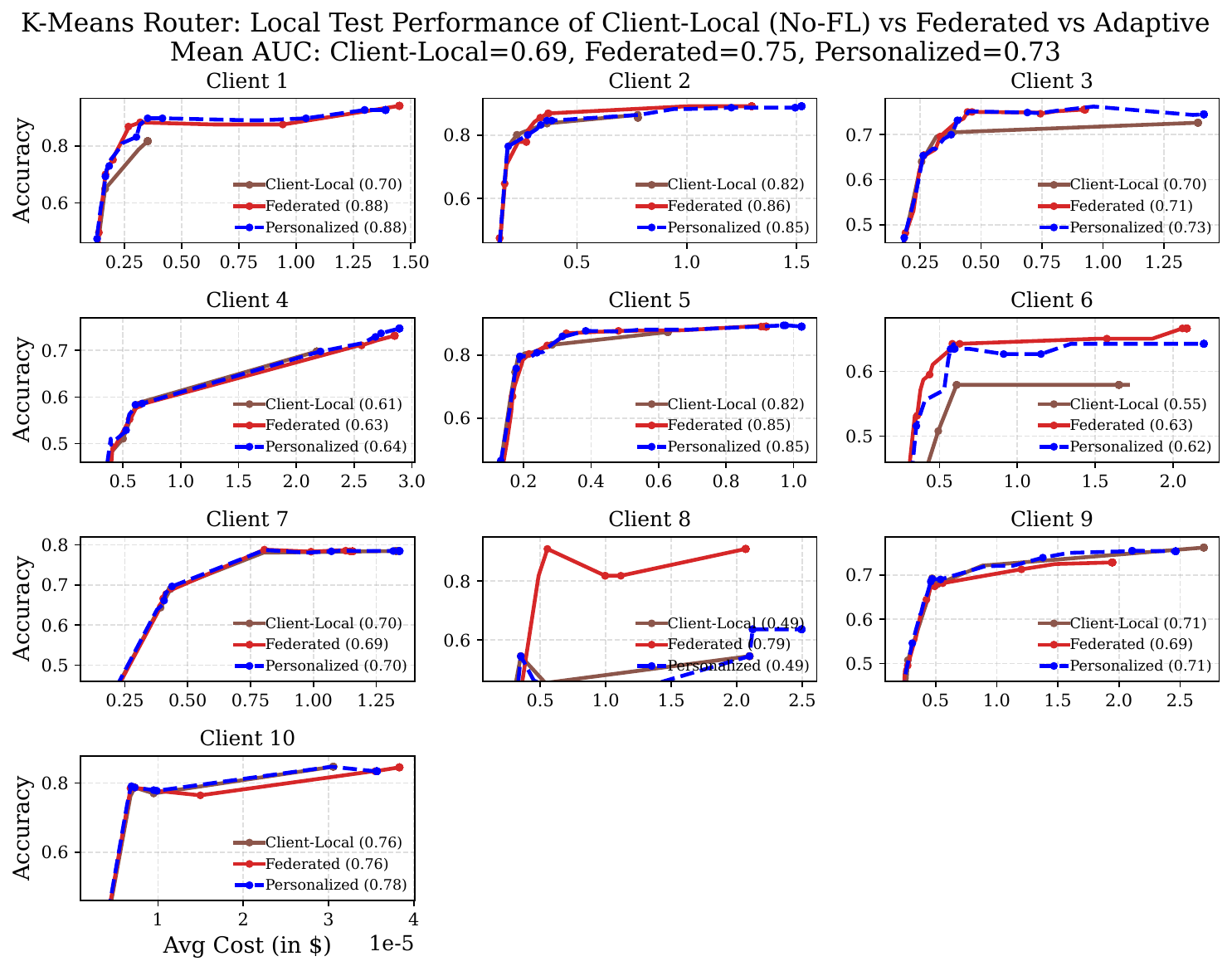}
  \caption{\textbf{(\proxr{}) Adaptive personalization under high heterogeneity for all clients with \kmeans{}.}
  Local test set performance of client-local (no-FL), federated, and personalized routers.}
  \label{fig:adaptive_personalization_proxr_all_clients_kmeans}
\end{figure}

\camera{
\FloatBarrier
\section{Experimental Results with SPROUT Dataset}
\label{apndx:sprout_results}

We further evaluate our framework on SPROUT dataset, introduced by
\citet{somerstep2025carrot}, which includes explanatory and
open-ended queries with graded response-quality scores obtained
through LLM evaluation against reference answers.
We compare federated and client-local routers with $N=10$ clients
under sparse query--model evaluations, using both \mlp{} and
\kmeans{}.

\Cref{fig:sprout_global_test} reports performance on the global
test distribution, while \Cref{fig:sprout_local_test} reports
local-test performance for representative clients.
Federated training improves both global generalization and local
routing performance, supporting its effectiveness with graded
quality supervision by LLM feedback.

\begin{figure}[htbp]
  \centering
  \includegraphics[width=0.52\linewidth]{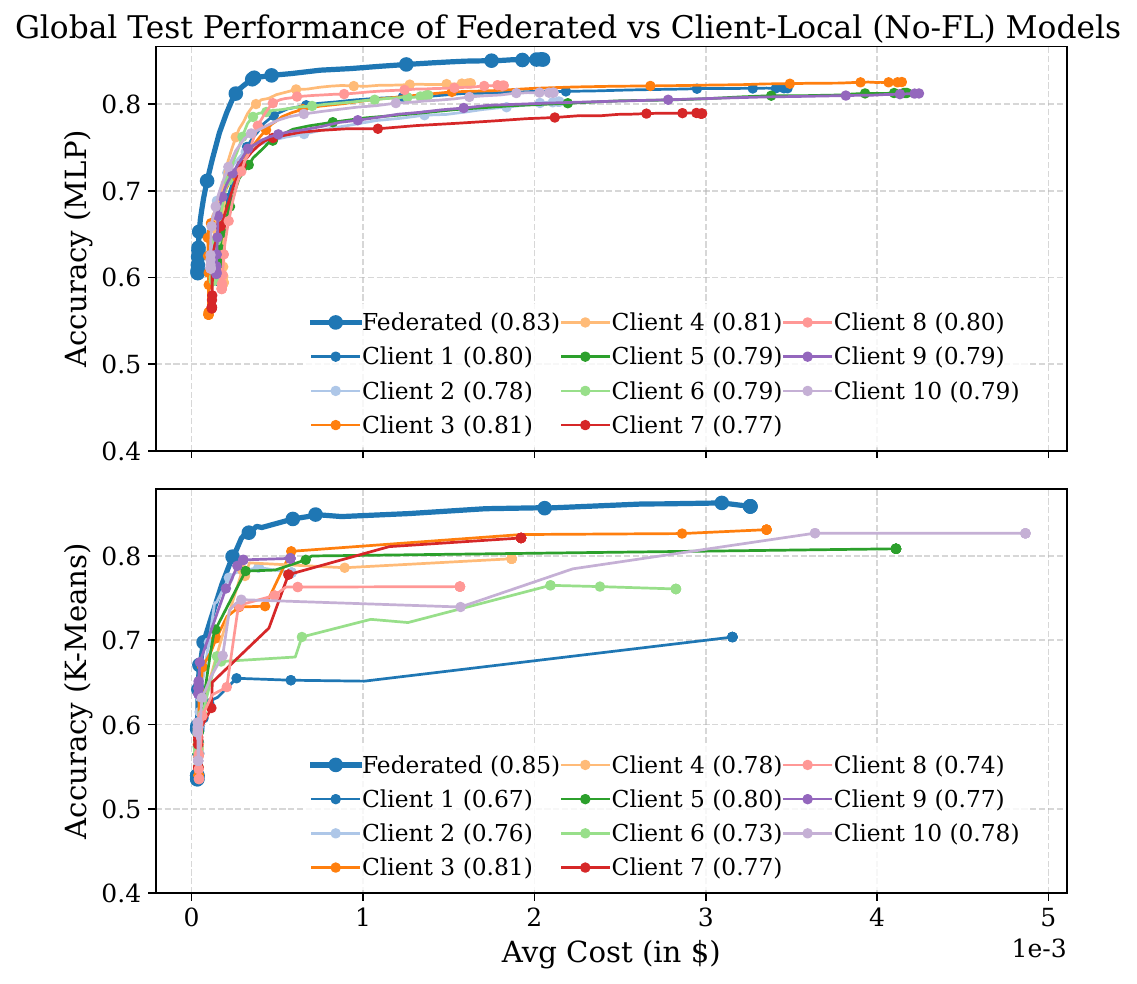}
  \caption{\camera{\textbf{SPROUT global-test performance.}
  Federated and client-local routers for \mlp{} and \kmeans{}.
  The accuracy axis denotes mean benchmark quality score, and
  parentheses denote normalized AUC.}}
  \label{fig:sprout_global_test}
\end{figure}

\begin{figure}[htbp]
  \centering
  \includegraphics[width=0.78\linewidth]{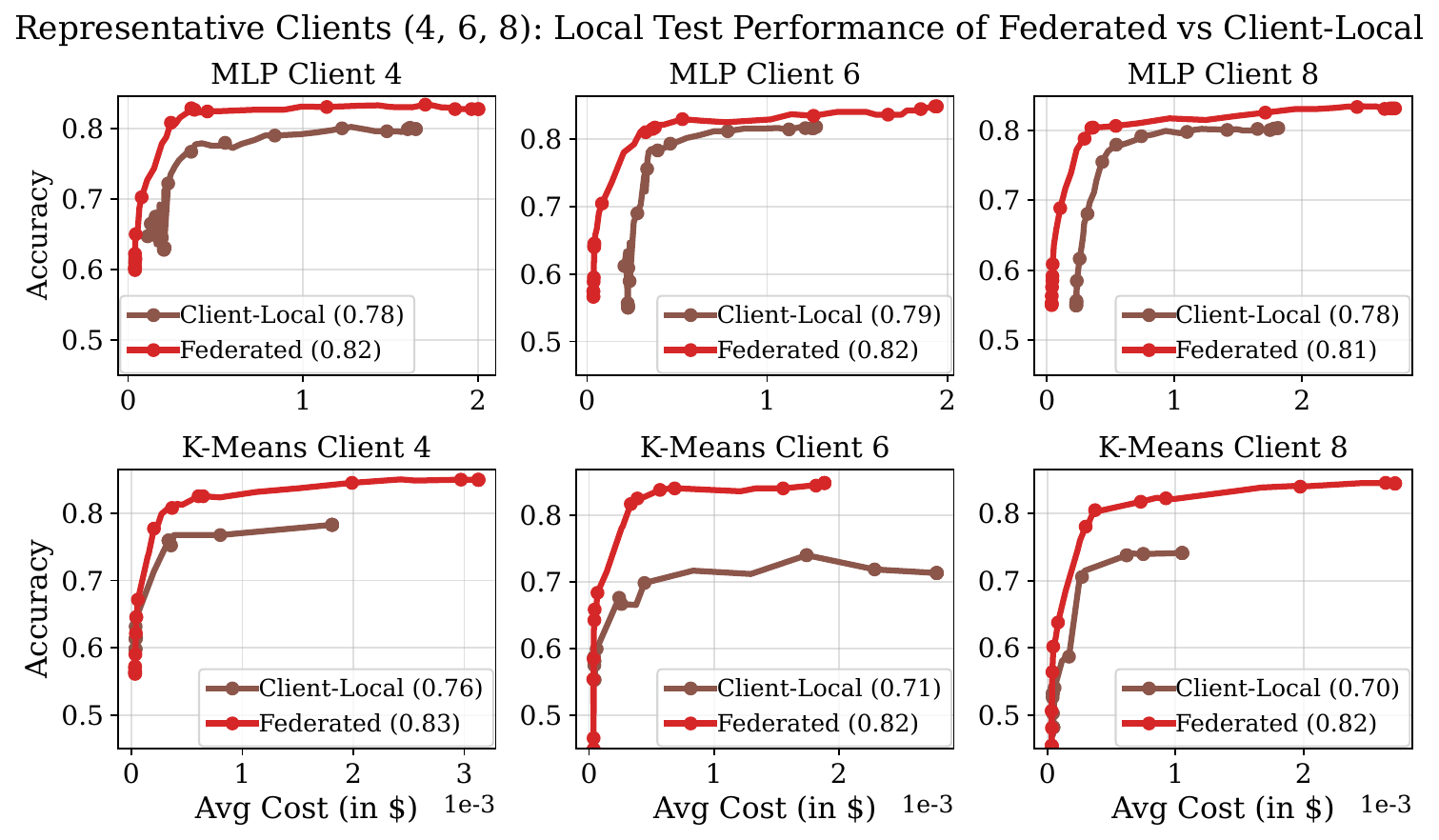}
  \caption{\camera{\textbf{SPROUT local-test performance.}
  Federated and client-local routers evaluated on the local test
  distributions of representative clients.
  The accuracy axis denotes mean benchmark quality score, and
  parentheses denote normalized AUC.}}
  \label{fig:sprout_local_test}
\end{figure}
\FloatBarrier
}

\section{Comparison with Federated EmbedLLM}
\label{apndx:embedllm_comparison}

Existing router baselines are typically designed for centralized training, while our setting assumes decentralized and sparse query--model evaluations. To the best of our knowledge, there is no prior federated LLM router baseline directly applicable to this setup. For a broader comparison, we adapt EmbedLLM \citep{zhuang2025embedllm}, a recent router that learns compact representations of language models for routing, to the federated setting. Specifically, we train EmbedLLM using the same client partitions, sparse one-model-per-query logging protocol, and federated simulation setup used for our proposed \mlp{} and \kmeans{} routers. This comparison serves two purposes: (i) to compare our proposed federated routers against a strong recent routing architecture adapted to FL, and (ii) to test whether the benefits of federation extend beyond the specific router designs proposed in this paper.

\begin{figure}[h]
  \centering
  \includegraphics[width=0.72\linewidth]{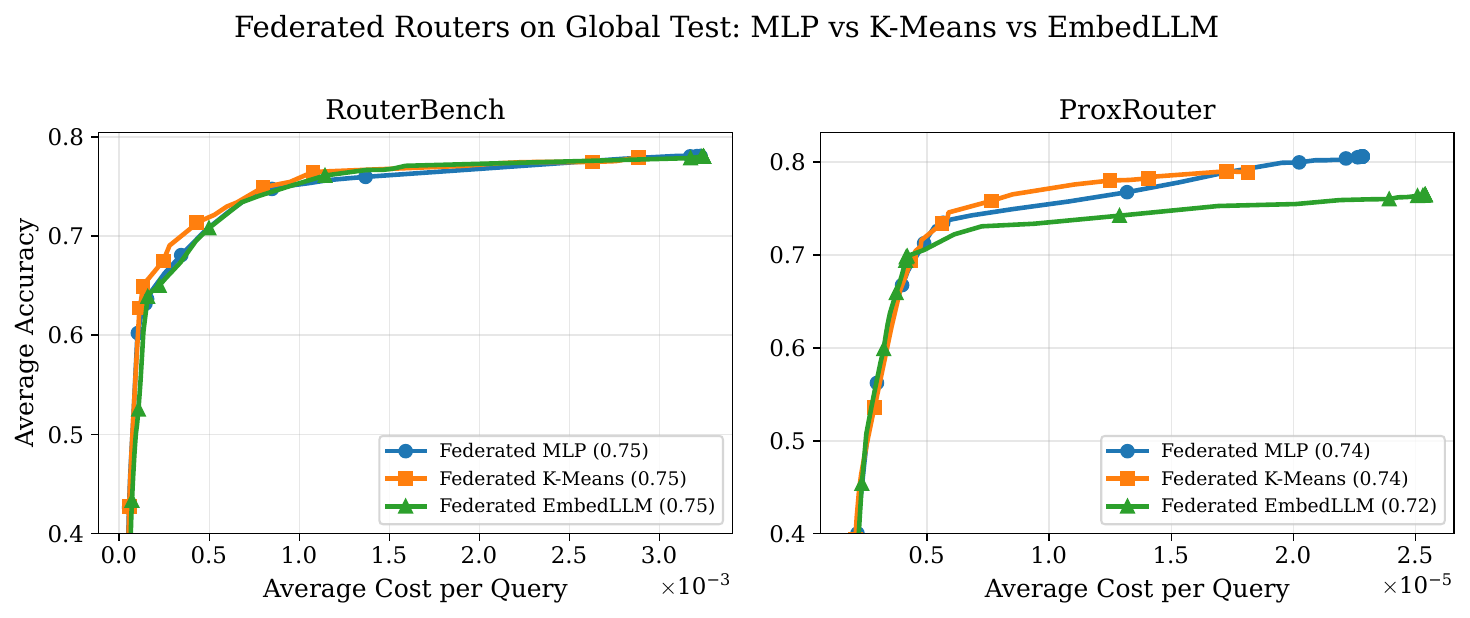}
  \caption{\textbf{Federated routers on the global test distribution with EmbedLLM.}
  We compare the proposed federated \mlp{} and federated \kmeans{} routers against a federated adaptation of EmbedLLM \citep{zhuang2025embedllm}. Results are shown on \routerbench{} and \proxr{}. Numbers in parentheses denote normalized AUC. The proposed methods outperform or are competitive with the federated EmbedLLM baseline across datasets.}
  \label{fig:embedllm_three_methods_global}
\end{figure}

\begin{figure}[h]
  \centering
  \includegraphics[width=0.72\linewidth]{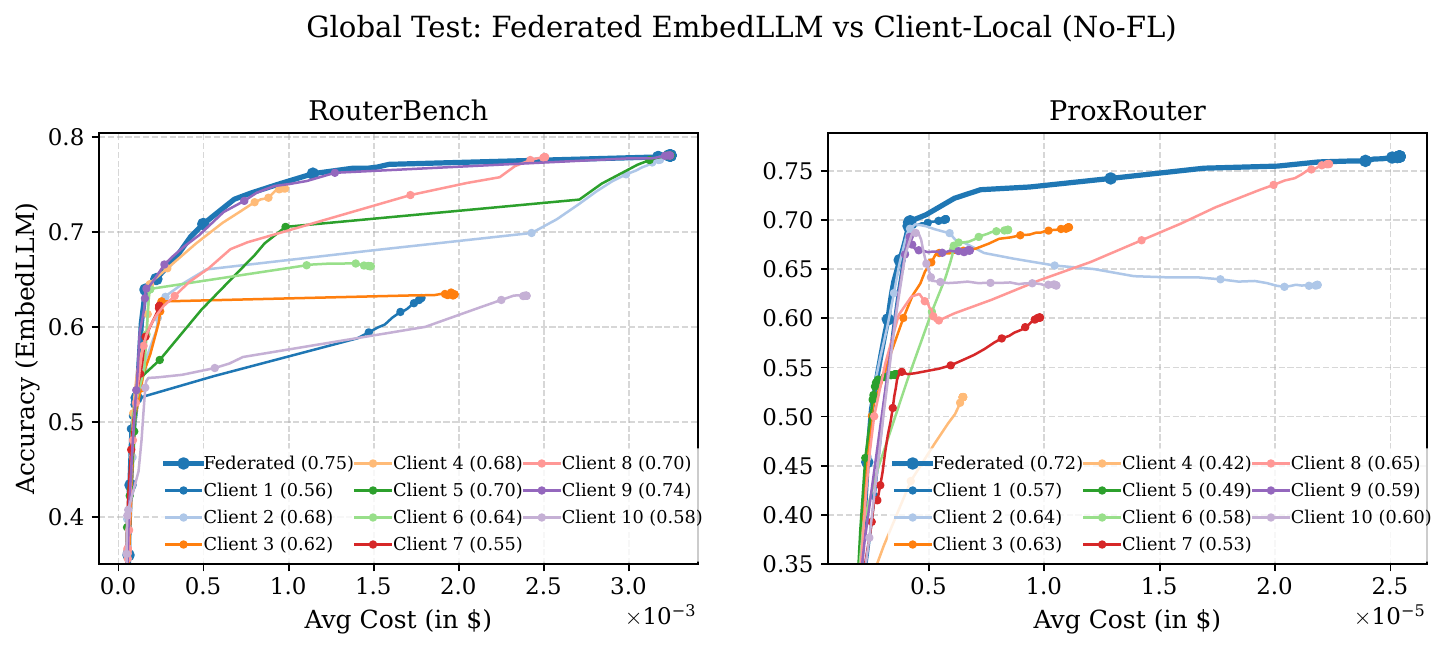}
  \caption{\textbf{Federated EmbedLLM vs.\ client-local EmbedLLM on the global test distribution.}
  We compare federated EmbedLLM against client-local, no-FL EmbedLLM routers trained independently on each client's local data. Results are shown on \routerbench{} and \proxr{}, with normalized AUC shown in parentheses. Federated training consistently improves global-distribution generalization over client-local training, supporting the value of federated collaboration beyond the proposed \mlp{} and \kmeans{} router architectures.}
  \label{fig:embedllm_fed_vs_local_global}
\end{figure}

As shown in \Cref{fig:embedllm_three_methods_global}, the proposed federated \mlp{} and \kmeans{} routers are competitive with or improve upon the federated EmbedLLM adaptation on the global test distribution. In addition, \Cref{fig:embedllm_fed_vs_local_global} shows that applying federated training to EmbedLLM itself improves over client-local EmbedLLM training. These results reinforce the central conclusion of our work: federated collaboration improves router learning under fragmented client data and sparse model coverage, and this benefit is not limited to a single router architecture.

\FloatBarrier

\section{Additional Seed Results}
\label{apndx:additional_seed_results}

We repeat the main global-generalization experiment with different random seeds.
The results are shown in \Cref{fig:additional_seed_results_routerbench} for
\routerbench{} and \Cref{fig:additional_seed_results_proxr} for \proxr{}.
Across both datasets, the results are consistent across seeds, confirming that
federated learning improves over client-local routers under different random
simulations.

\begin{figure}[h]
  \centering

  \begin{subfigure}[t]{0.48\linewidth}
    \centering
    \includegraphics[width=\linewidth]{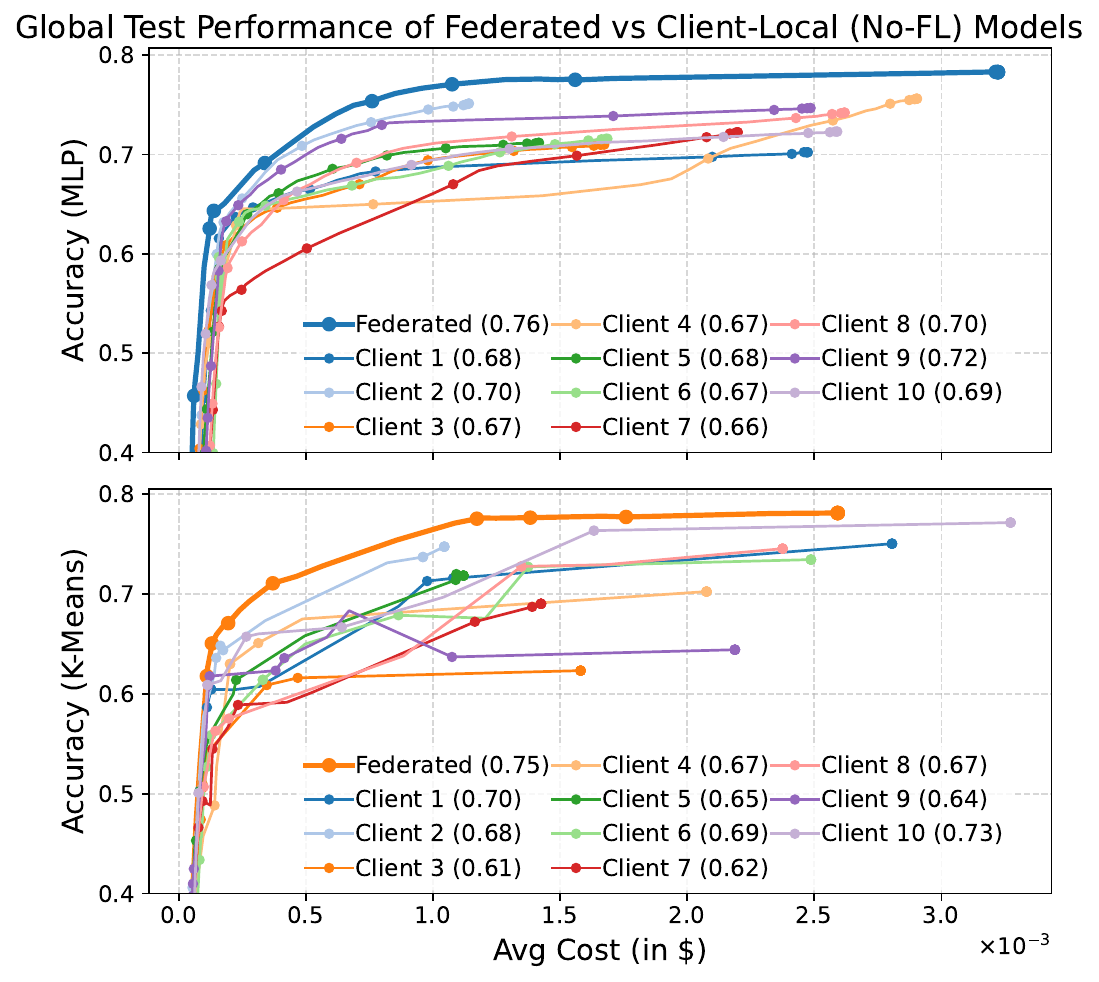}
    \caption{Seed 1.}
  \end{subfigure}
  \hfill
  \begin{subfigure}[t]{0.48\linewidth}
    \centering
    \includegraphics[width=\linewidth]{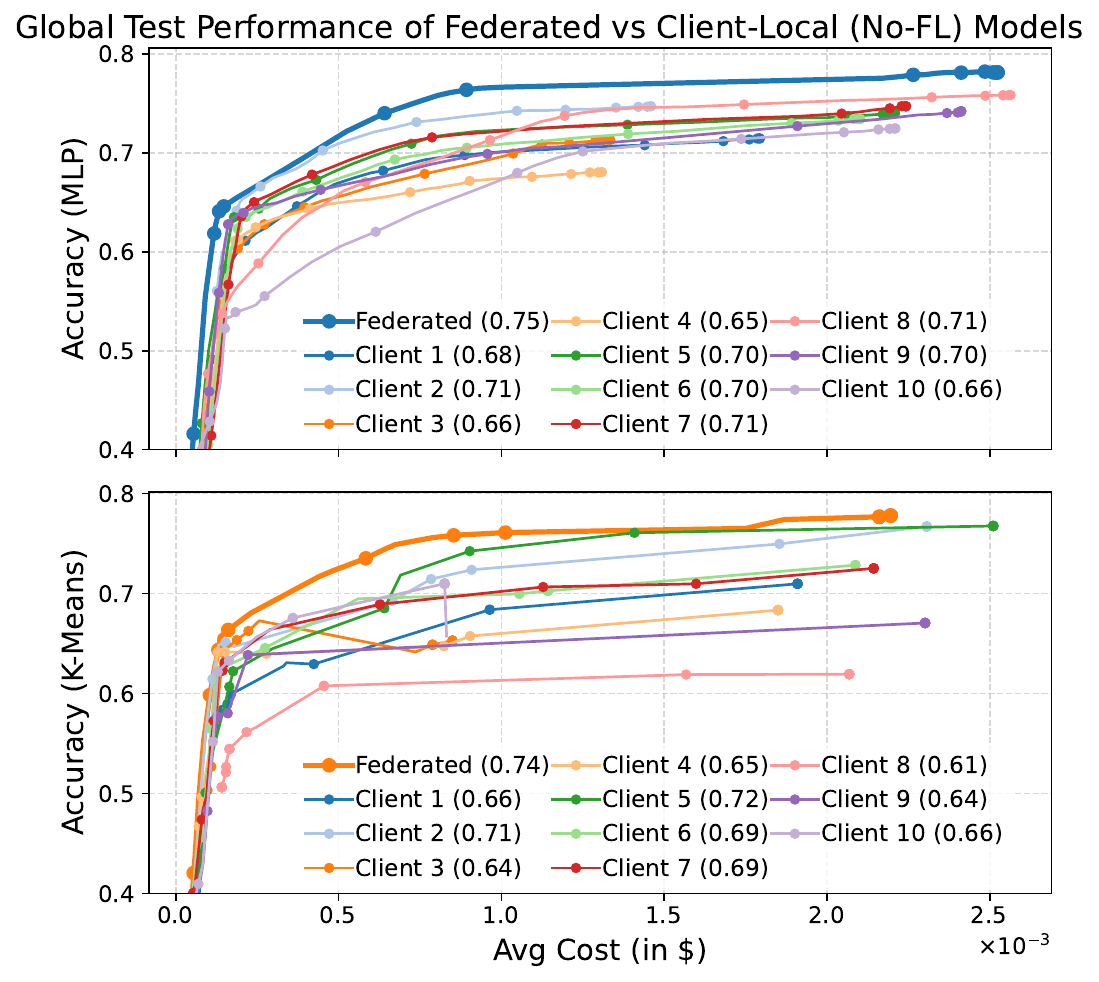}
    \caption{Seed 2.}
  \end{subfigure}

  \caption{\textbf{\routerbench{} global-generalization results under different seeds.}
  Federated routers consistently outperform client-local routers across both
  random seeds.}
  \label{fig:additional_seed_results_routerbench}
\end{figure}

\begin{figure}[h]
  \centering

  \begin{subfigure}[t]{0.48\linewidth}
    \centering
    \includegraphics[width=\linewidth]{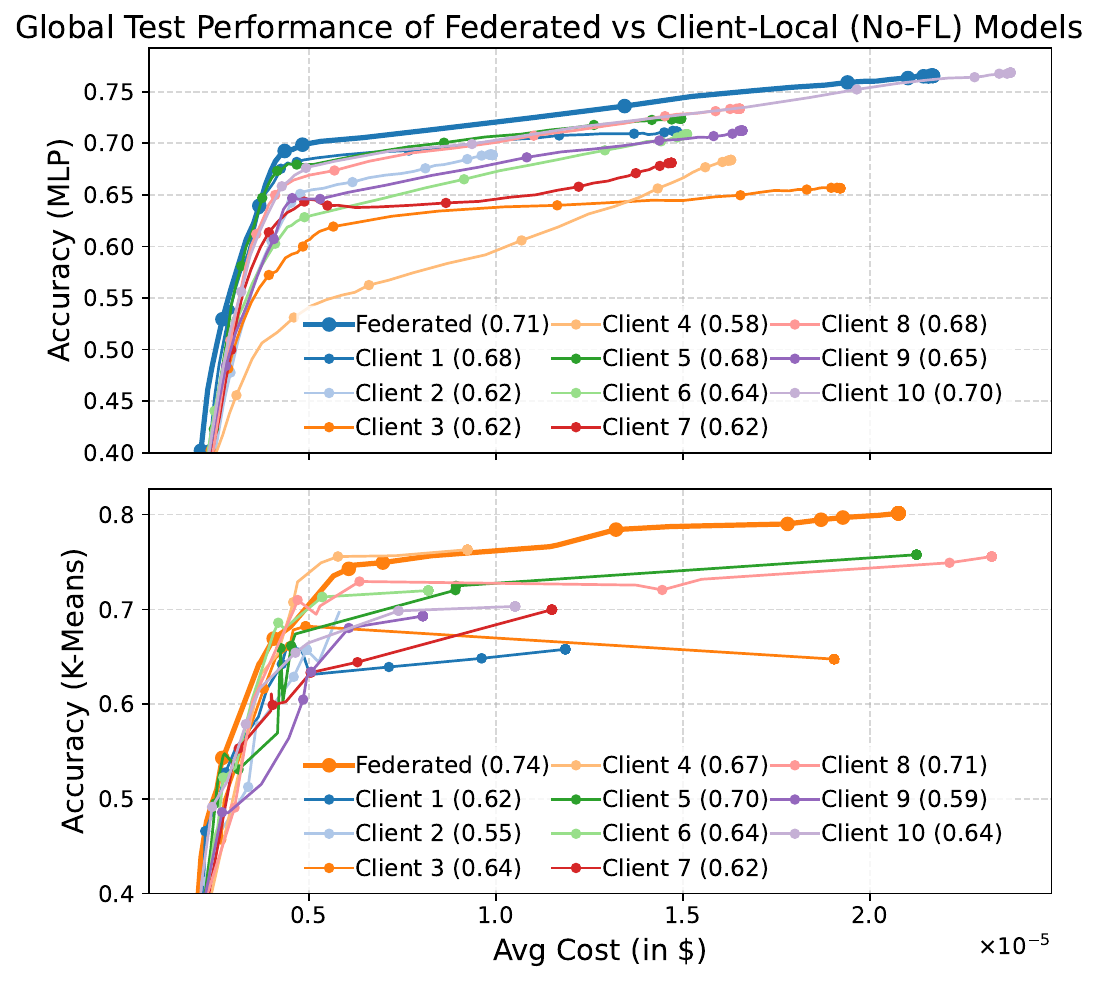}
    \caption{Seed 1.}
  \end{subfigure}
  \hfill
  \begin{subfigure}[t]{0.48\linewidth}
    \centering
    \includegraphics[width=\linewidth]{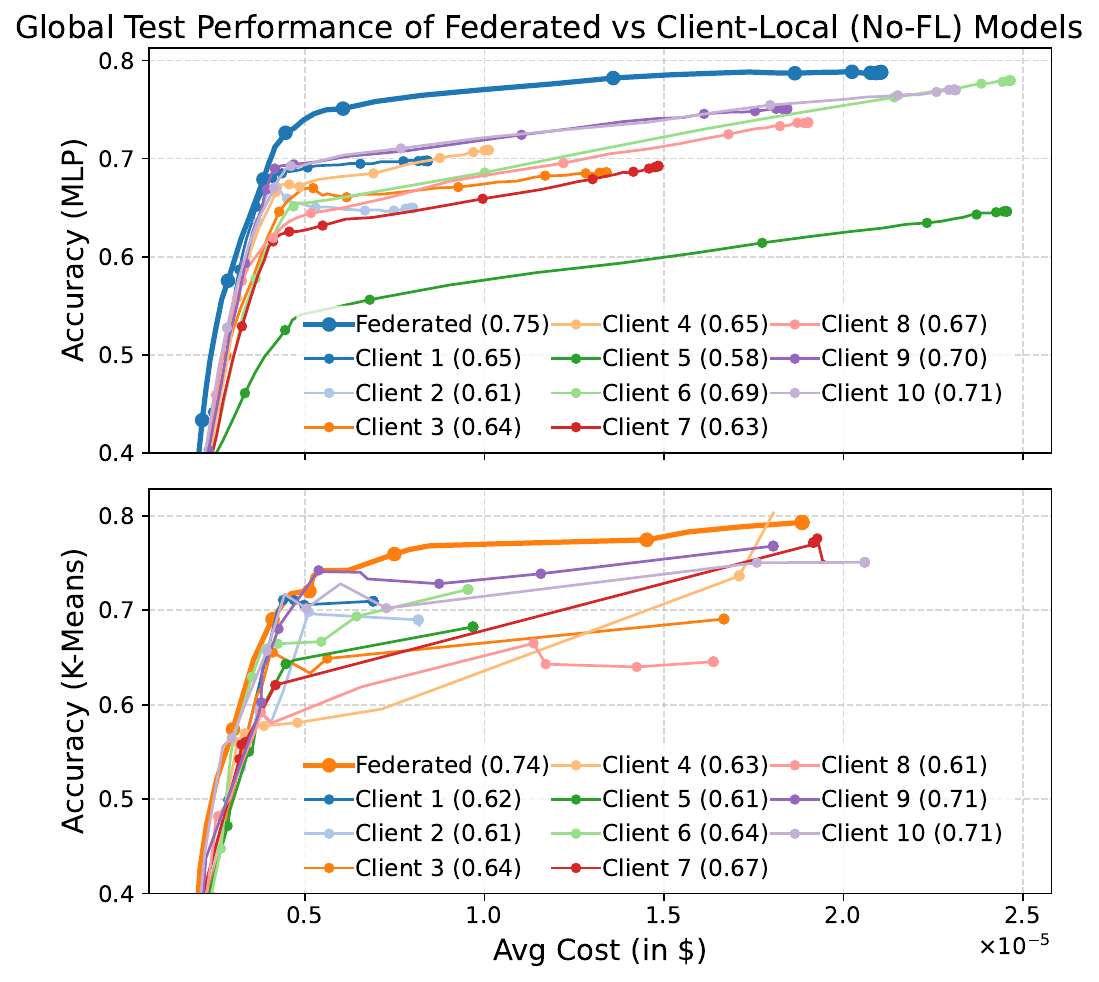}
    \caption{Seed 2.}
  \end{subfigure}

  \caption{\textbf{\proxr{} global-generalization results under different seeds.}
  Federated routers consistently outperform client-local routers across both
  random seeds.}
  \label{fig:additional_seed_results_proxr}
\end{figure}

\subsection{Compute Resources}
\label{app:compute_resources}

All experiments reported in this paper were run on a single NVIDIA V100 GPU. 
The full set of experiments for each benchmark and router configuration can be reproduced within a few hours on this hardware. 
Our experiments use lightweight router models and fixed pretrained query embeddings, so the computational cost is dominated by router training, clustering, and evaluation rather than by training or fine-tuning.

\FloatBarrier
\input{theorynew}

\end{document}

%% file: theorynew.tex
\section{Theory}
\label{sec:theoryapp}
In this section, we present the proofs and explanations for the theoretical claims in \Cref{sect:theory}. In \Cref{sec:theory_mlp}, we present the theoretical results for $\mlp$, and in
\Cref{sec:theory_kmeans}, we present the results for $\kmeans$.

\subsection{Theoretical Results for Federated \mlp{}}
\label{sec:theory_mlp}

\paragraph{Setup and data model.}
Recall that we consider $N$ clients indexed by $i \in \{1,\dots,N\}$ and a fixed set of $M$ candidate models indexed by $m \in \mcal := \{1,\dots,M\}$.
Client $i$ holds a local dataset of size $\D_i := |\dcal_i|$, and $\D=\sum_i^N\D_i$ denotes the number of total samples across clients in the federated setting.
\[
\dcal_i \;=\; \Bigl\{ \bigl(\xb_i^{(j)},\, m_i^{(j)},\, \hac_i^{(j)},\, \hcs_i^{(j)}\bigr) \Bigr\}_{j=1}^{\D_i},
\]
where $\xb_i^{(j)}\in\X$ is the query embedding,\footnote{We use \emph{query} and \emph{query embedding} interchangeably throughout this section.}
and $m_i^{(j)}\in\mcal$ is selected by the logging (data-collection) policy.
Given $(m_i^{(j)},\xb_i^{(j)})$, the system samples a model response, from which we compute an accuracy signal
$\hac_i^{(j)}\in[0,1]$ (binary in our setting, $\hac_i^{(j)}\in\{0,1\}$, although the analysis extends to any bounded reward)
and $\hcs_i^{(j)}\in[0,\cm]$.
Equivalently, this can be modeled using an unknown joint conditional distribution $\rho(\cdot,\cdot\mid m,\xb)$ such that whenever $\xb$ is routed to model $m\in\mcal$ we jointly sample $(\ac,\cs) \sim \rho(\cdot,\cdot|m,\xb)$. Let $\ac(\xb,m)$ and $\cs(\xb,m)$ denote the associated expected accuracy and cost, thus we have
\[
(\hac_i^{(j)},\hcs_i^{(j)})  \sim \rho(\cdot,\cdot\mid \xb_i^{(j)},m_i^{(j)}),\qquad
\]
and 
\[
\E\!\left[\hac_i^{(j)} \mid \xb_i^{(j)}, m_i^{(j)}\right] = \ac\left(\xb_i^{(j)}, m_i^{(j)}\right),\qquad 
\E\!\left[\hcs_i^{(j)} \mid \xb_i^{(j)}, m_i^{(j)}\right] = \cs\left(\xb_i^{(j)},m_i^{(j)}\right).
\]
Since $\ac_i^{(j)} \in [0,1]$ and $\cs_i^{(j)}\in[0,\cm]$, both random variables are sub-Gaussian (not independent) because they are bounded.

\paragraph{\mlp{} estimators and objectives.} To facilitate routing, we learn a neural network that estimates $\ac(m,\xb)$ and $\cs(m,\xb)$ for all models $m\in\mcal$ given an input $\xb$. We parameterize the neural network as follows: 
let $\wcal \subseteq \R^{d_\text{MLP}}$ denote the parameter space of the \mlp{} (shared trunk + $M$ heads as in \Cref{alg:mlp_fl}). The parameters $\mlpw \in \wcal$ induce, for each $m\in\mcal$, an accuracy-estimator function
$\aes_{\mlpw}(\cdot,\cdot):\X\times\mcal\to[0,1]$ and a cost-estimator function $\ces_{\mlpw}(\cdot,\cdot):\X\times\mcal\to[0,\cm]$. Accordingly, we define the estimator function class,
\begin{align}
    \fcal = \left\{
 \bigl[\aes_{\mlpw}(\cdot,\cdot),\,\ces_{\mlpw}(\cdot,\cdot)\bigr] 
:\ \mlpw\in\wcal
\right\}. \label{eq:function_class}
\end{align}

The loss for the $(i,j)^{\text{th}}$ sample is the $L_2$ norm of the prediction error that renders the loss function (\Cref{alg:mlp_fl}):
\[
{\mathcal{L}}_i(\mlpw)
:=
\frac{1}{\D_i}\sum_{j=1}^{\D_i}
\left(
\left(\aesw\left(\xb_i^{(j)},m_i^{(j)}\right)- \hac_i^{(j)}\right)^2
+
\left(\cesw\left(\xb_i^{(j)},m_i^{(j)}\right)-\hcs_i^{(j)}\right)^2
\right).
\]
The global objective function is defined as the average of local objectives weighted with the number of local samples: 
\[
\mathcal{L}(\mlpw) := \sum_{i=1}^N \frac{\D_i}{\D} \,\mathcal{L}_i(\mlpw).
\]

\begin{assumption}[Realizability] \label{assump:mlp_realizable}
The neural network function class, $\wcal$, is expressive enough such that there exist $\mlpw^{\star}\in \wcal$ such that 
\(
\left[\aopt(\xb,m), \copt(\xb,m)\right] = [\ac(\xb,m), \cs(\xb,m)]
\)
for all $\mathbf{x}\in\Xcal$ and $m\in\mcal$ .
\end{assumption}

\begin{assumption}[Smoothness]
\label{assump:mlp_smooth}
Each local objective $\mathcal{L}_i$ is $L$-smooth: for all $i\in[N]$ and for every $\mlpw_1,\mlpw_2\in\wcal$,
\[
\|\nabla \mathcal{L}_i(\mlpw_1)-\nabla \mathcal{L}_i(\mlpw_2)\|
\le
L\|\mlpw_1-\mlpw_2\|.
\]
\end{assumption}

\begin{assumption}[Unbiased  gradients and bounded heterogeneity]
\label{assump:mlp_grad_hetero}
For each client $i$,  stochastic gradient over a mini batch $\xi$ denoted by $\nabla \mathcal{L}_i(\mlpw,\xi)$ satisfies
\[
\E_\xi\!\left[\nabla \mathcal{L}_i(\mlpw, \xi)\right]=\nabla \mathcal{L}_i(\mlpw),
\qquad
\E_\xi\!\left[\left\|\nabla \mathcal{L}_i(\mlpw,\xi)-\nabla\mathcal{L}_i(\mlpw)\right\|^2\right]\le \sigma^2.
\]
for every $\mlpw\in\wcal$ where $\xi$ denote the stochasticity in stochastic gradients and is independent to everything else. Moreover, there exist $\beta^2\ge 1$ and $\kappa^2\ge 0$ such that
\[
\sum_{i=1}^N w_i \|\nabla \mathcal{L}_i(\mlpw)\|^2
\le
\beta^2\left\|\sum_{i=1}^N w_i \nabla \mathcal{L}_i(\mlpw)\right\|^2 + \kappa^2.
\]
\end{assumption}

The assumption~\ref{assump:mlp_realizable} is common in learning theory literature and is mild given a sufficiently large neural network \citep{zhang2023mathematical, foster2023foundations}. Assumptions~\ref{assump:mlp_smooth} and \ref{assump:mlp_grad_hetero} are standard in the federated learning literature \citep{wang2020tackling, koloskova2020unified}.

\subsubsection{Convergence analysis} We first provide the analysis for the convergence of \Cref{alg:mlp_fl}. 

\begin{proposition}[Convergence of federated optimization error in \Cref{alg:mlp_fl}]
\label{prop:mlp_fedavg_opt}
Under Assumptions \ref{assump:mlp_smooth} and \ref{assump:mlp_grad_hetero}, and choosing step size
$\eta = \Theta\!\bigl(\sqrt{N/(\tau T)}\bigr)$ (up to constants including $L$), \Cref{alg:mlp_fl} produces iterates $\{\mlpw^{(t)}\}_{t=1}^T$ such that
\[
\min_{t\in\{1,\dots,T\}} \E\!\left[\bigl\|\nabla {\mathcal{L}}(\mlpw^{(t)})\bigr\|^2\right]
\le
\mathcal{O}\!\left(\frac{1}{\sqrt{N\tau T}}\right) + \mathcal{O}\!\left(\frac{A\sigma^2}{\sqrt{N\tau T}}\right)
+
\mathcal{O}\!\left(\frac{N\sigma^2(\tau-1)}{\tau T}\right)
+
\mathcal{O}\!\left(\frac{N\kappa^2(\tau-1)}{T}\right),
\]
where $A = N\sum_{i=1}^N(\D_i/\D)^2$.
\end{proposition}
\begin{proof}
    The proof follows Theorem 1 of \cite{wang2020tackling} with the same number of local steps $\tau$. 
\end{proof}

The proof extends to partial client participation using standard arguments from literature \cite{DBLP:journals/corr/abs-2010-01243_client_selection_in_fl,jhunjhunwala2022fedvarp}.

\begin{remark}[Implication]
When all the clients have the same amount of data, i.e., $\D_i/\D = 1/N\;\forall i\in [N]$, we obtain a linear speedup of $N$ (number of clients) in the slowest decaying term $1/\sqrt{T}$ in the convergence to a stationary-point.
\end{remark}

\subsubsection{Sample complexity analysis}
We now assume ideal optimization and analyze statistical error to show that more data can reduce the routing suboptimality. We compare local empirical risk minimization (ERM) in client's local-only training vs federated ERM.
Let $\widehat{\mlpw}_i \in \arg\min_{\mlpw}\mathcal{L}_i(\mlpw)$ denote the local ERM of client $i$, and
\[
\widehat{\mlpw}_{\mathrm{fed}} \in \arg\min_{\mlpw}\ \sum_{i=1}^N w_i \mathcal{L}_i(\mlpw)
\]
denote the federated ERM (equivalently, ERM in $\dcal:=\bigcup_{i=1}^N \dcal_i$).

Each client $i$ faces its own test-time query distribution $\dt_i$ (possibly different from the distribution of its dataset $\dcal_i$) and has a cost tradeoff $\lmb_i\ge 0$.
Define the \emph{true} expected utility of query $\xb$ evaluated with model $m$
\[
U_i(\xb,m) := \ac(\xb,m) - \lmb_i \cs(\xb,m) ,
\]
and the {estimated} utility under parameters $\mlpw$,
\[
\widehat U_i(\xb,m;\mlpw) := \aesw(\xb,m) - \lmb_i \cesw(\xb,m).
\]
we also have $\widehat U_i(\xb,m;\mlpw^{\star}) = U_i(\xb,m)$ by Assumption~\ref{assump:mlp_realizable}. Let the optimal router and a learned router policies be
\[
\pi_i^\star(\xb) \in \arg\max_{m\in\mcal} U_i(\xb,m),
\qquad
\widehat \pi_{i,\mlpw}(\xb) \in \arg\max_{m\in\mcal} \widehat U_i(\xb,m;\mlpw).
\]

Let $\Pi_i := \{\widehat \pi_{i,\mlpw}:\mlpw\in\wcal\}$ be the induced routing-policy class at client $i$ using model $\mlpw$.

\begin{definition}[Router suboptimality]
\label{def:mlp_subopt}We define the suboptimality of the router as the utility gap between the optimal router and the learned router under the test distribution. For any policy $\widehat \pi$, we define
\[
\mathrm{Subopt}_i(\widehat\pi)
:=
\E_{\xb\sim\dt_i}\!\left[U_i\!\bigl(\xb,\pi_i^\star(\xb)\bigr)\right]
-
\E_{\xb\sim\dt_i}\!\left[U_i\!\bigl(\xb,\widehat \pi(\xb)\bigr)\right].
\]
\end{definition}

\paragraph{A data-dependent coverage coefficient.}
To separate \emph{estimation error on the training support} from \emph{test-time distribution shift},
we define a (client-specific) coefficient that measures how well the training sample controls test expectations.
For any dataset $S=\{(\xb^{(j)},m^{(j)})\}_{j=1}^{|S|}$, define
\begin{align}
\Gamma^{(i)}_{\mathrm{acc}}(S)
&:=
\sup_{\substack{\pi \in \Pi_i \\ \mlpw\in\wcal}}
\frac{\left|\E_{\xb\sim \dt_i}\!\left[\aopt(\xb,\pi(\xb))-\aesw(\xb,\pi(\xb))\right]\right|}
{\sqrt{\sum_{j=1}^{|S|}\left(
\aopt\left(\xb^{(j)},m^{(j)}\right)
-\aesw\left(\xb^{(j)},m^{(j)}\right)\right)^2}}, \label{eq:Gamma_acc_def}\\
\Gamma^{(i)}_{\mathrm{cost}}(S)
&:=
\sup_{\substack{\pi \in \Pi_i \\ \mlpw\in\wcal}}
\frac{\left|\E_{\xb\sim \dt_i}\!\left[\copt(\xb,\pi(\xb))-\cesw(\xb,\pi(\xb))\right]\right|}
{\sqrt{\sum_{j=1}^{|S|}\left(
\copt\left(\xb^{(j)},m^{(j)}\right)
-\cesw\left(\xb^{(j)},m^{(j)}\right)\right)^2}}, \label{eq:Gamma_cost_def}
\end{align}
and define $\Gamma^{(i)}(S):=\max\{\Gamma^{(i)}_{\mathrm{acc}}(S),\Gamma^{(i)}_{\mathrm{cost}}(S)\}$.
Intuitively, smaller $\Gamma^{(i)}(S)$ indicates better coverage of client $i$'s test distribution by dataset $S$. This definition is an adapted version of the concentrability coefficient in \cite{zhanprovable} for our setting. It quantifies the uncertainty inherently present in the dataset. Variations of this are used to quantify the difficulty of offline (via coverage) and online learning (via Eluder dimension) problems in learning  theory literature \citep{zhang2023mathematical}.  

\begin{remark}
    Note that the coverage coefficients used in the weaker results presented in the maintext are those of the global test distribution, which are defined as follows
    \begin{align*}
\Gamma_{\mathrm{acc}}(S)
&:=
\sup_{\substack{\pi \in \Pi_{\lmb} \\ \mlpw\in\wcal}}
\frac{\left|\E_{\xb\sim \dt}\!\left[\aopt(\xb,\pi(\xb))-\aesw(\xb,\pi(\xb))\right]\right|}
{\sqrt{\sum_{j=1}^{|S|}\left(
\aopt\left(\xb^{(j)},m^{(j)}\right)
-\aesw\left(\xb^{(j)},m^{(j)}\right)\right)^2}}, \\
\Gamma_{\mathrm{cost}}(S)
&:=
\sup_{\substack{\pi \in \Pi_{\lmb} \\ \mlpw\in\wcal}}
\frac{\left|\E_{\xb\sim \dt}\!\left[\copt(\xb,\pi(\xb))-\cesw(\xb,\pi(\xb))\right]\right|}
{\sqrt{\sum_{j=1}^{|S|}\left(
\copt\left(\xb^{(j)},m^{(j)}\right)
-\cesw\left(\xb^{(j)},m^{(j)}\right)\right)^2}}, 
\end{align*}
and $\Gamma(S) := \max{\Gamma_{\mathrm{acc}}(S), \Gamma_{\mathrm{cost}}(S)}$, here $\Pi_{\lmb}$ is the set of routing policies induced by the tradeoff parameter $\lmb$ for $\mlpw \in \wcal$. The proof for the maintext version follows the same arguments as its stronger counterpart Theorem~\ref{thm:mlp_subopt_bounds}.
\end{remark}

\begin{lemma}[Lemma 34 in \cite{foster2023foundations}]
\label{lemma:th1}
    Let $\{X_t\}_{t \le T}$ be any sequence of real-valued random variables adapted
to a filtration $(\mathcal{F}_t)_{t \le T}$. Then for any $\delta \in (0,1)$,
with probability at least $1-\delta$, the following holds for any
fixed $T' \le T$:
\begin{equation}
    \sum_{t=1}^{T'} X_t
    \;\le\;
    \sum_{t=1}^{T'} \log\!\bigl( \mathbb{E}_{t-1}[e^{X_t}] \bigr)
    \;+\;
    \log(\delta^{-1}).
\end{equation}
\end{lemma}

In the following, we will state and prove a vector adapted version of Theorem 13.15 from \cite{zhang2023mathematical} under zero misspecification error for completeness.
\begin{lemma}[Adapted version of Theorem 2.29 + Ville's inequality in \cite{zhang2023mathematical}]
Let $X_1,\dots,X_n$ be conditionally independent (adapted to filtration $\mathcal{F}_t$) zero-mean $\sigma^2$-sub-Gaussian random
variables. Then for any $\delta \in (0,1)$, with probability at least $1-\delta$,
\begin{equation}
    \sum_{i=1}^n X_i^2
    \;\le\;
    n\sigma^2
    +
    2\sigma^2 \sqrt{\,n \log(\delta^{-1})\,}
    +
    2\sigma^2 \log(\delta^{-1}) \leq 2n\sigma^2 + 3\sigma^2 \log(\delta^{-1}).
\end{equation}
\label{lemma:th2}
\end{lemma}
\begin{lemma}\label{lem:cs-mgf}
 Let \(u = (u_1, u_2) \in \mathbb{R}^2\) and let \(\varepsilon = (\varepsilon^{(1)}, \varepsilon^{(2)})\) be a random vector. Suppose that, conditioned on an event \(Y\), the random variables \(\varepsilon^{(1)}\) and \(\varepsilon^{(2)}\) are subgaussian with zero-mean with parameters \(\sigma_1^2\) and \(\sigma_2^2\), respectively (not necessarily independent). Then we have

\[
\ln \mathbb E\!\left[\exp(\langle u,\varepsilon\rangle)\mid Y\right]
\le u_1^2\sigma_1^2+u_2^2\sigma_2^2
\le \sigma_{\max}^2\|u\|_2^2.
\]
\end{lemma}
\begin{proof}
By Cauchy--Schwarz 
$\mathbb E\!\left[e^{u_1\varepsilon^{(1)}+u_2\varepsilon^{(2)}}\mid\cdot\right]
\le
\Big(\mathbb E[e^{2u_1\varepsilon^{(1)}}\mid\cdot]\Big)^{1/2}
\Big(\mathbb E[e^{2u_2\varepsilon^{(2)}}\mid\cdot]\Big)^{1/2}$ - no independence required. 
Taking logs and applying the conditional sub-Gaussian bounds yields $\ln(\cdot)\le \tfrac12\cdot\tfrac{(2u_1)^2\sigma_1^2}{2}+\tfrac12\cdot\tfrac{(2u_2)^2\sigma_2^2}{2}
= u_1^2\sigma_1^2+u_2^2\sigma_2^2
\le \sigma_{\max}^2\|u\|_2^2.$
\end{proof}

\begin{lemma}[In sample error - vector adapted version Theorem 13.15 \cite{zhang2023mathematical}]
\label{lem:mlp_ise_scalar}
Let $\{(X_t,\varepsilon_t)\}_{t\ge 1}$ be a filtered sequence with history $\mathcal S_{t-1}$, where
$\varepsilon_t=(\varepsilon_t^{(1)},\varepsilon_t^{(2)})\in\mathbb R^2$,
$
Y_t \;=\; f^\star(X_t)+\varepsilon_t \in \mathbb R^2$ and $f^\star\in\mathcal F\subseteq(\mathcal X\to\mathbb R^2).
$
Assume that for each coordinate $j\in\{1,2\}$ the noise is conditionally zero mean subgaussian and all $\lambda\in\mathbb R$,
\[
\mathbb E[\varepsilon_t^{(j)}\mid X_t,\mathcal S_{t-1}]=0,
\qquad
\ln \mathbb E\!\left[\exp(\lambda\varepsilon_t^{(j)})\mid X_t,\mathcal S_{t-1}\right]
\le \frac{\lambda^2\sigma_j^2}{2},
\]
and \emph{no independence} is assumed between $\varepsilon_t^{(1)}$ and $\varepsilon_t^{(2)}$. Let $N(\epsilon):=N(\epsilon,\mathcal F,\|\cdot\|_\infty)$ be the corresponding covering number with $\|f\|_{\infty} = \sup_{x}\|f(x)\|_2$ For each $t\ge 1$, let $\hat f_t$ be an empirical risk minimizer:
\[
\hat f_t= \arg
\min_{f\in\mathcal F}\sum_{s=1}^t \|f(X_s)-Y_s\|_2^2.
\]
Define $\sigma_{\max}:=\max\{\sigma_1,\sigma_2\}$ and $\bar\sigma:=\sqrt{\sigma_1^2+\sigma_2^2}$.
Then for any $\delta\in(0,1)$, with probability at least $1-\delta$, for any fixed $t\ge 1$,
\[
\sum_{s=1}^t \|\hat f_t(X_s)-f^\star(X_s)\|_2^2
\;\le\;
\inf_{\epsilon>0}\Bigg\{
10t\epsilon^2
+12\bar\sigma\,t\epsilon
+32\sigma_{\max}^2\ln\!\Big(\frac{2N(\epsilon)}{\delta}\Big)
+12\bar\sigma^2\ln\!\Big(\frac{4}{\delta}\Big)
\Bigg\}.
\]
\end{lemma}

\begin{proof}
Fix $\delta\in(0,1)$ and $\epsilon>0$.
For $f\in\mathcal F$, define 
\[
\phi(f,Z_t)
:=\|f(X_t)-Y_t\|_2^2-\|f^\star(X_t)-Y_t\|_2^2,
\qquad Z_t:=(X_t,Y_t).
\]
Let $\Delta_t:=f(X_t)-f^\star(X_t)\in\mathbb R^2$. Since $Y_t=f^\star(X_t)+\varepsilon_t$,
\begin{equation}\label{eq:phi-expand}
\phi(f,Z_t)=\|\Delta_t\|_2^2-2\langle \Delta_t,\varepsilon_t\rangle.
\end{equation}

Now, let $\mathcal F_\epsilon$ be an $\epsilon$-cover of $(\mathcal F,\|\cdot\|_\infty)$ with
$|\mathcal F_\epsilon|=N(\epsilon)$.
Choose
$\eta:=\frac{1}{8\sigma_{\max}^2},$ using \eqref{eq:phi-expand} and Lemma~\ref{lem:cs-mgf} with $u=2\eta\Delta_t$,
\begin{align}
   \ln \mathbb E\!\left[\exp\!\big(-\eta\phi(f,Z_t)\big)\mid X_t,\mathcal S_{t-1}\right]
&=
\ln \mathbb E\!\left[\exp\!\Big(-\eta\|\Delta_t\|_2^2+2\eta\langle\Delta_t,\varepsilon_t\rangle\Big)\mid\cdot\right]
\nn\\&\le -\eta\|\Delta_t\|_2^2 + 4\eta^2\sigma_{\max}^2\|\Delta_t\|_2^2
= -\frac{\|\Delta_t\|_2^2}{16\sigma_{\max}^2}.
\label{eq:appth1}
\end{align}
Now using equation \eqref{eq:appth1}, the lemma \ref{lemma:th1} and a union bound over $f\in\mathcal F_\epsilon$ implies that with probability at least $1-\delta/2$,
for the fixed $t\ge 1$ and all $f\in\mathcal F_\epsilon$,
\[
-\eta\sum_{i=1}^t \phi(f,Z_i) + \frac{1}{16\sigma_{\max}^2}\sum_{i=1}^t \|f(X_i)-f^\star(X_i)\|_2^2
\le \ln\!\Big(\frac{2N(\epsilon)}{\delta}\Big).
\]
Rearranging gives
\begin{equation}\label{eq:cover-ineq}
\sum_{i=1}^t \|f(X_i)-f^\star(X_i)\|_2^2
\le
2\sum_{i=1}^t \phi(f,Z_i)
+16\sigma_{\max}^2 \ln\!\Big(\frac{2N(\epsilon)}{\delta}\Big).
\end{equation}

Now for each coordinate $j$, using the lemma \ref{lemma:th2}
\[
\sum_{i=1}^t (\varepsilon_i^{(j)})^2
\le
2t\sigma_j^2 + 3\sigma_j^2\ln\!\Big(\frac{4}{\delta}\Big).
\]
Taking a union bound over $j\in\{1,2\}$, with probability at least $1-\delta/2$,
for the fixed $t\ge 1$,
\begin{equation}\label{eq:At}
A_t:=\sum_{i=1}^t \|\varepsilon_i\|_2^2
=\sum_{i=1}^t\big((\varepsilon_i^{(1)})^2+(\varepsilon_i^{(2)})^2\big)
\le 2t\bar\sigma^2 + 3\bar\sigma^2\ln\!\Big(\frac{4}{\delta}\Big).
\end{equation}
Fix $t\ge 1$. Let $f\in\mathcal F_\epsilon$ satisfy $\|f-\hat f_t\|_\infty\le \epsilon$.
Then for each $i\le t$,
$\|f(X_i)-Y_i\|_2 \le \|\hat f_t(X_i)-Y_i\|_2 + \epsilon.$ Therefore,
\[
\Big(\sum_{i=1}^t \|f(X_i)-Y_i\|_2^2\Big)^{1/2}
\le
\Big(\sum_{i=1}^t \|\hat f_t(X_i)-Y_i\|_2^2\Big)^{1/2} + \sqrt{t}\,\epsilon.
\]
By the property of ERM  and $f^\star\in\mathcal F$,
\[
\sum_{i=1}^t \|\hat f_t(X_i)-Y_i\|_2^2
\le
\sum_{i=1}^t \|f^\star(X_i)-Y_i\|_2^2
=
\sum_{i=1}^t \|\varepsilon_i\|_2^2
= A_t.
\]
Hence, using \eqref{eq:At},
\begin{align}
&    \sum_{i=1}^t \phi(f,Z_i)
=
\sum_{i=1}^t\|f(X_i)-Y_i\|_2^2 - A_t
\le (\sqrt{A_t}+\sqrt{t}\epsilon)^2 - A_t
= t\epsilon^2 + 2\epsilon\sqrt{tA_t} \nn\\
&\leq t\epsilon^2 + 2\sqrt{2}\,\bar\sigma\,t\epsilon
+2\sqrt{3}\,\bar\sigma\,\epsilon\sqrt{t\ln(4/\delta)}\nn\\
&\leq  2t\epsilon^2 + 2\sqrt{2}\,\bar\sigma\,t\epsilon + 3\bar\sigma^2\ln\!\Big(\frac{4}{\delta}\Big)\label{eq:phi-bound}
\end{align}

On the intersection of the events \eqref{eq:cover-ineq} and \eqref{eq:At} (probability at least $1-\delta$),
plug \eqref{eq:phi-bound} into \eqref{eq:cover-ineq} to obtain, for the fixed $t\ge 1$ and the selected $f\in\mathcal F_\epsilon$,
\[
\sum_{i=1}^t \|f(X_i)-f^\star(X_i)\|_2^2
\le
4t\epsilon^2 + 4\sqrt{2}\,\bar\sigma\,t\epsilon
+6\bar\sigma^2\ln\!\Big(\frac{4}{\delta}\Big)
+16\sigma_{\max}^2\ln\!\Big(\frac{2N(\epsilon)}{\delta}\Big).
\]
Also, since $\|f-\hat f_t\|_\infty\le \epsilon$,
$\Big(\sum_{i=1}^t\|\hat f_t(X_i)-f^\star(X_i)\|_2^2\Big)^{1/2}
\le
\Big(\sum_{i=1}^t\|\hat f_t(X_i)-f(X_i)\|_2^2\Big)^{1/2}
+
\Big(\sum_{i=1}^t\|f(X_i)-f^\star(X_i)\|_2^2\Big)^{1/2}
\le \sqrt{t}\,\epsilon + \sqrt{B_t},$
where $B_t$ denotes the right-hand side of the previous display. Thus
\[
\sum_{i=1}^t\|\hat f_t(X_i)-f^\star(X_i)\|_2^2
\le 2t\epsilon^2 + 2B_t
\le
10t\epsilon^2
+12\bar\sigma\,t\epsilon
+32\sigma_{\max}^2\ln\!\Big(\frac{2N(\epsilon)}{\delta}\Big)
+12\bar\sigma^2\ln\!\Big(\frac{4}{\delta}\Big),
\]
\end{proof}
\paragraph{Main bound: federated ERM vs local ERM.}
We apply Lemma~\ref{lem:mlp_ise_scalar} jointly to the vector-valued
estimator $f_{\mlpw}=(\aesw,\cesw)$, whose empirical squared loss
coincides with our joint ERM objective. The lemma directly controls
the combined estimation error of both outputs, so no separate
application or union bound over the two heads is required.

\begin{theorem}[Federated ERM can yield lower routing suboptimality]
\label{thm:mlp_subopt_bounds}
Assume realizability (Assumption \ref{assump:mlp_realizable}).
Fix a client $i$ and let $\widehat{\mlpw}_i$ be the local ERM on $\dcal_i$.
Then, for a universal constant $C>0$ and our estimator function class $\fcal$ (\Cref{eq:function_class}),  with probability at least $1-\delta$,
\begin{align}
\mathrm{Subopt}_i(\pi_{\widehat{i,\mlpw}_i})
&\le
C\,\max\{1,\lmb_i\}\,\Gamma^{(i)}(\dcal_i)\,\cm\,
\sqrt{\log\!\left(\frac{2\,\mathcal N(1/\D_i,\mathcal F,\|\cdot\|_\infty)}{\delta}\right)}.
\label{eq:mlp_subopt_local}
\end{align}
Also, let $\widehat{\mlpw}_{\mathrm{fed}}$ be the federated ERM on $\dcal=\bigcup_{j=1}^N \dcal_j$. Then, with probability at least $1-\delta$,
\begin{align}
\mathrm{Subopt}_i(\pi_{i,\widehat{\mlpw}_{\mathrm{fed}}})
&\le
C\,\max\{1,\lmb_i\}\,\Gamma^{(i)}(\dcal)\,\cm\,
\sqrt{\log\!\left(\frac{2\,\mathcal N(1/\D,\mathcal F,\|\cdot\|_\infty)}{\delta}\right)}.
\label{eq:mlp_subopt_federated}
\end{align}
\end{theorem}
where $N(\epsilon,\mathcal F,\|\cdot\|_\infty)$ is the covering number of the class of estimators we use in the $\|\cdot\|_{\infty}$ norm, i.e, $\|f(x) - f'(x)\|_{\infty} = \sup_x \|f(x)-f'(x)\|_2$.

\begin{proof}
Let $\mlpw\in\wcal$ be arbitrary model weights, abbreviate $\widehat \pi(\xb):=\widehat \pi_{i,\mlpw}(\xb)$ and $\pi^\star(\xb):=\pi_i^\star(\xb)$, and drop  $\dt_i$ from notation of sampling $\xb$ for conciseness.
By definition,
\begin{align}
\mathrm{Subopt}_i(\mlpw)
&=
\E_{\xb\sim\dt_i}\!\left[U_i(\xb,\pi^\star(\xb)) - U_i(\xb,\widehat \pi(\xb))\right]\nn\\
&=
\E_{\xb}\!\left[\widehat U_i(\xb,\pi^\star(\xb);\mlpw)-\widehat U_i(\xb,\widehat \pi(\xb);\mlpw)\right]
+
\E_{\xb}\!\left[U_i(\xb,\pi^\star(\xb))-\widehat U_i(\xb,\pi^\star(\xb);\mlpw)\right]\nn\\
&\qquad
+
\E_{\xb}\!\left[\widehat U_i(\xb,\widehat \pi(\xb);\mlpw)-U_i(\xb,\widehat \pi(\xb))\right]. \label{eq:subopt_decomp}
\end{align}
The first term in equation \eqref{eq:subopt_decomp} is $\le 0$ because $\widehat \pi(\xb)$ maximizes $\widehat U_i(\xb,\cdot)$ $\forall \xb \in \Xcal$.
Thus,
\begin{align}
&\mathrm{Subopt}_i(\mlpw)
\le
\left|\E_{\xb}\!\left[\aopt(\xb,\pi^\star(\xb))-\aesw(\xb,\pi^\star(\xb))\right]\right|
+
\lmb_i\,\left|\E_{\xb}\!\left[\copt(\xb,\pi^\star(\xb))-\cesw(\xb,\pi^\star(\xb))\right]\right|\nn\\
&\quad+
\left|\E_{\xb}\!\left[\aopt(\xb,\widehat \pi(\xb))-\aesw(\xb,\widehat \pi(\xb)\right]\right|
+
\lmb_i\,\left|\E_{\xb}\!\left[\copt(\xb,\widehat m(\xb))-\cesw(\xb,\widehat m(\xb)\right]\right|\nn\\
&\le
2\max\{1,\lmb_i\}
\left(
\sup_{\pi\in\Pi_i}\left|\E_{\xb}\!\left[\aopt(\xb,\pi(\xb))-\aesw(\xb,\pi(\xb)\right]\right|
+
\sup_{\pi\in\Pi_i}\left|\E_{\xb}\!\left[\copt(\xb,\pi(\xb))-\cesw(\xb,\pi(\xb)\right]\right|
\right).\nn
\end{align}

Now apply the definitions \eqref{eq:Gamma_acc_def}--\eqref{eq:Gamma_cost_def} with a dataset $S$, yielding
\[
\mathrm{Subopt}_i(\mlpw)
\le
2\max\{1,\lmb_i\}\,\Gamma^{(i)}(S)
\left(
\sqrt{\sum_{j=1}^{|S|}\Delta a_j^2}
+
\sqrt{\sum_{j=1}^{|S|}\Delta c_j^2}
\right),
\]
where $\Delta a_j$ and $\Delta c_j$ denote the pointwise estimation errors on $(\xb^{(j)},m^{(j)})$.
Using $\sqrt{u}+\sqrt{v}\le \sqrt{2(u+v)}$ gives
\[
\mathrm{Subopt}_i(\mlpw)
\le
2\sqrt{2}\max\{1,\lmb_i\}\,\Gamma^{(i)}(S)\,
\sqrt{\sum_{j=1}^{|S|}\left(\Delta a_j^2+\Delta c_j^2\right)}.
\]

Finally, applying Lemma~\ref{lem:mlp_ise_scalar} to the joint
estimator $f_{\mlpw}=(\aesw,\cesw)$ with $\epsilon=1/|S|$
bounds the combined squared estimation error, yielding
\eqref{eq:mlp_subopt_local} for $S=\dcal_i$ (local-only training)
and \eqref{eq:mlp_subopt_federated} for $S=\dcal$
(federated training).
\end{proof}

\begin{remark}[Federated learning helps with suboptimality.]
To compare the suboptimality bounds of local-only \mlp{} (\Cref{eq:mlp_subopt_local}) on client $i$ and federated \mlp{} (\Cref{eq:mlp_subopt_federated}), we need to compare $\Gamma^{(i)}(\dcal_i)$ vs $\Gamma^{(i)}(\dcal)$, and $\sqrt{\log\!\left(\frac{2\,\mathcal N(1/\D_i,\mathcal F,\|\cdot\|_\infty)}{\delta}\right)}$ vs $\sqrt{\log\!\left(\frac{2\,\mathcal N(1/\D,\mathcal F,\|\cdot\|_\infty)}{\delta}\right)}$. For the former comparison, one can see that since $\dcal_i\subset\dcal$, $\Gamma^{(i)}(\dcal_i) \geq \Gamma^{(i)}(\dcal)$ always holds and decays roughly as $\frac{1}{\sqrt{\D_i}}$ and $\frac{1}{\sqrt{\D}}$, respectively, \citep{wainwright2019high}. For the latter comparison, terms grow with $\log\lp\D_i\rp$ and $\log\lp\D\rp$, respectively \citep{wainwright2019high}. Since $\sqrt{\cdot}$ grows faster, federated learning provides an advantage in data coverage in terms of smaller suboptimality.  

\end{remark}

\input{kmeansapp}

%% file: kmeansapp.tex

\subsection{Theoretical Results for Federated \kmeans{}}
\label{sec:theory_kmeans}
\paragraph{Setup and notation.}
We reuse the federated data model from \Cref{sec:theory_mlp}.
Fix a test-time tradeoff parameter $\lmb\ge 0$, and let $\dt$ denote the \emph{global test
distribution} on embeddings $\xb\in\X$ (e.g., a mixture of client test distributions).
For each model $m\in\mcal$, recall the conditional expected accuracy and cost functions
$\ac(\xb,m)$ and $\cs(\xb,m)$, and define the \emph{true expected utility}
\begin{equation}
\util_\lmb(\xb,m)
~\triangleq~
\ac(\xb,m) - \lmb\,\cs(\xb,m).
\label{eq:kmeans_util_def}
\end{equation}
The optimal router under this setting is given by
\begin{equation}
\pi^\star(\xb) \in \arg\max_{m\in\mcal} \util_\lmb(\xb,m),
\label{eq:kmeans_bayes_opt_router}
\end{equation}
where all ties are broken uniformly at random and independently of all other randomness.
For any routing policy $\pi:\X\to\mcal$, define the \emph{routing suboptimality}
\begin{equation}
\Subopt(\pi)
~\triangleq~
\E_{\xb\sim\dt}\!\Big[
\util_\lmb\!\big(\xb,\pi^\star(\xb)\big) - \util_\lmb\!\big(\xb,\pi(\xb)\big)
\Big].
\label{eq:kmeans_subopt_def}
\end{equation}

\paragraph{\kmeans{} quantization map.}
Given any collection of centers $\mathbf{\mu} := \mu_{1:\kglobal} \subseteq \R^{d_{\text{emb}}}$ and define the
nearest-center assignment and quantized embedding
\begin{equation}
k_\mu(\xb) = \arg\min_{k \in \{1,\ldots,\kglobal\}} \lVert \xb - \mu_k\rVert_2,
\qquad
\mu(\xb) \triangleq \mu_{k_\mu(\xb)},
\label{eq:kmeans_assign_def}
\end{equation}
with uniform random tie-breaking in $\arg\min$.
Intuitively, $\mu(\xb)$ is the quantized embedding used by the router.

\paragraph{Population piecewise-constant oracle (given centers).}
Given the partition induced by $k_\mu(\cdot)$, define the \emph{cluster-mean utilities}
\begin{equation}
\bar \util_{k}(m)
~\triangleq~
\E_{\xb\sim\dt}\!\Big[\util_\lmb(\xb,m)\,\big|\,k_\mu(\xb)=k\Big],
\qquad k\in[\kglobal],~m\in\mcal.
\label{eq:kmeans_cluster_means_test}
\end{equation}
The best router that is allowed to use \emph{only} the cluster index (equivalently, only the representative
$\mu(\xb)$) is the \emph{piecewise constant oracle}
\begin{equation}
\pi_{\mu}(\xb) \in \arg\max_{m\in\mcal} \bar \util_{k_\mu(\xb)}(m).
\label{eq:kmeans_pop_pc_router}
\end{equation}

\paragraph{Implemented \kmeans{} router (cluster-level estimation).}
Our implementation estimates, for each cluster $k$ and model $m$, the cluster means of accuracy/cost using the
logged data, yielding estimators $\aest[(m)][k]$ and $\cest[(m)][k]$ and the estimated cluster utility
\begin{equation}
\widehat{\util}_{k}(m)
~\triangleq~
\aest[(m)][k] - \lmb\,\cest[(m)][k].
\label{eq:kmeans_hat_util_def}
\end{equation}
Given a query embedding $\xb$, the resulting \kmeans{} routing rule is
\begin{equation}
\widehat \pi(\xb)
=
\arg\max_{m\in\mcal}\widehat{\util}_{k_\mu(\xb)}(m),
\label{eq:kmeans_emp_router}
\end{equation}
again with uniform random tie-breaking.

\begin{assumption}
\label{ass:split}
    (Sample splitting during dataset construction and uniform logging (for analysis).) Let $\dx_i$ denote the query distribution at client $i$. The dataset $\dcal_i := \dc_i \bigcup \du_i$ at client $i$ is constructed as follows  
\begin{enumerate}
    \item $\dc_i$ construction : Sample $\D_i/2$ queries IID from $\dx_i$ and for each query \textit{arbitrarily} select a model $m\in\mcal$, observes $(\widehat \ac,\widehat \cs)$ with conditional means
$\E[\widehat \ac\mid \xb,M]=\ac(\xb,m)$ and $\E[\widehat \cs\mid \xb,M]=\cs(\xb,m)$ and appends $\{\xb,m, \widehat\ac,\widehat\cs\}$ to the dataset $\dc_i$.
    \item $\du_i$ construction : Sample $\D_i/2$ queries IID from $\dx_i$ and then for each query $\xb$ draws a logged model
$M\sim\mathrm{Unif}(\mcal)$ independently of $\xb$, observes $(\widehat \ac,\widehat \cs)$ with conditional means
$\E[\widehat \ac\mid \xb,M]=\ac(\xb,m)$ and $\E[\widehat \cs\mid \xb,M]=\cs(\xb,m)$ and appends $\{\xb,m, \widehat\ac,\widehat\cs\}$ to the dataset $\du_i$.
\end{enumerate}
\end{assumption}
Using a disjoint sample split for estimations, we enforce statistical independence between the centroid estimates $\mu_{1:\kglobal}$ (using points in dataset $\dc_i$) and the entities $\{\aest[(m)][k],\cest[(m)][k]\}_{k,m}$ (using the points in the dataset $\du_i$). Consequently, conditioning on the learned centers (and therefore on $k_\mu(\cdot)$) does not affect the distributional behavior of the utility estimators.

This split assumption dictates that we have uniform coverage over models $m\in\mcal$ for every prompt in the utility logging dataset. We need this assumption because K-means relies solely on euclidean proximity in embedding space to form clusters, it cannot capture structured correlations in model behaviors. In particular, neural models can potentially exploit these correlated performance patterns that are not explained by embedding similarity alone to generalize better. Particularly, K-means cannot capture the notion that ``similar" models will have ``similar" performance and only use the fact that ``similar" prompts will get ``similar" performance.
\subsubsection{A regret decomposition: quantization + estimation}

\begin{assumption}[Lipschitz utilities in embedding space]
\label{asmp:kmeans_lipschitz}
There exist constants $L_{\acc},L_{\cost}\ge 0$ such that for all $m\in\mcal$ and all $\xb,\xb'\in\R^{d_{\text{emb}}}$,
\[
\big|\ac(\xb,m)-\ac(\xb',m)\big| \le L_{\acc}\,\lVert \xb-\xb'\rVert_2,
\qquad
\big|\cs(\xb,m)-\cs(\xb',m)\big| \le L_{\cost}\,\lVert \xb-\xb'\rVert_2.
\]
Consequently, $\util_\lmb(\cdot,m)$ is $L_\lmb$-Lipschitz with $L_\lmb \triangleq L_{\acc} + \lmb L_{\cost}$.
\end{assumption}

\begin{proposition}[Decomposition into quantization and estimation terms]
\label{lem:kmeans_decomp}
For any fixed centers $\mu_{1:\kglobal}$ (hence fixed $k_\mu(\cdot)$),
\begin{align}
\Subopt(\widehat \pi)
&=
\E_{\xb\sim\distr^{\mathrm{test}}}\!\Big[\util_\lmb\!\big(\xb,\pi^\star(\xb)\big)
-\util_\lmb\!\big(\xb,\pi_{\mu}(\xb)\big)\Big]
+
\E_{\xb\sim\distr^{\mathrm{test}}}\!\Big[\util_\lmb\!\big(\xb,\pi_{\mu}(\xb)\big)
-\util_\lmb\!\big(\xb,\widehat \pi(\xb)\big)\Big].
\label{eq:kmeans_decomp}
\end{align}
\end{proposition}

\subsubsection{Quantization term}
\begin{proposition}[Quantization error controls the population piecewise-constant regret]
\label{prop:kmeans_quantization}
Under Assumption~\ref{asmp:kmeans_lipschitz},
\begin{equation}
\E_{\xb\sim\distr^{\mathrm{test}}}\!\Big[
\util_\lmb\!\big(\xb,\pi^\star(\xb)\big)-\util_\lmb\!\big(\xb,\pi_{\mu}(\xb)\big)
\Big]
~\le~
2L_\lmb\,\E_{\xb\sim\distr^{\mathrm{test}}}\!\big[\lVert \xb - \mu(\xb)\rVert_2\big].
\label{eq:kmeans_quantization_bound}
\end{equation}
\end{proposition}

\begin{proof}
Define the centroid-oracle policy $\pi_{\mathrm{cen}}(\xb)=\arg\max_{m\in\mcal}\util_\lmb(\mu(\xb),m)$.
A standard add-and-subtract argument plus Lipschitzness gives
\(
\E[\util_\lmb(\xb,\pi^\star(\xb))-\util_\lmb(\xb,\pi_{\mathrm{cen}}(\xb))]
\le 2L_\lmb\,\E\|\xb-\mu(\xb)\|_2.
\)
Since $\pi_{\mu}$ maximizes the cluster-mean utility in \eqref{eq:kmeans_pop_pc_router} over all policies that
depend only on $k_\mu(\xb)$, it is at least as good as $\pi_{\mathrm{cen}}$ within each cluster, hence has no greater regret.
\end{proof}

\subsubsection{Estimation term (uniform logging + shift + concentration)}

\paragraph{Logged vs.\ test distributions.}

Let $\dx$ denote the global distribution of all utility split embeddings across clients
used to estimate $\{\aest[(m)][k],\cest[(m)][k]\}$.
Define the corresponding logged cluster means
\begin{equation}
\bar \util^{\mathrm{log}}_{k}(m)
~\triangleq~
\E_{\xb\sim\dx}\!\Big[\util_\lmb(\xb,m)\,\big|\,k_\mu(\xb)=k\Big].
\label{eq:kmeans_cluster_means_log}
\end{equation}

\paragraph{Cell counts and empirical utility.}
For each $(k,m)$, let $n_{k,m}$ denote the \emph{global} number of logged evaluations in the utility split
whose embedding falls in cluster $k$ and whose logged model is $m$ (after federated aggregation across clients).
Equivalently, if the client $i$ contributes $n_{i,k,m}$ such evaluations, then $n_{k,m}=\sum_i n_{i,k,m}$.
Define the empirical cluster utility (which is equal to \eqref{eq:kmeans_hat_util_def} in our notation)
\begin{equation}
\widehat{\util}_{k}(m)
~=~
\frac{1}{n_{k,m}}\sum_{t:\,k_\mu(\xb_t)=k,\,M_t=m}\big(\tilde a_t-\lmb \tilde c_t\big),
\qquad\text{when }n_{k,m}\ge 1.
\label{eq:kmeans_emp_cell_mean}
\end{equation}
Since $\tilde a_t\in[0,1]$ and $\tilde c_t\in[0,c_{\max}]$, the utility per-sample 
$\tilde a_t-\lmb \tilde c_t$ lies in $[-\lmb c_{\max},1]$.

\begin{proposition}[Train--test shift within clusters]
For each cluster $k$, define the conditional distributions
$\dx(\cdot\mid k)$ and $\dt(\cdot\mid k)$ induced by conditioning on $k_\mu(\xb)=k$,
and define the per-cluster Wasserstein-1 shift
\begin{equation}
\Delta_k
~\triangleq~
W_1\!\Big(\dx(\cdot\mid k),\,\dt(\cdot\mid k)\Big).
\label{eq:kmeans_delta_def}
\end{equation}
Then for any $L$-Lipschitz function $f$,
$\left|\E_{\dx(\cdot\mid k)}[f]-\E_{\dt(\cdot\mid k)}[f]\right|\le L\Delta_k$
\end{proposition}
(Any equivalent definition of $W_1$ via optimal couplings is also acceptable; we use it only through
the standard Lipschitz transfer bound, cf.\ Kantorovich--Rubinstein duality \cite{chewi2024statistical}.)

\begin{lemma}[Uniform deviation to the \emph{test} cluster means]
\label{lem:kmeans_uniform_dev}
Under Assumptions \ref{ass:split} and ~\ref{asmp:kmeans_lipschitz}.
Fix the learned centers (equivalently, condition on the centroid split).
Then for any $\delta\in(0,1)$, with probability at least $1-\delta$, simultaneously for all $(k,m)$ with $n_{k,m}\ge 1$,
\begin{equation}
\Big|\widehat{\util}_{k}(m) - \bar \util_{k}(m)\Big|
~\le~
L_\lmb\,\Delta_k
~+~
B_\lmb\sqrt{\frac{\log\!\big(\tfrac{2\kglobal|\mcal|}{\delta}\big)}{2\,n_{k,m}}},
\qquad
B_\lmb \triangleq 1+\lmb \cm.
\label{eq:kmeans_uniform_dev_bound}
\end{equation}
\end{lemma}

\begin{proof}
Fix $(k,m)$ and condition on the event selecting the indices with $k_\mu(\xb)=k$ and $M=m$.
By Assumption~\ref{ass:split}, $M$ is independent of $\xb$, so the selected samples are unbiased for
the cluster-conditional law under $\dx(\cdot\mid k)$, and thus
$\E[\widehat{\util}_{k}(m)\mid \mu_{1:\kglobal}] = \bar \util^{\mathrm{log}}_{k}(m)$.
Hoeffding's inequality for bounded variables in $[-\lmb c_{\max},1]$ gives
\[
\Big|\widehat{\util}_{k}(m) - \bar \util^{\mathrm{log}}_{k}(m)\Big|
\le
B_\lmb\sqrt{\frac{\log\!\big(\tfrac{2\kglobal|\mcal|}{\delta}\big)}{2\,n_{k,m}}}
\]
uniformly over $(k,m)$ by a union bound.
Finally, since $\util_\lmb(\cdot,m)$ is $L_\lmb$-Lipschitz (Assumption~\ref{asmp:kmeans_lipschitz}),
the standard Lipschitz transfer bound yields
$
|\bar \util^{\mathrm{log}}_{k}(m)-\bar \util_{k}(m)| \le L_\lmb \Delta_k,
$
and the claim follows by the triangle inequality.
\end{proof}

\begin{proposition}[Estimation error term]
\label{prop:kmeans_estimation}
Conditioned on the event of Lemma~\ref{lem:kmeans_uniform_dev}, the estimation term in \eqref{eq:kmeans_decomp} satisfies
\begin{equation}
\E_{\xb\sim\distr^{\mathrm{test}}}\!\Big[
\util_\lmb\!\big(\xb,\pi_{\mu}(\xb)\big)-\util_\lmb\!\big(\xb,\widehat \pi(\xb)\big)
\Big]
~\le~
2\,\E_{\xb\sim\distr^{\mathrm{test}}}\!\Big[
\max_{m\in\mcal}\Big|\widehat{\util}_{k_\mu(\xb)}(m)-\bar \util_{k_\mu(\xb)}(m)\Big|
\Big].
\label{eq:kmeans_estimation_term_bound}
\end{equation}
In particular, letting $n_{\min}\triangleq \min_{k,m:\,n_{k,m}\ge 1} n_{k,m}$ and $\Delta_{\max}\triangleq \max_k \Delta_k$,
\begin{equation}
\E_{\xb\sim\distr^{\mathrm{test}}}\!\Big[
\util_\lmb\!\big(\xb,\pi_{\mu}(\xb)\big)-\util_\lmb\!\big(\xb,\widehat \pi(\xb)\big)
\Big]
~\le~
2L_\lmb\,\Delta_{\max}
~+~
2B_\lmb\sqrt{\frac{\log\!\big(\tfrac{2\kglobal|\mcal|}{\delta}\big)}{2\,n_{\min}}}.
\label{eq:kmeans_estimation_term_bound_simple}
\end{equation}
\end{proposition}

\begin{proof}
Fix $\xb$ and write $k=k_\mu(\xb)$. By construction,
\[
\E\!\Big[\util_\lmb(\xb,\pi_{\mu}(\xb))-\util_\lmb(\xb,\widehat \pi(\xb))\,\Big|\,k_\mu(\xb)=k\Big]
=
\bar\util_k\!\big(\pi_{\mu}(\mu_k)\big)-\bar\util_k\!\big(\widehat \pi(\mu_k)\big).
\]
Since $\widehat \pi(\mu_k) = \arg\max_{m\in\mcal}\widehat{\util}_{k}(m)$, we have
\begin{align*}
\bar\util_k\!\big(\pi_{\mu}(\mu_k)\big)-\bar\util_k\!\big(\widehat \pi(\mu_k)\big)
&\le
\bar\util_k\!\big(\pi_{\mu}(\mu_k)\big) - \widehat{\util}_k \big(\pi_{\mu}(\mu_k)\big)
+
\widehat{\util}_k\!\big(\widehat \pi(\mu_k)\big) -\bar\util_k\big(\widehat \pi(\mu_k)\big)\\
&\le
2\max_{m\in\mcal}\big|\widehat{\util}_k(m)-\bar\util_k(m)\big|.
\end{align*}
Taking the expectation over $k$ yields $\xb\sim\dt$ giving the equation \eqref{eq:kmeans_estimation_term_bound}, and equation
\eqref{eq:kmeans_estimation_term_bound_simple} follows by applying the Lemma~\ref{lem:kmeans_uniform_dev}
and then upper bounding $n_{k,m}\ge n_{\min}$ and $\Delta_k\le \Delta_{\max}$.
\end{proof}

\begin{theorem}[Regret bound for piecewise-constant \kmeans{} routing]
\label{thm:kmeans_regret}
Assume sample splitting, Assumptions~\ref{asmp:kmeans_lipschitz} and \ref{ass:split}.
Then for any $\delta\in(0,1)$, with probability at least $1-\delta$,
\begin{equation}
\Subopt(\widehat \pi)
~\le~
2L_\lmb\,\E_{\xb\sim\distr^{\mathrm{test}}}\!\big[\lVert \xb - \mu(\xb)\rVert_2\big]
~+~
2L_\lmb\,\Delta_{\max}
~+~
2B_\lmb\sqrt{\frac{\log\!\big(\tfrac{2\kglobal|\mcal|}{\delta}\big)}{2\,n_{\min}}}.
\label{eq:kmeans_final_bound}
\end{equation}
\end{theorem}

\begin{proof}
Combine the decomposition in Proposition~\ref{lem:kmeans_decomp} with
Proposition~\ref{prop:kmeans_quantization} (first term) and
Proposition~\ref{prop:kmeans_estimation} (second term), using
\eqref{eq:kmeans_estimation_term_bound_simple}.
\end{proof}

\paragraph{Federated vs.\ local-only (interpretation of the estimation term).}
In local-only training, the client $i$ would estimate the cell means using only its own counts $n_{i,k,m}$, so the
concentration term scales as $\tilde O(1/\sqrt{\min_{i,k,m} n_{i,k,m}})$.
In federated aggregation, cell counts add: $n_{k,m}=\sum_i n_{i,k,m}$, so typically
$n_{\min}\gg \min_{i,k,m} n_{i,k,m}$ whenever clients provide complementary coverage of $(k,m)$ cells.
This is precisely the regime where the federated \kmeans{} improves the most over the local-only routing.